%% file: Contextual_Dynamic_Pricing_with_Strategic_Buyers.tex
\documentclass[12pt]{article} 
\usepackage{amsmath,amsthm,amssymb,bm,mathrsfs,amsfonts}
\usepackage{algorithm,algorithmic,graphicx,subfig}
\usepackage{caption,enumerate,natbib,listings,setspace}
\usepackage[T1]{fontenc}
\usepackage[colorlinks,linkcolor=red,anchorcolor=blue,citecolor=blue]{hyperref}
\usepackage[margin=1in]{geometry}
\usepackage[protrusion=false,expansion=true]{microtype}
\usepackage[title]{appendix}
\usepackage{bbm}
\usepackage{comment}
\theoremstyle{plain}
\input{def-smile}

\usepackage{xr}
\usepackage{color}
\usepackage{tikz}
\usetikzlibrary{arrows,calc,positioning}

\tikzstyle{intt}=[draw,text centered,minimum size=6em,text width=5.25cm,text height=0.34cm]
\tikzstyle{intl}=[draw,text centered,minimum size=2em,text width=2.75cm,text height=0.34cm]
\tikzstyle{int}=[draw,minimum size=2.5em,text centered,text width=6.5cm]
\tikzstyle{intg}=[draw,minimum size=2.5em,text centered,text width=6.cm]
\tikzstyle{sum}=[draw,shape=circle,inner sep=2pt,text centered,node distance=3.5cm]
\tikzstyle{summ}=[drawshape=circle,inner sep=4pt,text centered,node distance=3.cm]

\title{\Large{\textbf{Contextual Dynamic Pricing with Strategic Buyers}} } 

\author
{
Pangpang Liu\thanks{Mitchell E. Daniels, Jr. School of Business, Purdue University. Email: liu3364@purdue.edu.}\qquad 
Zhuoran Yang\thanks{Department of Statistics and Data Science, Yale University, Email: zhuoran.yang@yale.edu.} \qquad 
Zhaoran Wang\thanks{Department of Industrial Engineering and Management Sciences, Northwestern University, Email: zhaoranwang@gmail.com.} \qquad 
Will Wei Sun\thanks{Mitchell E. Daniels, Jr. School of Business, Purdue University. Email: sun244@purdue.edu. Corresponding author.}
}

\date{}
\begin{document} 

\maketitle

\begin{abstract}
\noindent
Personalized pricing, which involves tailoring prices based on individual characteristics, is commonly used by firms to implement a consumer-specific pricing policy.  In this process, buyers can also strategically manipulate their feature data to obtain a lower price, incurring certain manipulation costs. Such strategic behavior can hinder firms from maximizing their profits. In this paper, we study the contextual dynamic pricing problem with strategic buyers. The seller does not observe the buyer's true feature, but a manipulated feature according to buyers' strategic behavior. In addition, the seller does not observe the buyers' valuation of the product, but only a binary response indicating whether a sale happens or not. Recognizing these challenges, we propose a strategic dynamic pricing policy that incorporates the buyers' strategic behavior into the online learning to maximize the seller's cumulative revenue. We first prove that existing non-strategic pricing policies that neglect the buyers' strategic behavior result in a linear $\Omega(T)$ regret with $T$ the total time horizon, indicating that these policies are not better than a random pricing policy. We then establish an $O(\sqrt{T})$ regret upper bound of our proposed policy and an $\Omega(\sqrt{T})$ regret lower bound for any pricing policy within our problem setting. This underscores the rate optimality of our policy. Importantly, our policy is not a mere amalgamation of existing dynamic pricing policies and strategic behavior handling algorithms. Our policy can also accommodate the scenario when the marginal cost of manipulation is unknown in advance. To account for it, we simultaneously estimate the valuation parameter and the cost parameter in the online pricing policy, which is shown to also achieve an $O(\sqrt{T})$ regret bound. Extensive experiments support our theoretical developments and demonstrate the superior performance of our policy compared to other pricing policies that are unaware of the strategic behaviors. 

\end{abstract}


\newpage
\baselineskip=25pt 

\section{Introduction}
\label{sec:introduction}
Price discrimination based on customer features, such as web browser, purchasing history, job status, is a common practice among firms \citep{Mikians2013, Hannak2014}.  Personalized pricing uses information on each individual’s observed characteristics to implement consumer-specific price discrimination. However, consumers can also manipulate their data to obtain a lower price, thereby contaminating the data that firms use for targeted pricing. These facts result in firms not always benefiting from acquiring more data to infer consumer preferences. These manipulating behaviors do not alter the true valuation of the costumers, but affect the offered price. Also, the manipulating behavior incurs some costs, which are determined by factors such as laws, technology, educational programs \citep{Li2023}.  \par
Strategic behaviors often arise when buyers become aware of personalized pricing strategies. One specific example is that Home Depot discriminates
against Android users \citep{Hannak2014}. The buyers can use browser plugins
such as the User-Agent Switcher to manipulate their device information. The feature manipulation does not change the buyer's valuation of the product, but it incurs some costs. One cost is to find the fact that Android users get a higher price on Home Depot. The other cost is to learn how to manipulate device information. Another example is loan fraud. To acquire a loan, the borrower may manipulate the income, job status, the value of the car or house \citep{Jerzy2021}. The borrower's valuation of the loan does not change due to the manipulation, but the manipulation causes some costs, such as, preparing documents to prove the income and job status, paying a price to the assets appraisal agency to obtain a high  assessed value of the asset. \par
\begin{figure}
    \centering
    \includegraphics[scale=0.35]{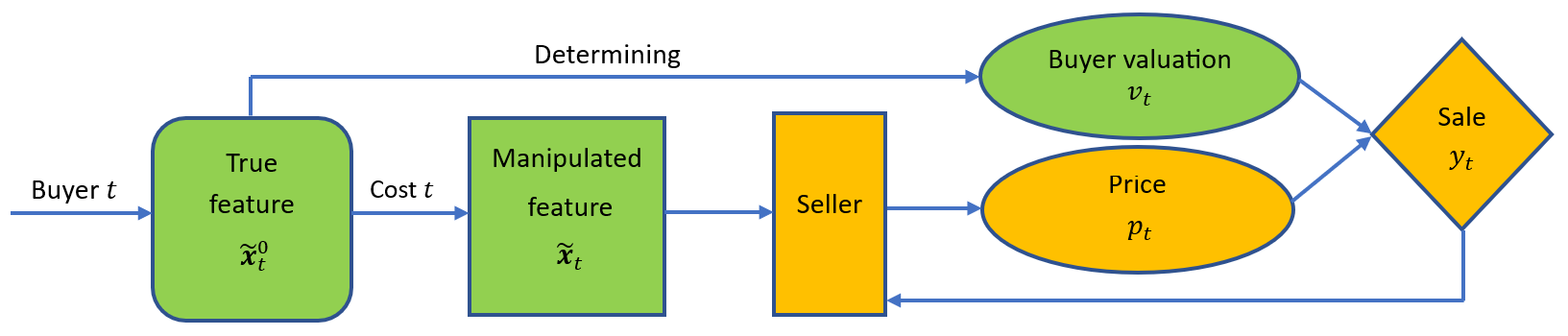}
    \caption{Online dynamic pricing process with strategic buyers. The seller can only observe the manipulated feature, while the buyer's valuation is determined by the true feature.}
    \label{fig0}
\end{figure}
In this paper, we study contextual dynamic pricing problem with strategic buyers. The buyer strategically manipulates features for pursuing a lower price. We consider the manipulating behavior which aims at gaming the pricing policy without altering the true valuation. Figure \ref{fig0} shows the schematic representation of the online dynamic pricing process with strategic buyers. At each time step $t$, a buyer arrives with a true feature vector $\tilde{\boldsymbol{x}}_t^0$. In order to obtain a lower price, the buyer incurs certain costs to manipulate the true feature $\tilde{\boldsymbol{x}}_t^0$ and subsequently reveals the manipulated feature $\tilde{\boldsymbol{x}}_t$ to the seller. Upon receiving the manipulated feature vector $\tilde{\boldsymbol{x}}_t$, the seller makes a pricing decision by selecting a price $p_t$. The buyer, after comparing the price $p_t$ with the valuation $v_t$, which is determined based on the true feature vector $\tilde{\boldsymbol{x}}_t^0$, decides whether to make a purchase ($y_t=1$) or not ($y_t=0$). Finally, the seller collects the revenue $p_t\mathbb{I}(y_t=1)$ at time $t$. These steps are repeated for buyers that arrives sequentially, forming the online dynamic pricing process. Our goal is to develop an online pricing policy to decide the price at each each time to maximize the overall revenue. 

\subsection{Our Contribution}


The aforementioned strategic behavior has not been taken into account in previous dynamic pricing literature \citep{Javanmard2019, Xu2021,Fan2022, Xu2022, luo2022, wang2022, luo2023}, which we refer to as the non-strategic dynamic pricing policies. Studying the strategic behavior of myopic buyers is a necessary and practical topic. To fill the gap, we develop a strategic dynamic pricing policy that takes into consideration buyers' strategic behaviors. As the best of our knowledge, we are the first to consider the strategic behavior of manipulating features in the field of dynamic pricing.\par 

Our policy comprises two phases: the exploration phase and the exploitation phase. In the exploration phase, the seller uses a uniform pricing policy, offering prices from a uniform distribution, to collect features without manipulation and obtain an estimation of buyers' preference parameters based on the collected true features. The rationale behind revealing true features lies in the fact that the offered uniform price is independent of the features, and the optimal action for buyers is to reveal their true features during the exploration phase. In the exploitation phase, the seller employs an optimal pricing policy to collect more revenues. The exploration phase incurs a higher regret but improves the accuracy of parameter estimation. The estimated parameters obtained from the exploration phase aid in learning the true features and implementing the optimal pricing policy during the exploitation phase, resulting in a smaller cumulative regret over a long run. Therefore, the seller faces the exploration-exploitation trade-off to decide between learning about the model parameters (exploration) and utilizing the knowledge gained so far to collect revenues (exploitation). \par


The performance of the pricing policy is evaluated via a (cumulative) regret, which is the cumulative expected revenue loss against a clairvoyant policy that possesses complete knowledge of both the valuation model parameters and the true features of buyers in advance, and always offers the revenue-maximizing price. Theoretically, we prove that our strategic dynamic pricing policies achieve a regret upper bound of $O(\sqrt{T})$, where $T$ is the time horizon. Importantly, we establish an $\Omega(\sqrt{T})$ regret lower bound of any pricing
policy in our problem setting, indicating the optimality of our pricing policy. In a strategic environment, the seller faces the challenge of not having direct access to the true buyer features. This lack of direct observation makes it difficult for the seller to accurately learn the true value associated with each buyer, as the true value is inherently determined by these unobservable features. Importantly, our policy is not a mere amalgamation of existing dynamic pricing policies and strategic behavior handling algorithms. Our policy can also accommodate the scenario when the marginal cost of manipulation is unknown in advance. To account for it, we simultaneously estimate the valuation parameter and the cost parameter in the online pricing policy, where the cost of manipulation is inferred via a small portion of repeated buyers in the exploration and exploitation stages. In contrast, we prove that any non-strategic pricing policy has an $\Omega(T)$ regret lower bound, indicating the necessity of considering strategic dynamic pricing in our problem.

\subsection{Related Literature}
Our work is related to recent literature on contextual dynamic pricing with online learning, and strategic classification. Additional relevant literature on timing and untruthful bidding in pricing and auction design is provided in the supplement.


\subsubsection{Contextual Dynamic Pricing with Online Learning}
There has been a growing interest in studying contextual dynamic pricing with online learning. 
Several aspects of contextual dynamic pricing have been studied, including dynamic pricing in high-dimensions \citep{Javanmard2019}, dynamic pricing with unknown noise distribution \citep{Fan2022, luo2022, Xu2022,luo2023}, always-valid high-dimensional dynamic pricing policy \citep{wang2022}, dynamic pricing with adversarial settings \citep{Xu2021}.  Notably, in these studies, the sellers have access to the true customer characteristics, and the buyers are not strategic in their behaviors. To enhance this existing body of work, our study introduces a novel dimension by considering strategic buyers who can manipulate features to game the pricing system. This extension allows us to explore the interplay between strategic behaviors and dynamic pricing, thereby contributing to the understanding of more realistic and complex market dynamics.

\subsubsection{Strategic Classification}
Strategic classification studies the interaction between a classification rule and the strategic agents it governs. Rational agents respond to the classification rule by manipulating
their features \citep{Hardt2016, Dong2018, chen20201, Ghalme2021,Bechavod2021, shao2023strategic}. 
Specifically, \cite{Ghalme2021} studied the strategic classification, in which the classifier is not revealed to the agents, and the agents' cost function is publicly known. In \cite{chen20201}, the learner knows that the agent misreports the features in a given ball of the true features. 
On the other hand, within the realm of improvement, certain studies have delved into methods for incentivizing agents to improve their outcomes instead of gaming the classifier \citep{Kleinberg2020, Keegan2022}, as well as approaches for identifying meaningful causal variables \citep{Bechavod2021}. 
The strategic classification problem differs significantly from our setting. Strategic classification is a supervised learning problem, where the objective is to minimize misclassification errors. The focus is on developing algorithms that can effectively classify instances based on their features, considering the strategic behavior of the entities involved. 
In contrast, the dynamic pricing problem we address is an online bandit learning problem, where the seller needs to make pricing decisions in a sequential and adaptive manner, and our objective is to minimize regret. In our setting, we consider the strategic behavior of buyers who manipulate their features to obtain lower prices. This introduces additional challenges in estimating buyer preferences and determining optimal pricing strategies. Our work extends the understanding of strategic behaviors in dynamic pricing by considering feature manipulation and its impact on regret minimization, thereby enriching the existing literature in this field.

 Moreover, our research is connected to, yet different from, the concept of performative prediction as introduced by \cite{Perdomo2020} and other related works \citep{Mendler,brown22a,yu2022strategic,chen2023performative}. Performative prediction addresses the distribution shift issue that arises when the collected data distribution changes in response to decision-making policies. It includes strategic classification as a specific case. Our work presents several distinctions from this line of research. Firstly, while performative prediction literature addresses cases where the feature undergoes genuine change, our approach deals with scenarios where the true feature remains unchanged, but the user strategically misreports the feature. Consequently, their work is not applicable to our specific problem, wherein the manipulation of features does not alter the buyer's valuation of the product. Secondly, this difference of problem setting leads to a fundamental difference in the construction of the loss function. The loss function in performative prediction literature integrates observed features from the shifted distribution, whereas in our methodology, it is formulated based on unobserved features prior to any manipulation. Moreover, \cite{Perdomo2020} assumed the loss function to be strongly convex, which is not needed in our setting. Thirdly, our study is tailored for dynamic pricing problems within an online bandit setting, presenting a low-regret algorithm, which has not been investigated in existing performative prediction literature. 
Therefore, these fundamental differences necessitate the development of new algorithms and analysis tools.  

\subsection{Notation}\label{notation}
Throughout this paper,  we denote $[T]=\{1, 2, \cdots, T\}$ for any positive integer $T$. For any vector $\boldsymbol{x}\in \mathbb{R}^n$ and any positive integer $q$, the $L_q$-norm is $\|\boldsymbol{x}\|_q=(\sum_{i=1}^n|x_i|^q)^{1/q}$. For any matrix $\boldsymbol{X}\in\mathbb{R}^{n_1\times n_2}$, we use $\|\cdot\|$ to denote the spectral norm of $\boldsymbol{X}$. For any event $E$, $\mathbb{I}(E)$ represents an indicator function which equals to 1 if $E$ is true and 0 otherwise. For two positive sequences $\{a_n\}_{n\geq 1}, \{b_n\}_{n\geq 1}$, we say $a_n=O(b_n)$ if $a_n\leq Cb_n$ for some positive constant $C$, and $a_n=\Omega (b_n)$ if $a_n\geq Cb_n$ for some positive constant $C$. We let $\tilde{O}(\cdot)$ represent the same meaning of $O(\cdot)$ except for ignoring log factors.

\subsection{Paper Organization}
The rest of the paper is organized as follows. In Section \ref{sec2}, we define the dynamic pricing problem with strategic buyers. In Section \ref{sec3}, We present the policy for dynamic pricing with the known marginal cost. In Section \ref{sec4}, we relax the known marginal cost assumption, and develop a policy for dynamic pricing with unknown marginal cost. In Section \ref{sec5}, we analyze the regret of our proposed strategic policies. In Section \ref{sec6}, we conduct experiments to demonstrate the performance of our algorithm. We provide additional information related to our paper and the proofs in the supplemental materials.



\section{Problem Setting}\label{sec2}
We study the pricing problem where a seller has a single product for sale at each time period $t=1,2,\cdots, T$, where $T$ denotes the length of the horizon and may be unknown to the seller. At time $t$, a buyer with a vector of true covariates $\tilde{\boldsymbol{x}}_t^0\in \mathbb{R}^{d}$ arrives. 
\begin{remark}
In dynamic pricing literature, covariates typically include product features, e.g., insurance product features, and customer characteristics, e.g., customer financial status, and both are observable by the seller. Since product features cannot be modified by buyers, to simply the presentation we only consider customer characteristics in the covariates $\tilde{\boldsymbol{x}}_t^0$ to study the buyers' strategic behavior of manipulating customer characteristics. Our analysis can be straightforwardly extended to the scenario where $\tilde{\boldsymbol{x}}_t^0$ includes both product features and customer characteristics.
\end{remark}
Following the dynamic pricing literature \citep{Javanmard2019, Xu2021,Xu2022, luo2022, wang2022, luo2023}, we assume the buyer's valuation of the product is a linear function of the feature covariates $\tilde{\boldsymbol{x}}_t^0$, which is unobservable by the seller. In particular, we define $\boldsymbol{x}_t^0=(\tilde{\boldsymbol{x}}_t^{0\top}, 1)^\top$, where $\{\tilde{\boldsymbol{x}}_t^0\}_{t\geq 1}$ are independently and identically distributed ($i.i.d.$) samples from an unknown distribution $\mathbb{P}_X$ supported on a bounded subset $\mathcal{X}\in \mathbb{R}^d$. The buyer's valuation function is defined as 
$v_t=\boldsymbol{\theta}_0^\top \boldsymbol{x}_t^0+z_t,$
where $\boldsymbol{\theta}_0= (\boldsymbol{\beta}_0^\top, \alpha_0)^\top\in\mathbb{R}^{d+1}$ represents the buyer's true preference parameter, which is unknown to the seller, and $\{z_t\}_{t\geq 1}$ are $i.i.d.$ noises from a distribution with mean zero and a cumulative density function $F$. At time $t$, the seller posts a price $p_t$. If $p_t\leq v_t$, a sale occurs, and the seller obtains the revenue $p_t$. Otherwise, no sale occurs. We denote $y_t$ as the response variable that indicates whether a sale occurs at time $t$, $i.e.$, 
\begin{equation}
  y_t =
    \begin{cases}
      1 & \text{if} \quad v_t\geq p_t \\
      0 & \text{if}  \quad v_t< p_t.
    \end{cases}       
\end{equation}
The response variable can be represented by the following probabilistic model,
\begin{equation*}
  y_t =
    \begin{cases}
      1 & \text{with probability} \quad 1-F(p_t-\boldsymbol{\theta}_0^\top \boldsymbol{x}_t^0) \\
      0 & \text{with probability} \quad F(p_t-\boldsymbol{\theta}_0^\top \boldsymbol{x}_t^0).
    \end{cases}       
\end{equation*}\par

\subsection{Clairvoyant Policy and Performance Metric}
A clairvoyant seller who knows the true parameter $\boldsymbol{\theta}_0$ and the true feature $\tilde{\boldsymbol{x}}_t^0$ is able to conduct an oracle pricing policy, which can serve as a benchmark for evaluating a pricing policy. The goal of a rational seller is to obtain more revenue. Hence, a clairvoyant seller would post the price by maximizing the expected revenue, that is, 
\begin{equation}\label{p}
p_t^*=\operatorname*{argmax}_p p(1-F(p-\boldsymbol{\theta}_0^\top \boldsymbol{x}_t^0)).
\end{equation}
The first-order condition of (\ref{p}) yields $p_t^*=\frac{1-F(p_t^*-\boldsymbol{\theta}_0^\top \boldsymbol{x}_t^0)}{f(p_t^*-\boldsymbol{\theta}_0^\top \boldsymbol{x}_t^0)}.$
We define $\phi(v)=v-[1-F(v)]/f(v)$ as the virtual valuation function and $g(v)=v+\phi^{-1}(-v)$ as the pricing function. By simple calculations, we obtain the oracle pricing policy as follows,
\begin{equation}\label{price}
 p_t^*=g(\boldsymbol{\theta}_0^\top \boldsymbol{x}_t^0).  
\end{equation}
Now, we discuss the performance measure of a pricing policy. Let $\pi$ be the seller's policy that sets price $p_t$ at time $t$.  To evaluate the performance of any policy $\pi$, we compare its revenue to that of an oracle pricing policy run by a clairvoyant seller who knows both $\boldsymbol{\theta}_0$ and $\tilde{\boldsymbol{x}}_t^0$ and offers $p_t^*$ according to (\ref{price}) for any given $t$.  The worst-case regret is defined as follows,
\begin{equation}\label{bench}
\text{Regret}_{\pi}(T)=\mathop{\max}_{\substack{\theta_0\in\Theta \\ \mathbb{P}_X\in Q(\mathcal{X})}} \mathbb{E}\bigg\{\sum_{t=1}^T[p_t^*\mathbb{I}(v_t\geq p_t^*)-p_t\mathbb{I}(v_t\geq p_t)]\bigg\},
\end{equation}
where the expectation is with respect to the randomness in the noise $z_t$ and the feature $\boldsymbol{x}_t^0$. Here $Q(\mathcal{X})$ represents the set of probability distributions supported on a bounded set $\mathcal{X}$. 
Our objective is to find a pricing policy $\pi$ such that the above total regret is minimized.

\subsection{Feature Manipulation}
As shown in (\ref{price}), the seller's price is determined by the features. Therefore, the buyer has an incentive to manipulate features to lower the price of the product. Following \cite{Bechavod2022}, we consider a quadratic cost function. That is, the buyers' cost for modifying feature $\tilde{\boldsymbol{x}}_t^0$ to $\tilde{\boldsymbol{x}}$ is 
$$cost(\tilde{\boldsymbol{x}}, \tilde{\boldsymbol{x}}_t^0)=\frac{1}{2}(\tilde{\boldsymbol{x}}-\tilde{\boldsymbol{x}}_t^0)^\top A (\tilde{\boldsymbol{x}}-\tilde{\boldsymbol{x}}_t^0),$$
where $A$ is a marginal cost of manipulating features. In the main paper, we assume that $A$ is fixed and same across users. In the supplementary materials, we extend our policy to accommodate the scenario of heterogeneous marginal costs.

\begin{assumption}\label{assc}
The marginal cost $A$ is assumed to be a symmetric positive definite matrix with the minimum eigenvalue $\lambda_{Amin}$ and the maximum eigenvalue $\lambda_{Amax}$.   
\end{assumption}
This functional form is a simple way to model important practical situations in which features can be modified in a correlated manner, and investing in one feature may lead to changes in other features. 
In Section \ref{sec3}, we assume the marginal cost of manipulation $A$ is known by the seller, same as \cite{Bechavod2022}, and in Section \ref{sec4}, we relax this assumption by considering a more challenging unknown manipulation cost. \par

Let $\tilde{\boldsymbol{x}}_t$ be the manipulated feature, which is observable by the seller. We define $\boldsymbol{x}_t=(\tilde{\boldsymbol{x}}_t^\top, 1)^\top$.  From buyers' perspective, the seller assesses the expected valuation by $\boldsymbol{\theta}_0^\top \boldsymbol{x}_t$. The total cost of buyers is the price $p_t$ and the manipulation cost $cost(\tilde{\boldsymbol{x}}, \tilde{\boldsymbol{x}}_t^0)$, where the price $p_t$ is determined by the seller's pricing policy. We consider two pricing policies, \textbf{the uniform pricing policy} in the exploration stage and \textbf{the optimal pricing policy} in the exploitation stage. In the exploration stage, the seller focuses on collecting more informative data for parameter estimation and hence implements a uniform pricing policy such that $p_t$ is randomly chosen from a uniform distribution $\text{Unif}(0, B)$. After this initial period, the exploitation stage implements an optimal pricing policy such that price is set by the pricing function $g(\cdot)$. Let $\hat{\boldsymbol{\theta}}$ be the seller's estimation of $\boldsymbol{\theta}_0$. In our pricing process, we assume that the seller's pricing policy is transparent to the buyers \citep{Chen2018}, meaning that the buyers are aware that the seller is implementing either a uniform pricing policy or an optimal pricing policy $g(\cdot)$. It is important to note that the specific assessment rule $\hat{\boldsymbol{\theta}}$ used by the seller is not revealed to the buyers,  which is a similar assumption made in  \cite{Bechavod2022}.\par
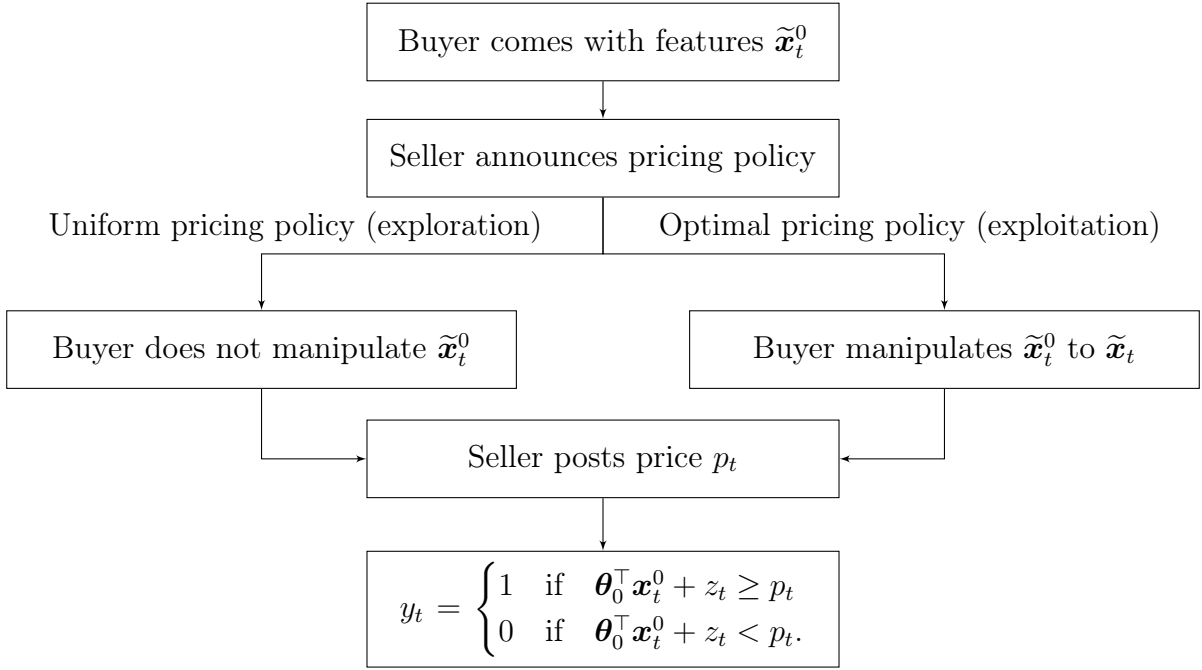
\begin{figure}[!htb]
    \centering
    \begin{tikzpicture}[
      >=latex',
      auto
    ]
      \node [intg] (start)   {Buyer comes with features $\tilde{\boldsymbol{x}}_t^0$};
      \node [intg] (kp) [node distance=0.5cm,below=of start] {Seller announces pricing policy};
      \node [int]  (ki1) [node distance=1.5cm and -2cm,below left=of kp] {Buyer does not manipulate $\tilde{\boldsymbol{x}}_t^0$};
      \node [int]  (ki2) [node distance=1.5cm and -2cm,below right=of kp] {Buyer manipulates $\tilde{\boldsymbol{x}}_t^0$ to $\tilde{\boldsymbol{x}}_t$};
      \node [intg] (ki3) [node distance=4cm,below of=kp] {Seller posts price $p_t$};
      \node [intg] (ki4) [node distance=2cm,below of=ki3] {$  y_t =
    \begin{cases}
      1 & \text{if} \quad \boldsymbol{\theta}_0^\top \boldsymbol{x}_t^0+z_t\geq p_t \\
      0 & \text{if}  \quad \boldsymbol{\theta}_0^\top \boldsymbol{x}_t^0+z_t< p_t.
    \end{cases} $};
      \draw[->] (start) -- (kp);
      \draw[->] (kp) -- ($(kp.south)+(0,-0.75)$) -| (ki1) node[above,pos=0.45] {Uniform pricing policy (exploration)} ;
      \draw[->] (kp) -- ($(kp.south)+(0,-0.75)$) -| (ki2) node[above,pos=0.45] {Optimal pricing policy (exploitation)};
      \draw[->] (ki1) |- (ki3);
      \draw[->] (ki2) |- (ki3);
      \draw[->] (ki3) -- (ki4);
    \end{tikzpicture}
    \caption{Schematic representation of the strategic dynamic pricing policy.}
    \label{fig01}
  \end{figure}
Prior to buyers making decisions, the seller discloses to the buyer: i) the chosen pricing policy (uniform or optimal) and ii) the pricing function $g(\cdot)$ if the optimal pricing policy is employed, without revealing the estimated parameter $\hat{\boldsymbol{\theta}}$. Based on this information, buyers engage in manipulation.
By revealing the manipulated features $\tilde{\boldsymbol{x}}$ to the seller, buyers estimate that the price for the product is $g(\alpha_0+\boldsymbol{\beta}_0^\top \tilde{\boldsymbol{x}})$ when the optimal pricing policy is conducted. It is noteworthy that buyers know $\alpha_0$ and $\boldsymbol{\beta}_0$, which represent their valuation parameters.  Additionally, it is important to acknowledge that buyers cannot access $\hat{\alpha}$ and $\hat{\boldsymbol{\beta}}$, as they lack access to the data utilized in obtaining these estimates.  Consequently, the best values available to the buyers for estimating the price offered by the seller is $\alpha_0$ and $\boldsymbol{\beta}_0$, as $\hat{\alpha}$ and $\hat{\boldsymbol{\beta}}$ serve as estimates of these parameters.
Given the true covariate $\tilde{\boldsymbol{x}}_t^0$ and the pricing policy $p$, the buyer chooses the manipulated features $\tilde{\boldsymbol{x}}$ by minimizing the following total cost,
\begin{equation}\label{utility}
C(\tilde{\boldsymbol{x}}, \tilde{\boldsymbol{x}}_t^0)=p+\frac{1}{2}(\tilde{\boldsymbol{x}}-\tilde{\boldsymbol{x}}_t^0)^\top A (\tilde{\boldsymbol{x}}-\tilde{\boldsymbol{x}}_t^0),
\end{equation}
where
\begin{equation*}
  p =
    \begin{cases}
      \tilde{p} \sim \text{Unif}(0, B) & \text{if the uniform pricing policy is conducted},\\
      g(\alpha_0+\boldsymbol{\beta}_0^\top \tilde{\boldsymbol{x}}) & \text{if the optimal pricing policy is conducted}.
    \end{cases}       
\end{equation*}
When the uniform pricing policy is conducted in the exploration stage, the price is not related to the features, hence buyers have no incentive to manipulate features, and $\tilde{\boldsymbol{x}}_t=\tilde{\boldsymbol{x}}_t^0$. When the optimal pricing policy is conducted, the first-order condition of (\ref{utility}) yields
\begin{equation}\label{xt}
\tilde{\boldsymbol{x}}_t=\tilde{\boldsymbol{x}}_t^0- A^{-1}\boldsymbol{\beta}_0 g'(\alpha_0+\boldsymbol{\beta}_0^\top \tilde{\boldsymbol{x}}_t).
\end{equation}
Figure \ref{fig01} displays the schematic representation of the strategic dynamic pricing policy. 
\begin{remark}
Equation (\ref{xt}) is the first-order necessary condition of minimizing (\ref{utility}) when the optimal pricing policy is conducted. For simplicity, we consider the case where $g(\cdot)$ is convex in $\tilde{\boldsymbol{x}}$ and hence (\ref{xt}) is a unique minimizer of (\ref{utility}). When minimizing (\ref{utility}) is not a convex problem, $\tilde{\boldsymbol{x}}_t$ from (\ref{xt}) is not necessarily the global minimum, and multiple $\tilde{\boldsymbol{x}}_t$'s may satisfy (\ref{xt}). In practice, the buyers can try different $\tilde{\boldsymbol{x}}_t$'s which satisfy  (\ref{xt}) and determine an $\tilde{\boldsymbol{x}}_t$ such that $C(\tilde{\boldsymbol{x}}_t, \tilde{\boldsymbol{x}}_t^0)$ is the smallest.
\end{remark}

\subsection{Linear Regret for Non-strategic Pricing Policy}

While various dynamic pricing policies have been proposed \citep{Javanmard2019, Fan2022,luo2022,wang2022}, none of them considers the impact of strategic behaviors in the pricing problem. Since the true feature $\boldsymbol{x}_t^0$ is unobservable by the seller, the pricing policy $g(\hat{\boldsymbol{\theta}}^\top \boldsymbol{x}_t^0)$ used in previous literature is not applicable. In this case, \textbf{the non-strategic pricing policy} would set the price as $p_t=g(\hat{\boldsymbol{\theta}}^\top \boldsymbol{x}_t)$, which uses the manipulated feature for pricing. 

In this section, we prove that the non-strategic pricing policy incurs a linear regret lower bound of $\Omega(T)$ in the considered pricing problem. 
We first present some standard assumptions in the dynamic pricing literature. Under these assumptions, the non-strategic pricing policy incurs a linear regret. In later sections, we will show that our proposed strategic pricing policy achieve a sub-linear regret under the same assumptions. 

\begin{assumption}\label{a0}
$\|\boldsymbol{x}_t^0\|_{2}\leq W_x, \|\boldsymbol{\theta}_0\|_1\leq W_\theta$ for some constants $W_x>0, W_\theta>0$.
\end{assumption}
Assumption \ref{a0} is standard in dynamic pricing literature \citep{Javanmard2019, Fan2022,zhao2023highdimensional}. By Assumption \ref{a0}, we know $\boldsymbol{\theta}_0\in \Theta= \{\boldsymbol{\theta}, \|\boldsymbol{\theta}\|_1\leq W_\theta\}$.
\begin{assumption}\label{a3}
The buyers' valuation $v_t\in (0, B)$ for a known constant $B>0$. 
\end{assumption}
Assumption \ref{a3} assumes a known upper bound for the buyers' valuations \citep{Fan2022, luo2022,Bu2022}, which is a mild condition in practical applications. With this assumption, the seller can set a price $p_t\in (0, B).$ \par

\begin{assumption}\label{a1}
The function $F(z)$ is strictly increasing, $F(z)$ and $1-F(z)$ are log-concave in $z$. For $z\in [-W, B]$, where $W=W_{\theta}W_x$, we assume $|f(z)|< M_f, |f'(z)|< M_{f'}$ and $|f''(z)|< M_{f''}$, for some constants $M_f>0, M_{f'}>0, M_{f''}>0$.
\end{assumption}
The assumption of log-concavity is commonly used in dynamic pricing literature \citep{Javanmard2017,Javanmard2019,Tang2020,Xu2021,wang2022}. By Assumptions \ref{a0} and \ref{a3}, we have $(p_t-\boldsymbol{\theta}_0^\top \boldsymbol{x}_t^0)\in [-W, B]$. Assumption \ref{a1} states that $f, f'$ and $f''$ are bounded on a finite interval $[-W, B]$, and is satisfied by some common probability distributions including normal, uniform, Laplace, exponential, and logistic distributions. 
\begin{assumption}\label{a2}
The second moment matrix $\Sigma =\mathbb{E}[\boldsymbol{x}_t^0\boldsymbol{x}_t^{0\top}]$ is positive definite. We denote the minimum eigenvalue and maximum eigenvalue of $\Sigma$ as $\lambda_{min}$ and $\lambda_{max}$, respectively.
\end{assumption}
Assumption \ref{a2} is a standard condition on the feature distribution, and 
holds for many common probability distributions, such as uniform, truncated normal, and in general truncated version of many more distributions \citep{Javanmard2019}.

The pricing policy operates in an episodic manner, allowing for the consideration of an unknown total time horizon $T$, see Figure \ref{episode}. Episodes are indexed by $k$ and time periods are indexed by $t$. The length of episode $k$ is denoted by $\ell_k$. Each episode is divided into two phases: the exploration phase of length $a_k$ and the exploitation phase of length $\ell_k-a_k$.
\begin{theorem}\label{theory3}
Let Assumptions \ref{assc}, \ref{a0}, \ref{a3}, \ref{a1} and \ref{a2} hold. Let $\hat{\boldsymbol{\theta}}_k$ be the estimate from (\ref{est}) in the $k$-th episode. At the time period $t$  during the exploitation phase in the $k$-th episode, using the non-strategic pricing policy $p_t=g(\hat{\boldsymbol{\theta}}_k^\top \boldsymbol{x}_t)$,  for the problem instance with a uniform distribution $F(\cdot)$ on $(-1/2, -1/2), \|\boldsymbol{\beta}_0\|_1=1, B=7/16$ and $\|\tilde{\boldsymbol{x}}_t^0\|_2\leq 1/4$, there exist constants $\epsilon>0, C>0$, such that when $T>\frac{C}{(1-\epsilon)^4}$, we have
$Regret_{\pi}(T)>\epsilon T/4.$
\end{theorem}
Theorem \ref{theory3} reveals that under a uniform distribution $F(\cdot)$, the non-strategic pricing policy with strategic buyers has a linear regret lower bound of $\Omega(T)$, indicating that it is not better than a random pricing policy. 
This result underscores the necessity of a new strategic pricing policy in the presence of strategic buyers. Motivated from this, in Sections \ref{sec3} and \ref{sec4}, we develop new strategic dynamic pricing policies to account for the strategic behaviors. 

\section{Strategic Pricing with Known Marginal Cost}\label{sec3}
In this section, we introduce a novel dynamic pricing policy when the marginal cost $A$ is known in advance. In Section \ref{sec4}, we will relax this assumption and consider the case of unknown $A$. The detail of the strategic pricing policy with known marginal cost is shown in Algorithm \ref{alg1}.\par
\begin{algorithm}[!ht]
\caption{Strategic Dynamic  Pricing Policy with Known Marginal Cost}
\label{alg1}
\begin{algorithmic}[1]
\STATE  \textbf{Input}: $B, \ell_0, C_a$
\FOR{each episode $k=1,2,...$}
\STATE Set the length of $k$-th episode as $\ell_k= 2^{k-1}\ell_0$, and $a_k=\lfloor \sqrt{C_a\ell_k}\rfloor$. 
\STATE \textbf{Exploration Phase (Uniform Pricing Policy)}:
\FOR{$t\in I_k:=\{ \ell_k,...,\ell_k+a_k-1\}$}
\STATE The buyer reveals $\tilde{\boldsymbol{x}}_t=\tilde{\boldsymbol{x}}_t^0$. Denote $\boldsymbol{x}_t=(\tilde{\boldsymbol{x}}_t^\top, 1)^\top$.
\STATE The seller offers a price $p_t$ randomly chosen from Unif$(0,B)$. 
\STATE Observe a binary response $y_t$.
\ENDFOR
\STATE Calculate the estimate of $\boldsymbol{\theta}_0$ by 
\begin{align}\label{est} 
\hat{\boldsymbol{\theta}}_{k}=\mathop{\arg\min}_{\boldsymbol{\theta}\in\Theta} L_k(\boldsymbol{\theta}),
\end{align}
where
$ L_k(\boldsymbol{\theta})=\frac{1}{a_k}\sum_{ t\in I_k}\big\{\mathbb{I}(y_t=1)\log [1-F(p_t-\boldsymbol{\theta}^\top \boldsymbol{x}_t)]+\mathbb{I}(y_t=0)\log F(p_t-\boldsymbol{\theta}^\top \boldsymbol{x}_t)\big\}.$
\STATE \textbf{Exploitation Phase (Optimal Pricing Policy)}:
\FOR{$t\in I_k':=\{ \ell_k+a_k,...,\ell_{k+1}-1\}$}
\STATE The buyer reveals $\tilde{\boldsymbol{x}}_t$ as shown in Equation (\ref{xt}). Denote $\boldsymbol{x}_t=(\tilde{\boldsymbol{x}}_t^\top, 1)^\top$.
\STATE The seller offers the price $p_t=g(\hat{\boldsymbol{\theta}}_k^\top \boldsymbol{x}_t+\hat{\boldsymbol{\beta}}_k^\top A^{-1}\hat{\boldsymbol{\beta}}_k g'(\hat{\boldsymbol{\theta}}_k^\top \boldsymbol{x}_t))$.
\ENDFOR
\ENDFOR
\end{algorithmic}
\end{algorithm}
Lacking knowledge of the horizon length $T$, we employ the doubling trick \citep{lattimore} to partition the horizon into episodes. Each episode comprises an exploration phase followed by an exploitation phase, as illustrated in Figure \ref{episode}.
     \begin{figure}[h!]
    \centering
    \begin{tabular}{cc}  
        \includegraphics[scale = 0.6]{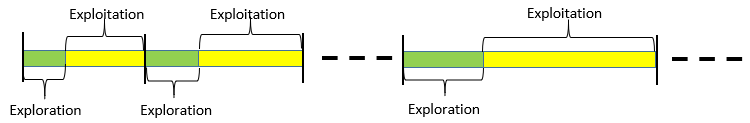}
    \end{tabular}
     \caption{Schematic representation of the segmentation of episodes.}
         \label{episode}
\end{figure}
Algorithm \ref{alg1} requires three input parameters. The first input is the upper bound of market value $B$, which is assumed to be known in Assumption \ref{a3}. This is consistent with the approach used in previous works such as \cite{Fan2022} and \cite{luo2022}. Here,  we only need an upper bound on the price and a rough upper bound $B$ is sufficient. In practice, we can determine the price upper bound $B$ using surveys\footnote{\href{https://online.hbs.edu/blog/post/willingness-to-pay}{https://online.hbs.edu/blog/post/willingness-to-pay}}. By surveying diverse customers and identifying their willingness to pay, we can estimate $B$ as the highest reported value. The second input is the minimum episode length $\ell_0$, which is also aligned with the approach used in \cite{Fan2022} and \cite{luo2022}. The third input is denoted as $C_a$ and is used to determine the length of the exploration phase. In our algorithm, the length of the exploration phase is set to $\lfloor \sqrt{C_a\ell_k} \rfloor$, which differs from previous works such as \cite{Fan2022} and \cite{luo2022} that consider the case of unknown noise distribution. In \cite{Fan2022}, the exploration length is $\lceil (\ell_k d)^{\frac{2m+1}{4m-1}}\rceil$ with $m\geq 2$, while in \cite{luo2022}, it is $\lceil c_1\ell_k^{c_2} \rceil$ for some constants $c_1 > 0$ and $c_2 = 2/3$ or $3/4$. Our approach results in a shorter exploration phase length compared to \cite{Fan2022} and \cite{luo2022} due to the assumption of known noise distribution. This shorter exploration phase leads to a reduced regret, making our algorithm more efficient in the strategic setting. \par
Algorithm \ref{alg1} can been seen as a variant of the explore-then-commit algorithm. During the exploration phase, the seller implements the uniform pricing policy, and the buyers do not manipulate features and reveal $\tilde{\boldsymbol{x}}_t=\tilde{\boldsymbol{x}}_t^0$. Note that by design, prices posted in the exploration phase are independent from the noise $z_t$. The seller collects the true features to obtain an accurate estimate of $\boldsymbol{\theta}_0$ by the maximum likelihood estimation (MLE). During the exploitation phase, the estimated model parameters are fixed, and the seller commits to the optimal pricing policy $g(\cdot)$ by using the parameters obtained in the exploration phase. It is worth mentioning that the seller only discloses the function $g(\cdot)$ but keeps the assessment rule $\hat{\boldsymbol{\theta}}_k$ undisclosed \citep{Bechavod2022}.  The estimator $\hat{\boldsymbol{\theta}}_k$ is derived exclusively from data collected during the exploration phase of the $k$-th episode not all the past exploration phases. Although using data from all exploration phases 1 to $k$-th episodes might enhance finite-sample performance, it does not alter the regret rate, as $\sum_{j=1}^ka_j\leq \sqrt{2}a_k$. Moreover, when the demand parameters are not stationary, it is more practical to estimate $\hat{\boldsymbol{\theta}}_k$ solely based on the data from the exploration phase of the $k$-th episode. Focusing on recent exploration data enables adaptation to parameter changes.

\begin{remark}
The two-phase exploration-exploitation mechanism in Algorithm \ref{alg1} is commonly employed in the dynamic pricing literature. Our uniform pricing policy in the exploration stage aligns with \cite{Golrezaei2019,luo2022,Fan2022}, where prices during the exploration phase are also set from the uniform distribution to facilitate parameter estimation. It is important to note that in the exploration stage, prices are not necessarily to be entirely random. For instance, in \cite{Broder2012} and \cite{Rustamdjan}, fixed price sequences are offered in the exploration phase to avoid the uninformative price. Moreover, adaptive model-based exploration is also feasible by utilizing some prior information in the Thompson Sampling pricing algorithm \citep{jain2024effective}. For simplicity, we focus on the uniform exploration in this paper and leave a complete investigation of such adaptive exploration for future work. 
\end{remark}

Different from existing dynamic pricing works, an important distinction of our policy is the consideration of the strategic behavior during the exploitation phase, which leads to a significantly improved regret bound. Our proposed policy is both practical and reasonable. When introducing new products to the market, companies often conduct price experiments to assess the impact of varying prices, particularly when historical data is lacking to offer valuable insights\footnote{\href{https://www.corrily.com/blog/price-experimentation-101}{https://www.corrily.com/blog/price-experimentation-101}}. This process aligns with the exploration phase. Following this experimentation, an estimated optimal policy will be implemented for exploitation purposes.

Furthermore, in the supplementary materials, we introduce an extension of our policy known as the strategic $\epsilon$-greedy pricing policy. This approach integrates both exploration and exploitation phases, where exploration takes place with probability $\epsilon$ and exploitation with probability $1-\epsilon$ at each time. We also include additional experiments to assess the performance of this extended policy.

\section{Strategic Pricing with Unknown Marginal Cost}\label{sec4}
In Algorithm \ref{alg1}, the seller has knowledge of the marginal cost $A$. In this section, we extend Algorithm \ref{alg1} to the scenario where the marginal cost $A$ is unknown. We first introduce how to match the true features and the manipulated features for the repeated buyers. Then we present the strategy to handle the unknown marginal cost. Finally, we develop the  strategic pricing policy with unknown marginal cost.\par
\subsection{Matching of True Features and Manipulated Features}\label{4.1}
We assume that some buyers return to make repeated purchases, which is common in real-world scenarios such as the mentioned Home Depot and loan application examples. The seller keeps track of an unique identification number (ID, denoted by $e$) assigned to each buyer, such as the account email in the Home Depot example or the social security number in the loan example. By recording the ID of each buyer, the seller can distinguish between different buyers and keep track of their reported features. To develop a strategic dynamic pricing algorithm in the absence of the known marginal cost, we introduce a concept of the repeat buyer rate $\tau$, which is also used in previous literature \citep{Funk2009, Rajat2023}.
\begin{definition}\label{def}
The repeat buyer rate $\tau$ is the proportion of buyers who have made purchases during both the exploration and the exploitation phases.
\end{definition}
The presence of a repeat buyer rate $\tau>0$ allows the seller to acquire both the original features and the manipulated features of the same buyer. During the exploration phase, the seller collects the original feature $\tilde{\boldsymbol{x}}_{t}^0$ along with the corresponding unique ID $e_{t}$ for each buyer. These pairs $(e_{t}, \tilde{\boldsymbol{x}}_{t}^0)$ are recorded by the seller. In the exploitation phase, the seller obtains the manipulated feature $\tilde{\boldsymbol{x}}_{t}$ and the corresponding ID $e_{t}$, and again records the pairs $(e_{t}, \tilde{\boldsymbol{x}}_{t})$. By matching the unique ID $e_{t}$ obtained from both phases, the seller can establish the feature pair $(\tilde{\boldsymbol{x}}_{t}^0, \tilde{\boldsymbol{x}}_{t})$ for the same buyer. This matching process allows the seller to link the original and manipulated features for individual buyers.
\subsection{Strategy for Unknown Marginal Cost}\label{4.2}
In this section, we introduce the strategy to handle the unknown marginal cost $A$. Let $\hat{\boldsymbol{\theta}}$ be an estimate of $\boldsymbol{\theta}_0$. For the matched pair $(\tilde{\boldsymbol{x}}_{t}^0, \tilde{\boldsymbol{x}}_{t})$,  using Equation (\ref{xt}), we can obtain
\begin{equation}\label{linear}
 \begin{aligned}
\tilde{\boldsymbol{x}}_t-\tilde{\boldsymbol{x}}_{t}^0&=- A^{-1}\boldsymbol{\beta}_0 g'(\boldsymbol{\theta}_0^\top \boldsymbol{x}_t)=- A^{-1}\boldsymbol{\beta}_0 g'(\hat{\boldsymbol{\theta}}^\top \boldsymbol{x}_t)+A^{-1}\boldsymbol{\beta}_0 [g'(\hat{\boldsymbol{\theta}}^\top \boldsymbol{x}_t)-g'(\boldsymbol{\theta}_0^\top \boldsymbol{x}_t)].
\end{aligned}   
\end{equation}
To simplify Equation (\ref{linear}), we introduce the following new variables,
\begin{equation}\label{gamma1}
\begin{aligned}
 &\boldsymbol{\delta}_t:=\tilde{\boldsymbol{x}}_t-\tilde{\boldsymbol{x}}_{t}^0\in\mathbb{R}^d,
&&\boldsymbol{\epsilon}_t:=A^{-1}\boldsymbol{\beta}_0[g'(\hat{\boldsymbol{\theta}}^\top \boldsymbol{x}_t)-g'(\boldsymbol{\theta}_0^\top \boldsymbol{x}_t)]\in\mathbb{R}^d,\\
&\boldsymbol{\gamma}:= - A^{-1}\boldsymbol{\beta}_0\in\mathbb{R}^d,
& &u_t:=g'(\hat{\boldsymbol{\theta}}^\top \boldsymbol{x}_t)\in\mathbb{R}.
\end{aligned}    
\end{equation}
By introducing these new variables, Equation (\ref{linear}) can be rewritten as the following $d$-dimensional equation,
\begin{equation}\label{liear2}
    \boldsymbol{\delta}_t=\boldsymbol{\gamma} u_t+\boldsymbol{\epsilon}_t.
\end{equation}
In Equation (\ref{liear2}), $u_t$ is known by the seller, and $\boldsymbol{\delta}_t$ can be obtained using the matched pairs $(\tilde{\boldsymbol{x}}_{t}^0, \tilde{\boldsymbol{x}}_{t})$ recorded by the seller.
Let $\boldsymbol{\gamma}_j$ and $\boldsymbol{\epsilon}_{jt}$ be the $j$-th ($j\in [d]$) component of $\boldsymbol{\gamma}$ and $\boldsymbol{\epsilon}_{t}$, respectively.
The $j$-th component equation of (\ref{liear2}) can be written as
$\boldsymbol{\delta}_{jt}=\boldsymbol{\gamma}_j u_t+\boldsymbol{\epsilon}_{jt}.$
Assume that we obtain $n$ repeated samples $\{(\boldsymbol{\delta}_{1}, u_1),...,(\boldsymbol{\delta}_{n}, u_n)\}$, and define
$ \boldsymbol{\Delta}_j:=(\boldsymbol{\delta}_{j1},...,\boldsymbol{\delta}_{jn})^\top, 
 \boldsymbol{u}:=(u_1,...,u_n)^\top. $
 By the least square method, we obtain the estimation of $\boldsymbol{\gamma}_j$ as
 \begin{equation}\label{ols}
\hat{\boldsymbol{\gamma}}_j=\frac{\boldsymbol{u}^{\top}\boldsymbol{\Delta}_j}{\boldsymbol{u}^\top \boldsymbol{u}}.     
 \end{equation}
 This $\hat{\boldsymbol{\gamma}}_j$ can be used in our pricing policy to handle the case of unknown marginal cost $A$. Note that if we directly estimate the unknown $A$, we would need to estimate a total of $(d^2+d)/2$ elements ($A$ is a $d \times d$ symmetric matrix). However, with our strategy, we can reduce the number of elements to be estimated to $d$ by using Equation (\ref{ols}). This significantly reduces the complexity of the estimation process and makes it more computationally feasible.
 \subsection{Pricing Policy with Unknown Marginal Cost}\label{4.3}
By leveraging the results of Section \ref{4.1} and Section \ref{4.2}, we are ready to introduce the details of the strategic dynamic pricing policy with unknown marginal cost in Algorithm \ref{alg2}. 
\begin{algorithm}[!ht]
\caption{Strategic Dynamic Pricing with Unknown Marginal Cost}
\label{alg2}
\begin{algorithmic}[1]
\STATE \textbf{Input}: $B, \ell_0, C_a, \mathcal{E}_1=\emptyset, \mathcal{E}_2=\emptyset, \mathcal{M}=\emptyset$
\FOR{each episode $k=1,2,...$}
\STATE Set the length of $k$-th episode as $\ell_k= 2^{k-1}\ell_0$, and $a_k=\lfloor\sqrt{C_a\ell_k}\rfloor$.
\STATE \textbf{Exploration Phase (Uniform Pricing Policy)}:
\FOR{$t\in I_k:=\{ \ell_k,...,\ell_k+a_k-1\}$}
\STATE The buyer reveals $\tilde{\boldsymbol{x}}_t=\tilde{\boldsymbol{x}}_t^0$.
\STATE The seller sets a price $p_t$ randomly from Unif$(0,B)$. 
\STATE Observe a binary response $y_t$.
\STATE $\mathcal{E}_1 \gets \mathcal{E}_1 \cup \{e_t: \tilde{\boldsymbol{x}}_t^0\}$.
\IF{$e_t$ is in $\mathcal{E}_2$}
\STATE $\mathcal{M}\gets \mathcal{M}\cup \{(\tilde{\boldsymbol{x}}_t^0, \tilde{\boldsymbol{x}}_t)\}$
\ENDIF
\ENDFOR
\STATE Calculate the estimate $\hat{\boldsymbol{\theta}}_k$ of $\boldsymbol{\theta}_0$ by (\ref{est}).
\STATE \textbf{Exploitation Phase (Optimal Pricing Policy)}:
\FOR{$t\in I_k':=\{ \ell_k+a_k,...,\ell_{k+1}-1\}$}
\STATE The buyer reveals $\tilde{\boldsymbol{x}}_t$ as shown in Equation (\ref{xt}).
\STATE $\mathcal{E}_2 \gets \mathcal{E}_2 \cup \{e_t: \tilde{\boldsymbol{x}}_t\}$.
\IF{$e_t$ is in $\mathcal{E}_1$}
\STATE $\mathcal{M}\gets \mathcal{M}\cup \{(\tilde{\boldsymbol{x}}_t^0, \tilde{\boldsymbol{x}}_t)\}$.
\ENDIF
\STATE The seller sets a price $p_t$ by Algorithm \ref{alg3}.
\ENDFOR
\ENDFOR
\end{algorithmic}
\end{algorithm}

Algorithm \ref{alg2} takes six input parameters. The first three inputs, $B, \ell_0,$ and $C_a$, are the same as those used in Algorithm \ref{alg1}. The set $\mathcal{E}_1$ is used to store the IDs and true features of buyers $(e_t, \tilde{\boldsymbol{x}}_t^0)$ collected during the exploration phase. The set $\mathcal{E}_2$ stores the IDs and manipulated features of buyers $(e_t, \tilde{\boldsymbol{x}}_t)$ collected during the exploitation phase. The set $\mathcal{M}$ stores the matched pairs $(\tilde{\boldsymbol{x}}_t^0, \tilde{\boldsymbol{x}}_t)$ by linking the unique ID $e_t$ from sets $\mathcal{E}_1$ and $\mathcal{E}_2$. These sets play a crucial role to link the true and manipulated features in the exploration and exploitation phases to enable the cost parameter estimation shown in Section \ref{4.2}.
The core principle of Algorithm \ref{alg2} remains similar to that of Algorithm \ref{alg1}, as they both employ a two-phase mechanism consisting of an exploration phase and an exploitation phase. 

The distinguishing feature of Algorithm \ref{alg2} lies in its handling of the unknown marginal cost $A$. To address this challenge, the algorithm 
uses the matched $\tilde{\boldsymbol{x}}_t^0$ and $\tilde{\boldsymbol{x}}_t$ to learn the pricing parameter $\boldsymbol{\gamma}$. It is important to highlight that the number of matched pairs $(\tilde{\boldsymbol{x}}_t^0, \tilde{\boldsymbol{x}}_t)$ is controlled by the repeat buyer rate $\tau$. The higher the $\tau$, the more different buyers make repeated purchases, resulting in a larger pool of matched pairs $(\tilde{\boldsymbol{x}}_t^0, \tilde{\boldsymbol{x}}_t)$. This increase in matched pairs leads to a more precise estimate of $\boldsymbol{\gamma}$, which in turn enhances the effectiveness of the pricing strategy. By leveraging the matched pairs and adjusting the repeat buyer rate, Algorithm \ref{alg2} can effectively learn the pricing parameter $\boldsymbol{\gamma}$ and implement the pricing policy in the absence of knowledge about the marginal cost $A$. 

\begin{algorithm}[!ht]
\caption{Calculation of $p_t$}
\label{alg3}
\begin{algorithmic}[1]
\STATE \textbf{Input}: $e_t, \mathcal{E}_1, \mathcal{E}_2, \mathcal{M}, \tilde{\boldsymbol{x}}_t, \hat{\boldsymbol{\theta}}_k$.
\IF{$e_t\in \mathcal{E}_1\cap\mathcal{E}_2$}
\STATE $p_t=g(\hat{\boldsymbol{\theta}}_k^\top \boldsymbol{x}_t^0)$, where $(\tilde{\boldsymbol{x}}_t^0, \tilde{\boldsymbol{x}}_t)\in \mathcal{M}$.
\ELSE 
\IF{  $\mathcal{M} \neq \emptyset$}
\STATE Obtain $\hat{\boldsymbol{\gamma}}=(\hat{{\boldsymbol{\gamma}}}_1,\cdots, \hat{{\boldsymbol{\gamma}}}_d)^\top$ by (\ref{ols}).
\STATE The seller offers the price $p_t=g(\hat{\boldsymbol{\theta}}_k^\top \boldsymbol{x}_t-\hat{\boldsymbol{\beta}}_k^\top \hat{\boldsymbol{\gamma}} g'(\hat{\boldsymbol{\theta}}_k^\top \boldsymbol{x}_t))$.
\ELSE 
\STATE The seller offers the price $p_t=g(\hat{\boldsymbol{\theta}}_k^\top \boldsymbol{x}_t)$.
\ENDIF
\ENDIF
\STATE \textbf{Output}: $p_t$.
\end{algorithmic}
\end{algorithm}
The detail of the calculation of the price $p_t$ during the exploitation phase is shown in Algorithm \ref{alg3}. If the original feature $\tilde{\boldsymbol{x}}_t^0$ is recorded in the seller's system, the price can be directly determined as $p_t = g(\hat{\boldsymbol{\theta}}_k^\top \boldsymbol{x}_t^0)$. On the other hand, if the original feature $\tilde{\boldsymbol{x}}_t^0$ is not recorded, the seller utilizes the estimated $\boldsymbol{\gamma}$ to determine the price. In the absence of an estimation of $\boldsymbol{\gamma}$, the price is set as $p_t = g(\hat{\boldsymbol{\theta}}_k^\top \boldsymbol{x}_t)$. Once the estimation of $\boldsymbol{\gamma}$ becomes available, the price is determined as $p_t = g(\hat{\boldsymbol{\theta}}_k^\top \boldsymbol{x}_t - \hat{\boldsymbol{\beta}}_k^\top \hat{\boldsymbol{\gamma}} g'(\hat{\boldsymbol{\theta}}_k^\top \boldsymbol{x}_t))$.

\section{Regret Analysis}\label{sec5}
In this section, we analyze the regret of the proposed pricing policy when the marginal cost $A$ is known (Section \ref{reg1}) and unknown (Section \ref{reg2}).
\subsection{Regret Analysis under Known Marginal Cost}\label{reg1}
We consider the strategic dynamic pricing policy in Algorithm \ref{alg1} with strategic buyers. 
We first introduce two important measures to characterize the properties related to the function $F(\cdot)$. We define the "steepness" of the function $F(\cdot)$ as
\begin{equation}\label{cup}
C_{up}=\mathop{\sup}_{\omega\in[-W, B]}\mathop{\max}\{\log' F(\omega), -\log' [1-F(\omega)]\}
\end{equation}
and the "flatness" of the function log$F(\cdot)$ as
\begin{equation}\label{cdown}
 C_{down}=\mathop{\inf}_{\omega\in[-W, B]}\mathop{\min}\{-\log''(1-F(\omega)), -\log(F''(\omega))\}.
\end{equation}
We then present a lemma that establishes an upper bound on the estimation error of $\boldsymbol{\theta}_0$ at the end of the exploration phase within each episode. 
\begin{lemma} \label{lemma11}
Suppose that Assumptions \ref{a0}, \ref{a3}, \ref{a1} and \ref{a2} hold. Let $\hat{\boldsymbol{\theta}}_k$ be the estimator from (\ref{est}), $a_k$ be the exploration length in the $k$-th episode. We have
$$\mathbb{E}\|\hat{\boldsymbol{\theta}}_k-\boldsymbol{\theta}_0\|_2^2\leq \frac{2(d+1)C_{up}^2}{C^2_{down}\lambda_{min}(a_k+1)},$$
where $C_{up}$ and $C_{down}$ are defined in (\ref{cup}) and (\ref{cdown}), respectively.
\end{lemma}
Lemma \ref{lemma11} shows that the expected squared estimation error of $\boldsymbol{\theta}_0$ decreases as the exploration length $a_k$ increases. As $a_k$ increases, the number of the samples used to estimate $\boldsymbol{\theta}_0$ becomes larger, leading to a better estimation accuracy for $d+1$ parameters. However, when $a_k$ is too large, the pricing policy will over-explore and incur a large regret. By using the optimal choice of $a_k$ in our Algorithm \ref{alg1}, we establish a tight upper bound on the regret for the proposed strategic dynamic pricing policy with known $A$.
\begin{theorem}\label{theory2}
Assume that the marginal cost $A$ is known by the seller. Under Assumptions  \ref{assc}, \ref{a0}, \ref{a3}, \ref{a1} and \ref{a2}, using the strategic pricing policy (Algorithm \ref{alg1}), there exist constants $C_1^*>0$ and  $C_2^*>0$ such that the total expected regret satisfies 
$$Regret_{\pi}(T)\leq \sqrt{C_1^*+\frac{C_2^*}{\lambda_{Amin}^2}}\sqrt{(d+1)T}.$$
\end{theorem}
The constants $C_1^*$ and $C_2^*$ in the regret bound only depend on some absolute constants derived from the assumptions. To read the regret bound, we break it into three elements.
First, the regret bound is influenced by the minimum eigenvalue $\lambda_{Amin}$ of the marginal cost $A$, serving as an indicator of the manipulation capability. As expressed in (\ref{xt}), the extent of deviation between the manipulated feature $\tilde{\boldsymbol{x}}_t$ and the original feature $\tilde{\boldsymbol{x}}_t^0$ is associated with the marginal cost $A$. When the minimum eigenvalue $\lambda_{Amin}$ decreases, the deviation between $\boldsymbol{x}_t$ and $\boldsymbol{x}_t^0$ increases, making the pricing problem more challenging and resulting in a higher regret.  Second,  the regret bound depends on the dimension of the features at the rate $\sqrt{d}$. A larger feature dimension makes the estimation of the parameters more difficult, leading to a larger regret. Third, the regret bound depends on the time length at the rate of $\sqrt{T}$. In comparison, Theorem \ref{theory3} demonstrates that the non-strategic pricing policy has a regret bound of at least $\Omega(T)$. Consequently, our proposed strategic dynamic pricing policy, which accounts for the strategic behavior of buyers, outperforms the non-strategic pricing policy in terms of minimizing regret.

\begin{remark}
In traditional bandit problems, the explore-then-commit algorithm achieves an upper regret bound of $O(T^{2/3})$ \citep{lattimore}. However, in pricing problems, the explore-then-commit algorithm yields an upper regret bound of $O(\sqrt{T})$, attributed to the fact that $p_t^* \in \arg\max_p r_t(p)$ and thus $r'_t(p_t^*)=0$, where $r_t(p)=p[1-F(p-\boldsymbol{\theta}_0^\top \boldsymbol{x}_t^0)]$. This special structure does not typically hold in traditional bandit problems. For a detailed discussion on the upper regret bound, please refer to the supplementary materials.
\end{remark}
\subsection{Regret Analysis under Unknown Marginal Cost}\label{reg2}
In this section, we analyze the regret of the strategic dynamic pricing policy with the unknown marginal cost. We first provide an upper bound on the estimation error of $\boldsymbol{\gamma}=-A^{-1}\boldsymbol{\beta}_0$.
\begin{lemma}\label{lemmaunA}
Suppose that Assumptions  \ref{assc}, \ref{a0}, \ref{a3}, \ref{a1} and \ref{a2} hold, and the latest sample used in (\ref{ols}) is obtained in the $k$-th episode. Let $\ell_k$ be the total length of the $k$-th episode,  $\tau$ be the repeat buyer rate defined in Definition \ref{def}. We denote $\hat{\boldsymbol{\gamma}}=(\hat{\boldsymbol{\gamma}}_1,...,\hat{\boldsymbol{\gamma}}_d)^\top$ as the estimate of $\boldsymbol{\gamma}$, where $\hat{\boldsymbol{\gamma}}_j\ (j\in [d])$ is estimated from (\ref{ols}). There exists constant $C_{\gamma}^*>0$ such that for $k>1$
\begin{align*}
\mathbb{E}\|\hat{\boldsymbol{\gamma}}-\boldsymbol{\gamma}\|_2^2<\frac{C_{\gamma}^*(d+1)}{\tau \sqrt{\ell_k}}.
\end{align*}
\end{lemma}
Lemma \ref{lemmaunA} reveals the estimate error of $\boldsymbol{\gamma}$ scales inversely with the repeat buyer rate. This implies that a higher repeat buyer rate leads to a lower estimation error, as more samples can be obtained to estimate $\boldsymbol{\gamma}$ when $\tau$ is higher. Now we establish an upper bound on the total expected regret for the strategic pricing policy in the case of an unknown marginal cost.
\begin{theorem}\label{theory5}
Assume that the marginal cost $A$ is unknown by the seller. Under Assumptions \ref{a0}, \ref{a3}, \ref{a1} and \ref{a2}, using the strategic pricing policy (Algorithm \ref{alg2}), there exist constants $C_3^*>0$ and $C_4^*>0$ such that for $k>1$, the total expected regret satisfies 
$$Regert_{\pi}(T)< \bigg[C_3^*+\frac{C_4^*\sqrt{d+1}}{\tau\lambda_{Amin}^2}\bigg]\sqrt{(d+1)T}.$$
\end{theorem}
The regret bound is influenced by several interesting factors, including $\lambda_{Amin}$, $d$, $T$, and $\tau$. The relationship between the regret bound and the first three factors $\lambda_{Amin}$, $d$, and $T$ is similar to that established in Theorem \ref{theory2}. Here, we focus on analyzing the impact of the repeat buyer rate $\tau$ on the regret bound. The parameter $\tau$ represents the proportion of buyers who have made purchases during both the exploration and the exploitation phases.  A higher value of $\tau$ results in more repeat buyers, providing more samples for the estimation of $A^{-1}\boldsymbol{\beta}_0$. This increase in the sample size leads to a more accurate estimation of $\boldsymbol{\gamma}$, as indicated by Lemma \ref{lemmaunA}. Consequently, the seller obtains a more precise estimate of the true feature $\boldsymbol{x}_t^0$, which translates to a lower regret. Theorem \ref{theory5} establishes that our proposed strategic pricing policy, even in the absence of prior knowledge regarding the marginal cost, attains the same regret upper bound of $O(\sqrt{T})$ as demonstrated in Theorem \ref{theory2}.

\section{Experiments}\label{sec6}


In this section, we empirically evaluate the performance of our proposed strategic dynamic pricing policies and compare them with the benchmark method. We first conduct simulation studies to validate the theoretical results and investigate the impacts of key factors on our policies, and then evaluate the performance of our policies using real-world data. We present sensitivity tests on the hyperparameters in the supplementary materials. Additionally, the experimental results on the strategic $\epsilon$-greedy pricing policy and the heterogeneity of marginal costs are also detailed in the supplementary materials. All experimental results are derived from 100 independent runs.

\subsection{Justification of Theoretical Results}\label{sec6.1}
We consider the dimension of the features $d=2$, with the true parameter $\boldsymbol{\theta}_0=(1/3, 2/3, 1/2)^\top$. The covariates $x_1$ and $x_2$ are both independently and identically distributed from Unif(0, 4). The noise distribution is assume as the normal distribution $z_t \sim N(0,1)$. \par
To implement our algorithm, we divide the time horizon into consecutive episodes, with the length of the $k$-th episode set to $\ell_k=2^{k-1}\ell_0$, where $\ell_0=200$. We further  partition each episode into an exploration phase with length $a_k=\lfloor \sqrt{100\ell_k}\rfloor$, and an exploration phase with length $\ell_k-a_k$. In the exploration phase, we sample $p_t$ from Unif(0, 6) among all the policies, and we obtain the estimate $\hat{\boldsymbol{\theta}}_k$. In the exploitation phase, we implement different policies to compare the performance. 
\begin{itemize}
    \item For the non-strategic pricing policy \citep{Javanmard2019, wang2022}, the price is determined by $p_t=g(\hat{\boldsymbol{\theta}}_k^\top \boldsymbol{x}_t)$. 
    \item For the strategic pricing policy with the known marginal cost, the price is determined by $p_t=g(\hat{\boldsymbol{\theta}}_k^\top \boldsymbol{x}_t+\hat{\boldsymbol{\beta}}_k^\top A^{-1}\hat{\boldsymbol{\beta}}_k g'(\hat{\boldsymbol{\theta}}_k^\top \boldsymbol{x}_t))$ according to Algorithm \ref{alg1}.
    \item  For the strategic pricing policy with the unknown marginal cost, the price $p_t$ is determined according to Algorithm \ref{alg2}. 
\end{itemize}
Among the experiments, we denote the base marginal cost $A_0=\begin{pmatrix}
1/4 & 1/8 \\
1/8 & 1/4
\end{pmatrix}$, and consider different repeat buyer rates $\tau=0.05\%$ and $\tau=0.1\%$.
\subsubsection{Comparison of Strategic and Non-strategic Pricing Policies}\label{6.1.1}
In this section, we compare the performance of the strategic and non-strategic pricing policies. Figure \ref{fig1} shows the regrets of the non-strategic pricing policy, the strategic pricing policies with known and unknown marginal cost $A$. We set $\tau=0.05\%$.\par
\begin{figure}[t!]
    \centering
\begin{tabular}{cc}
        \includegraphics[scale = 0.45]{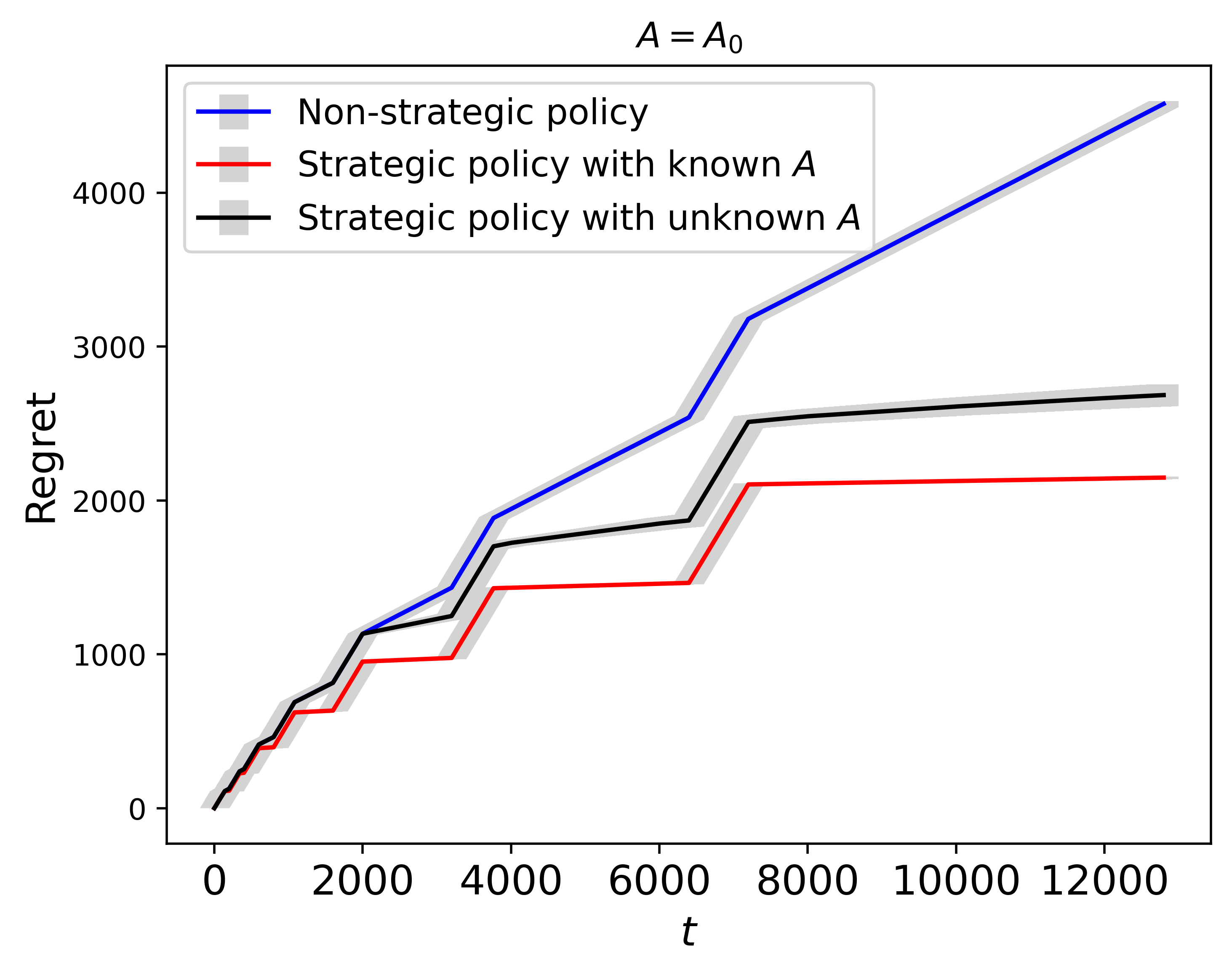}&
        \includegraphics[scale = 0.45]{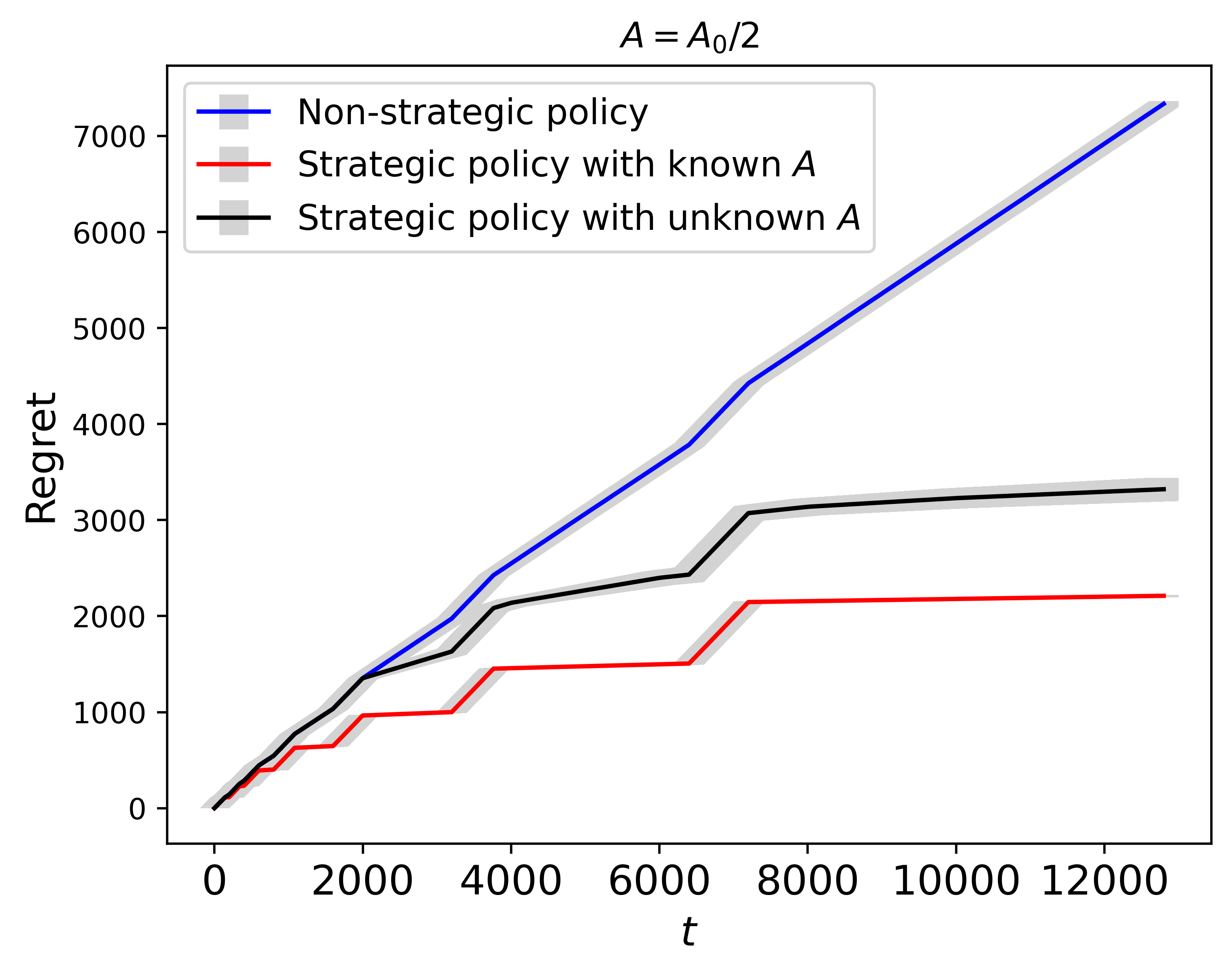}
    \end{tabular}
        \caption{Regret plots for the three policies. } 
    \label{fig1}
\end{figure}
Based on the results presented in Figure \ref{fig1}, it can be concluded that the non-strategic policy generates larger rates of empirical regrets increment compared to the strategic policies under both settings $A=A_0$ and $A=A_0/2$.   The performance comparison between the strategic policy with known $A$ and unknown $A$ depends on the repeat buyer rate $\tau$ of the buyer. For smaller $\tau$, the regret of the strategic policy with unknown $A$ is higher. As  $\tau$ increases, the gap of regrets between these two policies decreases. This phenomenon can be attributed to two reasons. Firstly, with a larger $\tau$, the estimate of $\boldsymbol{\gamma}$ becomes more accurate, leading to a lower regret. Secondly, a larger $\tau$ indicates the presence of more returning buyers. In the strategic policy with unknown marginal cost $A$, the seller is assumed to record the information of the buyer. If the buyer's information $(\tilde{\boldsymbol{x}}_t^0, \tilde{\boldsymbol{x}}_t)$ is stored in the seller's system,  the seller can set the price based on $x_t^0$, which also results in a smaller regret. It indicates that collecting and storing buyer information thus helps the seller increase the profit, which has a great practical significance.\par
\subsubsection{Impact of Marginal Cost on Strategic Pricing Policy}
In this section, we study the impact of the marginal cost on the strategic pricing policy. We set $\tau=0.05\%$.
\begin{figure}[t!]
    \centering
    \begin{tabular}{cc}  
        \includegraphics[scale = 0.450]{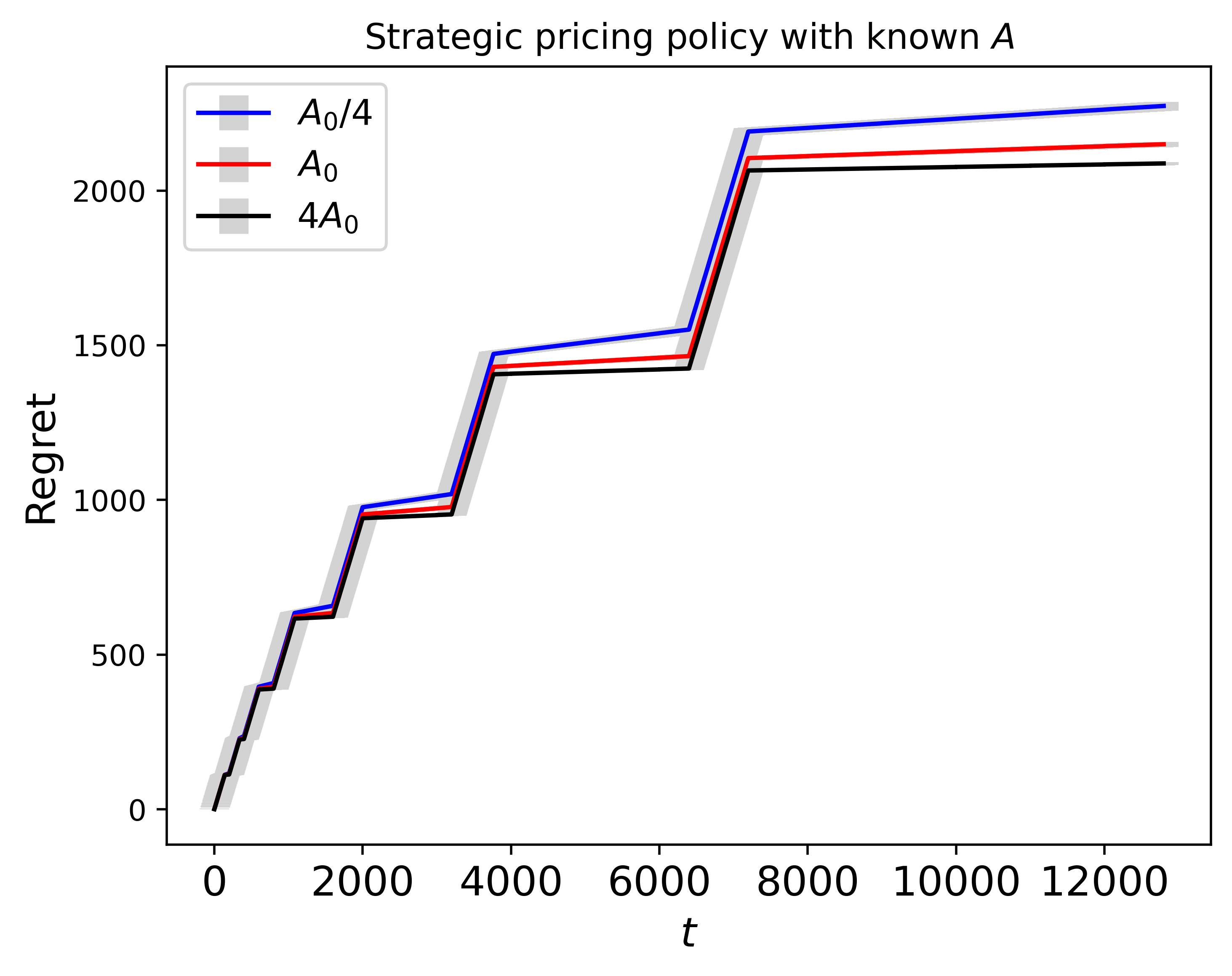}&
        \includegraphics[scale = 0.450]{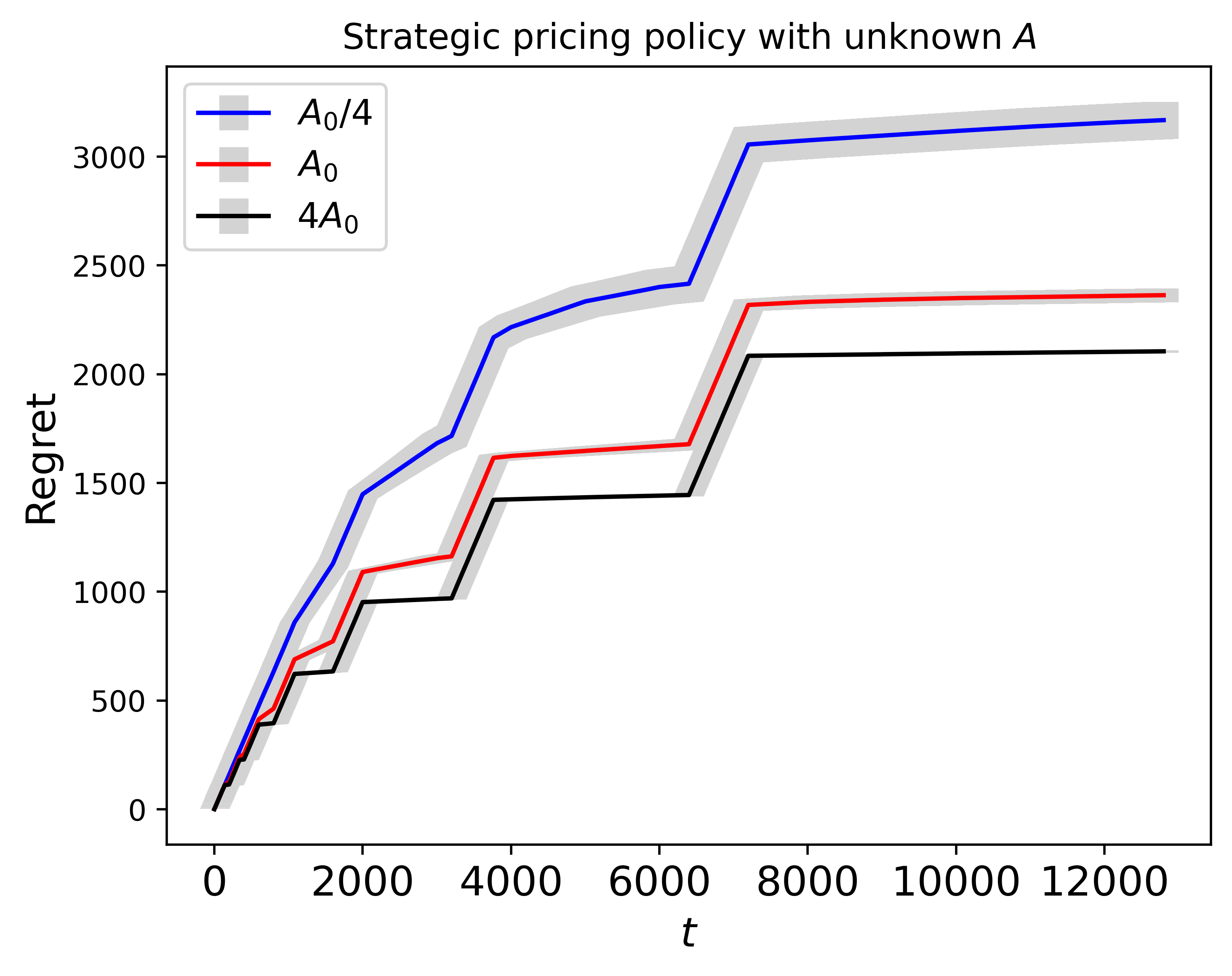}
    \end{tabular}
     \caption{Regret plots for different marginal costs. }
         \label{fig2}
\end{figure}
The marginal costs, denoted as $A_0/4$, $A_0$, and $4A_0$, are intentionally designed to have varying minimum eigenvalues. Specifically, $A_0/4$ has the smallest minimum eigenvalue, while $4A_0$ has the largest minimum eigenvalue. In Figure \ref{fig2}, we observe the impact of different marginal costs on the regret. It is evident that the pricing policy based on the marginal cost $4A_0$ yields the lowest regret, whereas the policy based on $A_0/4$ results in the highest regret, regardless of whether the marginal cost is known or unknown. This observation aligns with the findings of Theorems \ref{theory2} and \ref{theory5}.
\subsubsection{ Impact of Repeat Buyer Rate on Strategic Pricing Policy}
In this section, we investigate the influence of the repeat buyer rate $\tau$ on the strategic pricing policy with an unknown marginal cost. Figure \ref{fig3} shows that the regret of the strategic pricing policy, in the presence of an unknown marginal cost $A$, with varying repeat buyer rate $\tau$'s. Comparing the results at the same marginal cost, we observe that the regret is higher when $\tau=0.05\%$ compared to when $\tau=0.1\%$. This observation is consistent with the conclusions drawn from Theorem \ref{theory5}. The repeat buyer rate $\tau$ plays a crucial role in determining the effectiveness of the strategic pricing policy under the unknown marginal cost.
\begin{figure}[t!]
    \centering
    \begin{tabular}{cc}  
        \includegraphics[scale = 0.450]{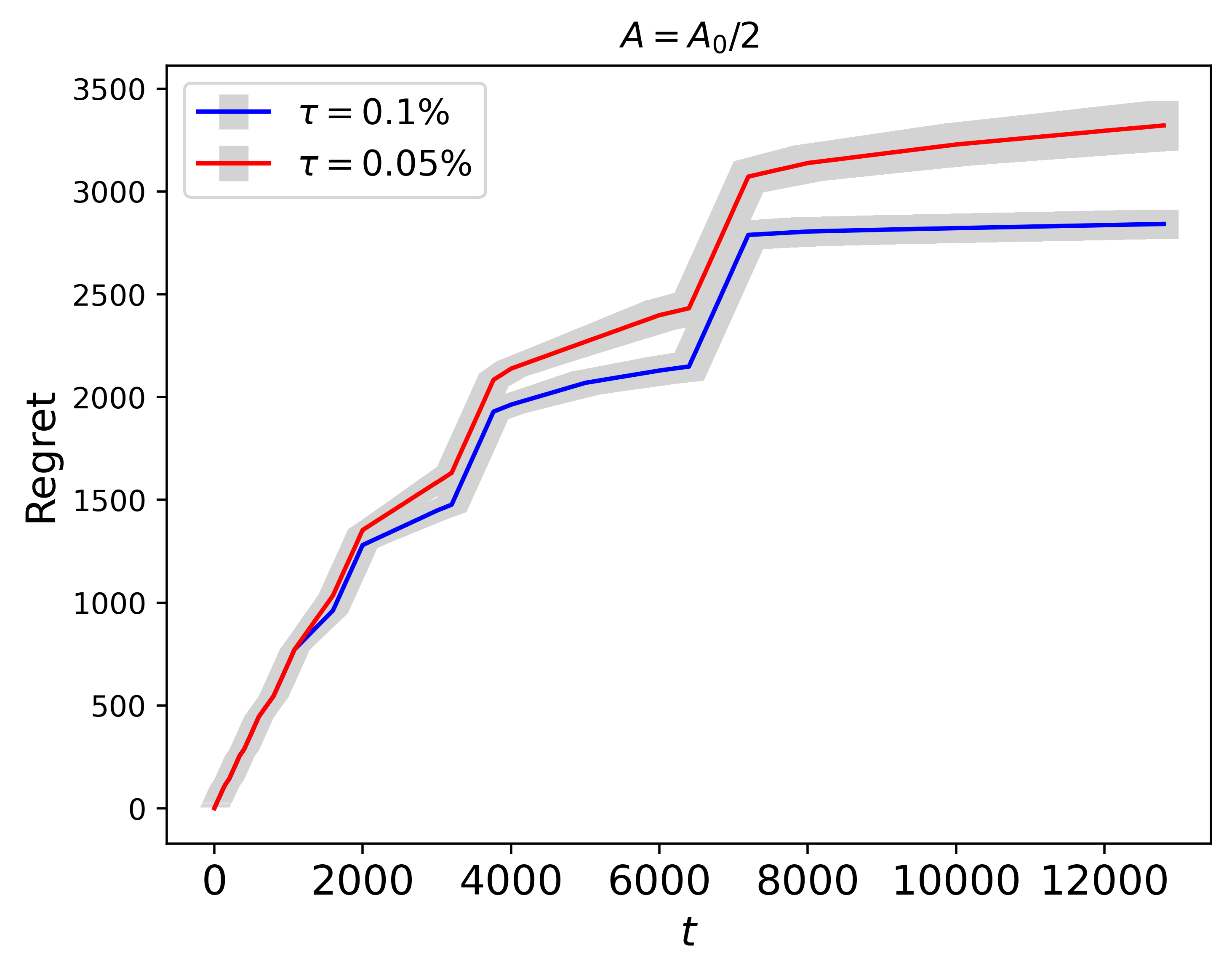}&
        \includegraphics[scale = 0.450]{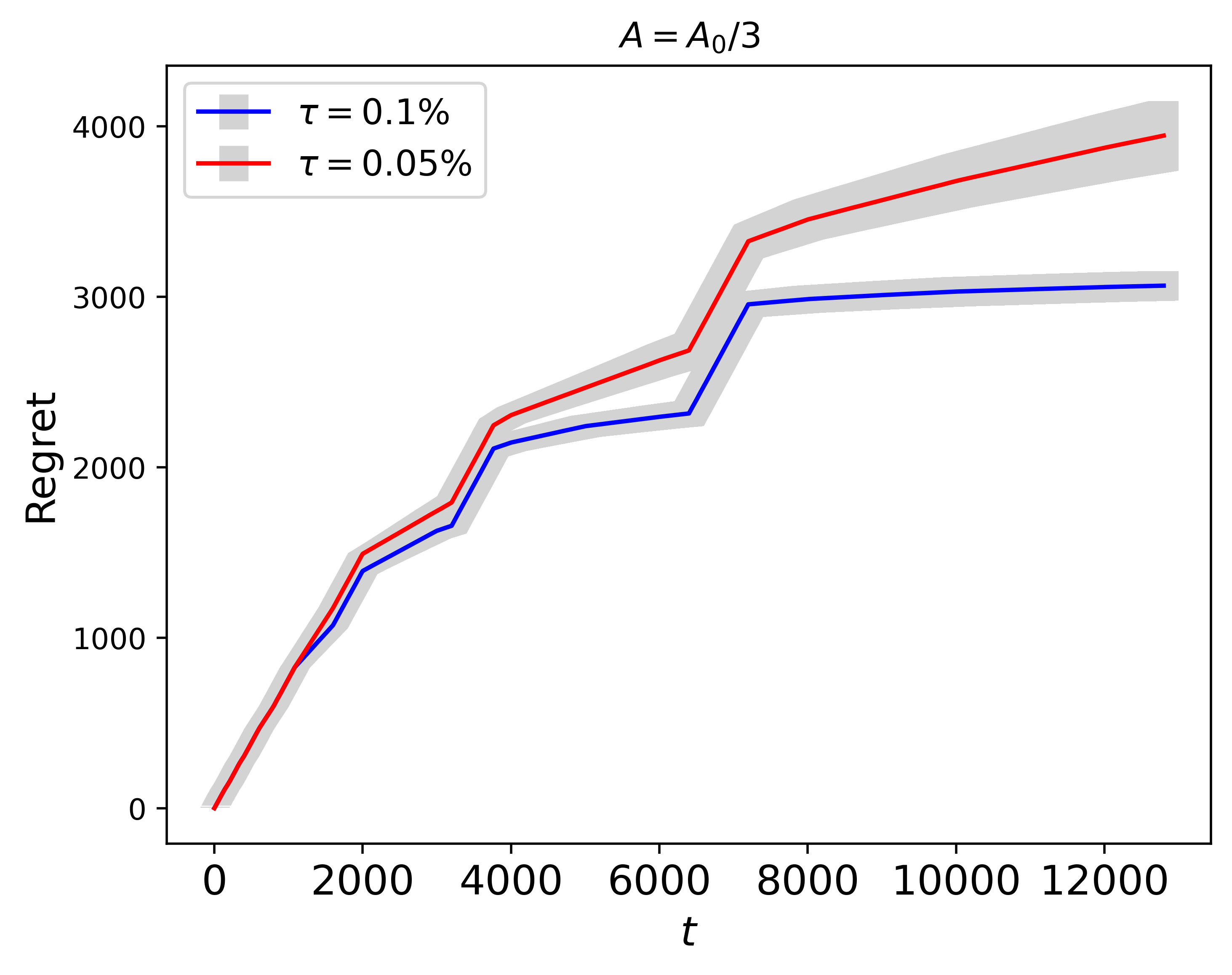}
    \end{tabular}
     \caption{Regret plots for strategic pricing policy with the unknown marginal cost at different repeat buyer rates.}
         \label{fig3}
\end{figure}

\subsection{Real Application} \label{sec6.3}
We explore the efficiency of our proposed policy on a real-world auto loan dataset provided by the Center for Pricing and Revenue Management at Columbia University. This dataset has been used by several dynamic pricing works \citep{Phillips2015,Ban2021,Bastani2022,wang2022, Fan2022,luo2023}. It contains 208,805 auto loan applications received between July 2002 and November
2004. For each application, we observe some loan-specific features such as the amount of loan, the borrower’s information. The dataset also records purchasing decision of the borrowers. We adopt the features used in \citet{Ban2021,Fan2022, luo2023} and consider the following four features: the loan amount approved, FICO score, prime rate, and the competitor’s rate. The price $p$ of a loan is computed by $p=$Monthly Payment $\times\sum_{t=1}^{\text{Term}}(1+\text{Rate})^{-t}-$Loan Amount, and the rate is set at $0.12\%$ \citep{Fan2022,luo2022}. 

Numerous methods have been provided for detecting feature manipulation \citep{Jerzy2021,Jiang2021,Khaled2021,Ali2022,Gu2022,Liao2022}, including supervised, unsupervised, semi-supervised methods and graph-based methods \citep{Waleed}. The conventional approach involves developing supervised models using datasets comprising customer information and labels for loan feature manipulation.  For an in-depth discussion on the detection of feature manipulation, please refer to the supplementary materials.

We acknowledge that online responses to any dynamic pricing strategy are not available unless a real online experiment is conducted. To address this issue, we adopt the calibration approach used in \citet{Ban2021, wang2022, Fan2022,luo2023} to first estimate the binary choice model using the entire dataset and leverage it as the ground truth for our online evaluation. We randomly sample 12,800 applications from the original dataset for 20 times and apply the policies to each of the 20 replications and then record the average cumulative regrets. In the experiment, we assume $z_t\sim N(0, 1)$, and $A=\begin{pmatrix}
A_0 & \boldsymbol{0} \\
\boldsymbol{0} & A_0
\end{pmatrix}$, and set $B=3, C_a=100, \ell_0=200$ and $\tau=0.1\%$.\par
\begin{figure}[t!]
    \centering
    \begin{tabular}{ccc}  
        \includegraphics[scale = 0.60]{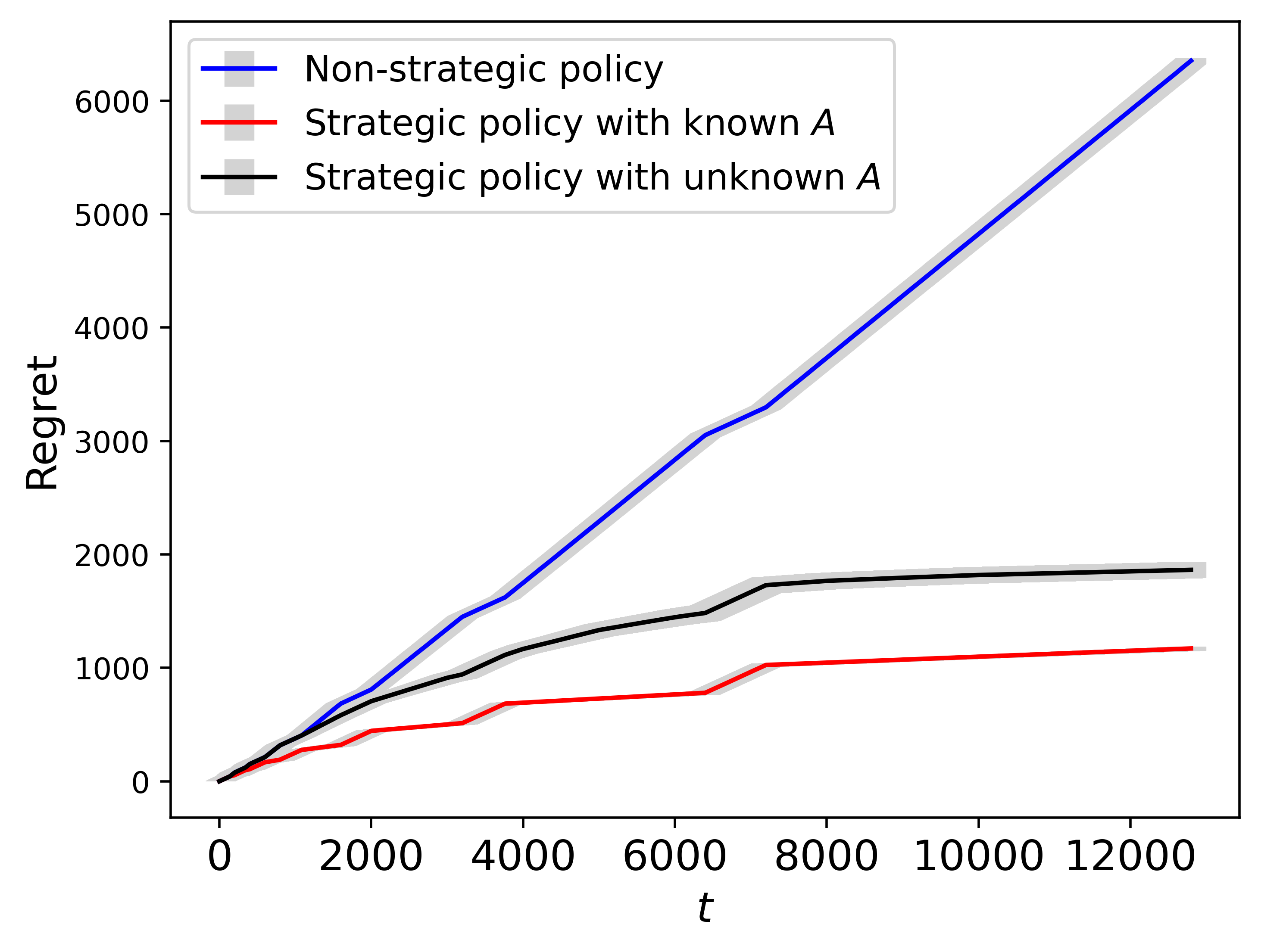}
    \end{tabular}
     \caption{Regret plots for the three policies on the real data.}
         \label{fig7}
\end{figure}
Figure \ref{fig7} depicts the cumulative regret of the non-strategic pricing policy compared to our proposed strategic pricing policies. It is evident from the plot that the cumulative regret of the non-strategic policy increases at a much faster rate compared to our strategic policies. This observation aligns with our previous findings in the simulated data. The strategic policies, which take into account the potential manipulation behavior of buyers, outperform the non-strategic policy in terms of the cumulative regret. 

\section*{Acknowledgment}
The authors thank the editor Professor Annie Qu, the associate editor and two anonymous reviewers for their valuable comments and suggestions which led to a much improved paper. Will Wei Sun’s research was partially supported by National Science Foundation (Award 2217440). Zhaoran Wang acknowledges National Science Foundation (Awards 2048075, 2008827, 2015568, 1934931), Simons Institute (Theory of Reinforcement Learning), Amazon, J.P. Morgan, and Two Sigma for their supports. Zhuoran Yang acknowledges Simons Institute (Theory of Reinforcement Learning) for their support. Any opinions, findings, and conclusions expressed in this material are those of the authors and do not reflect the views of the funding agency. The authors report there are no competing interests to declare. 

\baselineskip=18pt
\bibliographystyle{asa}
\bibliography{reference}

\newpage
\appendix 
\baselineskip=24pt
\setcounter{page}{1}
\setcounter{equation}{0}
\setcounter{section}{0}
\renewcommand{\thesection}{S.\arabic{section}}
\renewcommand{\thelemma}{S\arabic{lemma}}
\renewcommand{\theequation}{S\arabic{equation}}

\begin{center}
{\Large\bf Supplementary Materials} \\
\medskip
{\Large\bf ``Contextual Dynamic Pricing with Strategic Buyers"}  \\
\bigskip
\textbf{Pangpang Liu, Zhuoran Yang, Zhaoran Wang, Will Wei Sun}
\vspace{0.2in} 
\end{center}
\bigskip

\noindent
In this supplement, we provide additional information related to our paper, and include detailed proofs of the theorems and lemmas. We provide a discussion on the regret lower bound of any pricing policy under our problem setting in Section \ref{lower}.   Section \ref{greedy1} extends our policy to the strategic $\epsilon$-greedy pricing policy. Section \ref{heter} considers the heterogeneity of the marginal cost. Section \ref{fm} discusses the detection of feature manipulation in real life. Section \ref{dub} discusses the $O(\sqrt{T})$ upper regret bound. Section \ref{sec7} presents some future directions.  Section \ref{sec6.2} gives sensitivity tests of our proposed pricing policies.   Section \ref{stra} provides additional related literature. Section \ref{s2} gives the proof under the non-strategic pricing policy, $i.e.$, Theorem \ref{theory3}. Section \ref{s1} provides the proofs under the strategic pricing policy with the known marginal cost $A$, including Lemma \ref{lemma11} and Theorem \ref{theory2}. Section \ref{s3} offers the proofs under the strategic pricing policy with the unknown marginal cost $A$, including Lemma \ref{lemmaunA} and Theorem \ref{theory5}. Section \ref{s4} includes all supporting technical lemmas.

\section{Discussion on Minimax Lower Bound}\label{lower}
To establish the lower bound $\Omega(\sqrt{T})$ for any pricing policy, we borrow the "uninformative price" idea 
 from \cite{Broder2012} and construct a special instance following \cite{Fan2022}. The problem setting in our work is similar to that of \cite{Fan2022}, differing in that our paper considers a known $F$, while \cite{Fan2022} addresses an unknown $F$. An uninformative price is a price that all demand curves (probability of successful sales) as offered price indexed by unknown parameters intersect. Namely, the demands at this uninformative price are the same for all unknown parameters. In addition, such price is also the optimal price with some parameters. In this case, the price is uninformative because it does not reveal any information on the true parameter. Intuitively, if one tries to learn model parameters, the only way is to offer prices that are sufficiently far from the uninformative price (optimal price) which leads to a larger regret. Following \cite{Fan2022}, we consider a class of distributions $\mathcal{F}$ which satisfies Assumption \ref{a1},
     $$\mathcal{F}:=\{F_{\sigma}: \sigma>0, F_{\sigma}=F(x/\sigma)\}.$$ Here, $F$ is the c.d.f. of a known distribution with mean zero. We set $\boldsymbol{\beta}_0=0$ and fix a number $\xi$ with $F'(\xi)\neq 0$. Then we choose a collection of $\{(\sigma, \alpha_0)\}$ which satisfies $1/\sigma=\xi +\alpha_0/\sigma$. When the price $p=1$, all demand curves intersect at a point $1-F_{\sigma}(\xi)$. Then the price $p=1$ is an uninformative price. For the case $(\sigma,\alpha_0)=(1/(\xi-\phi(\xi)),-\phi(\xi)/(\xi-\phi(\xi)))$, $p=1$ is also the optimal price \citep{Fan2022}. From \cite{Broder2012}, we know that for a policy to reduce its uncertainty
about the unknown demand parameter, it must necessarily
set prices away from the uninformative price $p=1$ and thus incur large regret, and  any policy that does not reduce its uncertainty about the demand parameter must also incur a cost in regret. Therefore, following the argument in \cite{Fan2022}, the $\Omega(\sqrt{T})$ lower bound can be established.
\section{Strategic $\epsilon$-greedy Pricing Policy}\label{greedy1}


While the two-phase exploration-exploitation mechanism in our algorithm is a common practice in dynamic pricing, our method can be extended to a variant of the $\epsilon$-greedy algorithm that involves simultaneous exploration and exploitation. Here, we introduce a new strategic $\epsilon$-greedy pricing policy that integrates both exploration and exploitation phases.

The workflow of the strategic $\epsilon$-greedy pricing policy is as follows. At each time $t$, the seller decides to implement the uniform policy with probability $\epsilon$ and the optimal pricing policy with probability $1-\epsilon$. The $\epsilon$-greedy pricing policy with the known marginal cost is shown as Algorithm \ref{alg01}. When the marginal cost is unknown, we can use the same strategy as Algorithm 2 in our paper to estimate it.
\begin{algorithm}[!ht]
\caption{Strategic $\epsilon$-greedy Policy with Known Marginal Cost}
\label{alg01}
\begin{algorithmic}[1]
\STATE  \textbf{Input}: $B, \ell_0, C_a$
\FOR{each episode $k=1,2,...$}
\STATE Set the length of $k$-th episode as $\ell_k= 2^{k-1}\ell_0$, $\epsilon_k= \sqrt{C_a/\ell_k}$, $I_k'=\varnothing$, $\hat{\boldsymbol{\theta}}_k=0$.
\FOR{$t\in I_k:=\{ \ell_k,...,2\ell_k\}$}
\STATE The seller chooses the uniform policy with probability $\epsilon_k$ and the optimal policy with probability $1-\epsilon_k$ and informs the buyer the chosen policy.
\STATE The buyer reveals $\tilde{\boldsymbol{x}}_t$. Denote $\boldsymbol{x}_t=(\tilde{\boldsymbol{x}}_t^\top, 1)^\top$.
\STATE The seller offers a price by. 
\begin{equation*}
  p_t =
    \begin{cases}
      \tilde{p} \sim \text{Unif}(0, B) & \text{if the uniform pricing policy is chosen},\\
      g(\hat{\boldsymbol{\theta}}_k^\top \boldsymbol{x}_t+\hat{\boldsymbol{\beta}}_k^\top A^{-1}\hat{\boldsymbol{\beta}}_k g'(\hat{\boldsymbol{\theta}}_k^\top \boldsymbol{x}_t)) & \text{if the optimal pricing policy is chosen}.
    \end{cases}       
\end{equation*}
\STATE If the uniform pricing policy is chosen, $I'_k=I'_k\cup \{t\}$.
\STATE Observe a binary response $y_t$.
\ENDFOR
\STATE Calculate the estimate of $\boldsymbol{\theta}_0$ by 
\begin{align*} 
\hat{\boldsymbol{\theta}}_{k}=\mathop{\arg\min}_{\boldsymbol{\theta}\in\Theta} L_k(\boldsymbol{\theta}),
\end{align*}
where
$ L_k(\boldsymbol{\theta})=\frac{1}{|I'_k|}\sum_{ t\in I'_k}\big\{\mathbb{I}(y_t=1)\log [1-F(p_t-\boldsymbol{\theta}^\top \boldsymbol{x}_t)]+\mathbb{I}(y_t=0)\log F(p_t-\boldsymbol{\theta}^\top \boldsymbol{x}_t)\big\}.$
\ENDFOR
\end{algorithmic}
\end{algorithm}
We next conduct experiments to verify the effectiveness of the new strategic $\epsilon$-greed pricing policy. We set the repeat buyer rate $\tau=0.1\%$ and the  marginal cost as $A=\begin{pmatrix}
1/8 & 1/16 \\
1/16 & 1/8
\end{pmatrix}$. Other settings are the same with those in Section \ref{6.1.1}. The result is shown in Figure \ref{fig03}, indicating that both strategic $\epsilon$-greedy policies outperform the non-strategic policy, whether the marginal cost $A$ is known or unknown.
\begin{figure}[h!]
    \centering
    \begin{tabular}{cc}  
        \includegraphics[scale = 0.70]{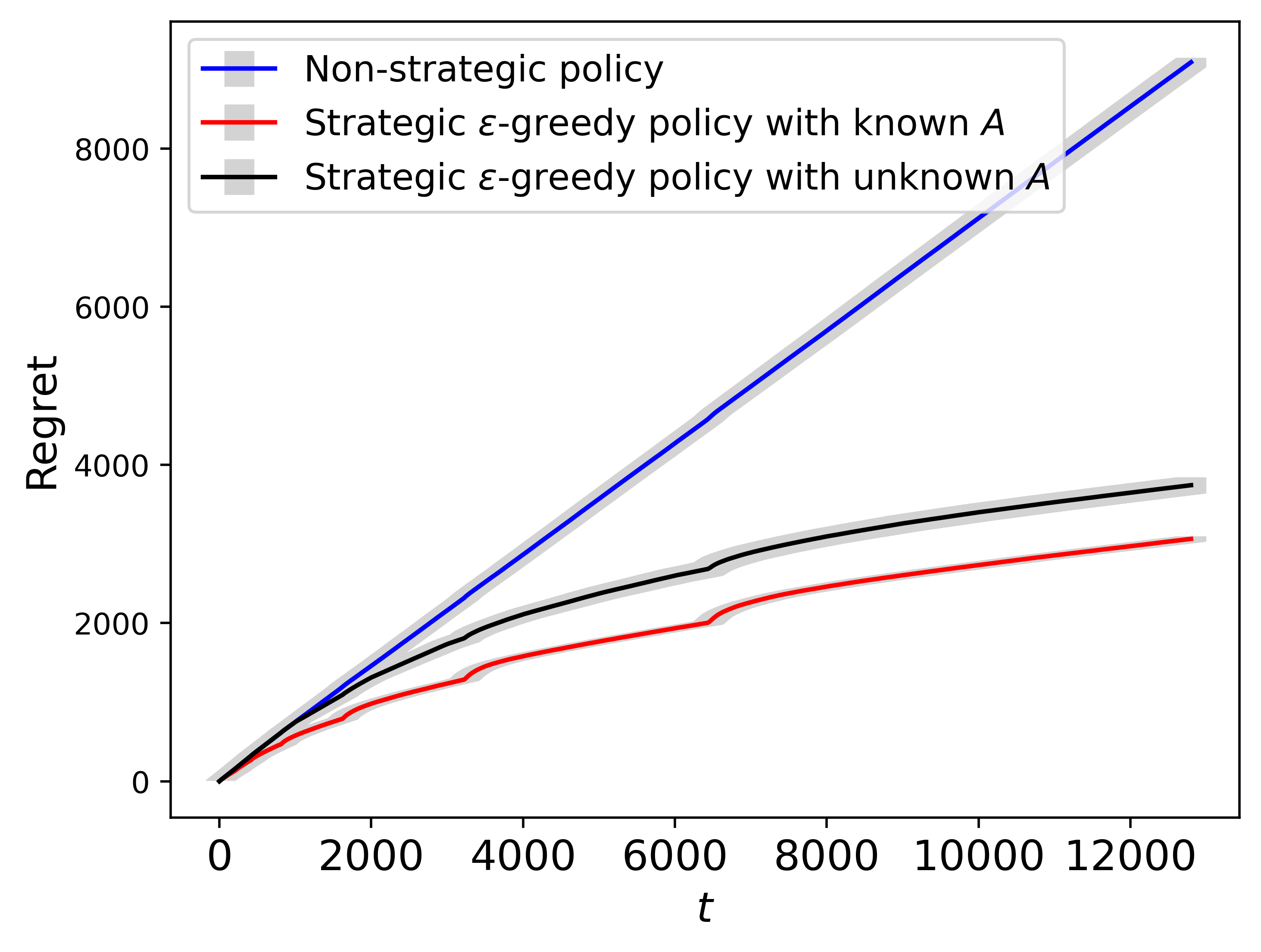}
    \end{tabular}
     \caption{Regret plots for the three policies.}
         \label{fig03}
\end{figure}
\section{Heterogeneity of Marginal Cost}\label{heter}
In practical scenarios, manipulation costs may differ among individual buyers. To address this variability, we broaden our pricing policy to accommodate the existence of heterogeneity in marginal costs. This extension encompasses two cases: (1) where there are distinct groups, each sharing the same cost matrix; (2) where different buyers possess varied costs, but share a common random cost structure.
\subsection{Different Costs in Different Groups}
There are $K$ groups of buyers, and buyers from group $k$ share the cost matrix $A_k$. When the cost matrix $A_k$ is known by the seller, the strategic pricing policy aligns with Algorithm 1. If the cost matrix $A_k$ is unknown, an estimation approach similar to Algorithm 2 can be employed, provided that the true group status is known to the seller. 
Specifically, during the exploration phase, buyers disclose their true features and group status. In the exploitation phase, buyers reveal manipulated features. We can estimate the unknown $A_k$ by matching the true and manipulated features using the method outlined in Section 4.2 of our paper.
         
         We conduct some experiments to validate the effectiveness of our policy under varying marginal costs for different buyer groups. Consider two buyer groups and set the cost matrix for group 1 as $A_1=\begin{pmatrix}
1/4 & 1/8 \\
1/8 & 1/4
\end{pmatrix}$ and for group 2 as $A_2=1/2\begin{pmatrix}
1/4 & 1/8 \\
1/8 & 1/4
\end{pmatrix}$.  Other settings are the same with those in Section \ref{6.1.1}.
Figure \ref{fig001} illustrates the results of our pricing policies alongside the non-strategic policy, demonstrating superior performance when faced with different buyer groups with varying marginal costs.
 \begin{figure}[h!]
    \centering
    \begin{tabular}{cc}  
        \includegraphics[scale = 0.70]{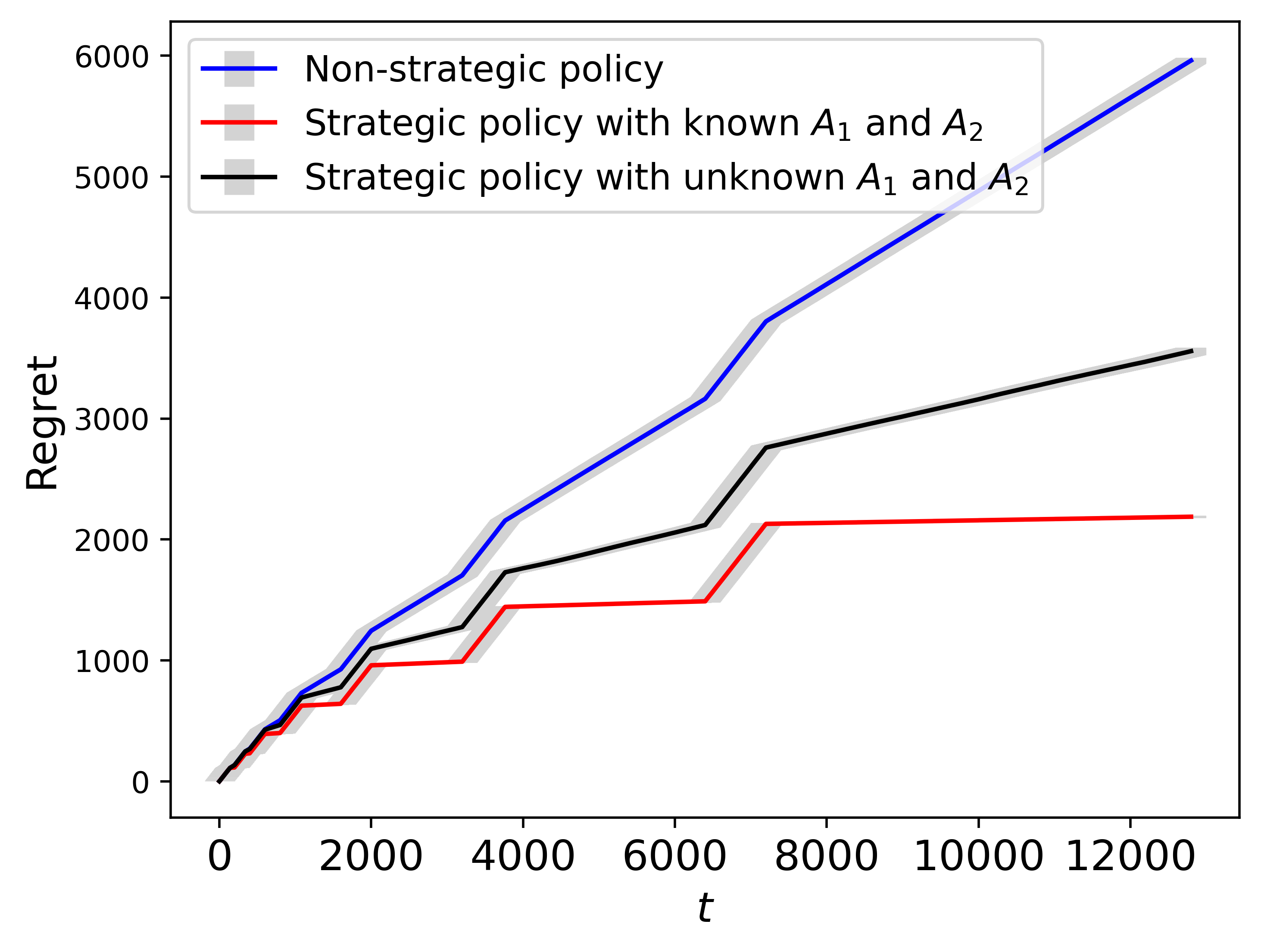}
    \end{tabular}
     \caption{Regret plots for the three policies with two buyer groups.}
         \label{fig001}
\end{figure}
         
\subsection{Random Cost}
The cost matrix adopts a random structure represented as $A_t=A_0+\epsilon_t$, where $A_0$ is unknown to the seller and $\mathbb{E}\epsilon_t=\boldsymbol{0}$. To address the unknown $A_0$, we leverage the approach outlined in Algorithm 2.  In the exploration phase, the seller gathers true features, while in the exploitation phase, buyers disclose manipulated features. Through the alignment of true and manipulated features, we estimate the unknown $A_0$ using the methodology detailed in Section 4.2 of our paper.
         
To assess the efficacy of our policy in the presence of random marginal costs, we conduct experiments by setting $A_0=\begin{pmatrix}
1/4 & 1/8 \\
1/8 & 1/4
\end{pmatrix}$ and $\epsilon_t\sim \begin{pmatrix}
N(0, 0.01) & 0 \\
0 & N(0, 0.01)
\end{pmatrix}$. We set $\tau=0.1\%$. Other settings are the same with those in Section \ref{6.1.1}. Figure \ref{fig002} depicts the results of our pricing policies and the non-strategic policy, illustrating the superior performance of our policies in scenarios where the marginal cost is random.
    \begin{figure}[h!]
    \centering
    \begin{tabular}{cc}  
        \includegraphics[scale = 0.70]{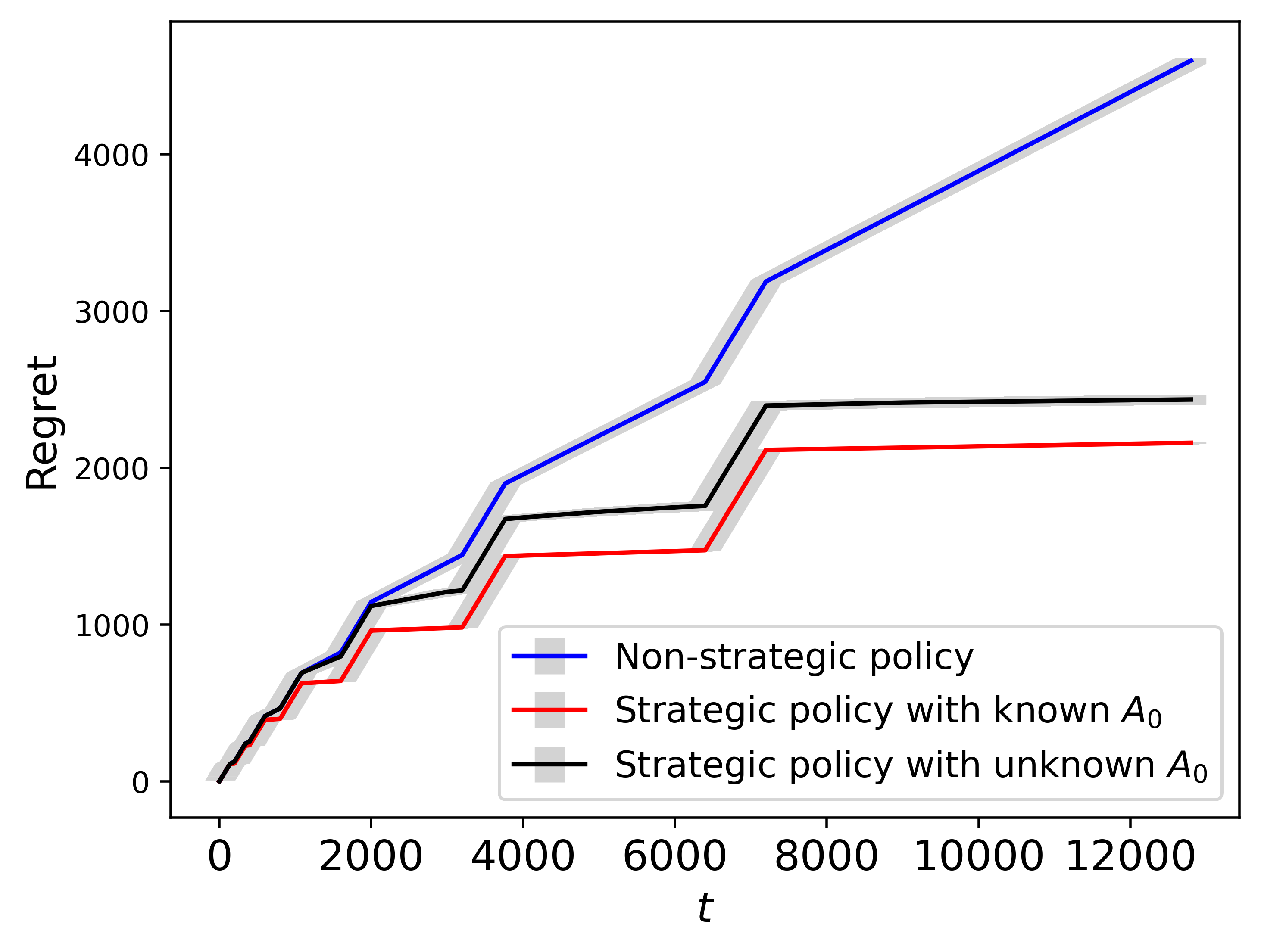}
    \end{tabular}
     \caption{Regret plots for the three policies with $A_t=A_0+\epsilon_t$.}
         \label{fig002}
\end{figure}

\section{Detection of Feature Manipulation in Real Life}\label{fm}
Feature manipulation is prevalent in real life. We first present several examples to illustrate its existence in our auto loan real application.
     \begin{itemize}
         \item As indicated in the article "The Basics of Loan Fraud and How To Prevent It"\footnote{\href{https://fingerprint.com/blog/what-is-loan-fraud/}{https://fingerprint.com/blog/what-is-loan-fraud/}}, one of the most common forms of loan fraud is application fraud, which involves falsely applying for a loan by providing inaccurate or incomplete information on an application form. This includes providing false employment history or exaggerating the income level. The National Mortgage Application Fraud Risk Index increased by 15\% between the 2021 first quarter and the first quarter of 2022.
         \item The internet makes it easy to create seemingly legitimate documents that support auto loan fraud\footnote{\href{https://defisolutions.com/answers-old/growing-threat-fraud-auto-loan-origination/}{https://defisolutions.com/answers-old/growing-threat-fraud-auto-loan-origination/}}. Various online services assist fraudsters in fabricating income statements, often exaggerating figures in anticipation of extravagant purchases. Some websites help fraudster create a fake paystub, ``recommending the type of statement, income, monthly, or weekly pay ranges based upon the supposed occupation and location. The goal is to make the resultant paystub appear as authentic as possible". The number of submitting fake pay stubs or overstating incomes is increasing\footnote{\href{https://www.forbes.com/sites/edgarsten/2023/07/21/fake-paystubs-overstating-income-bank-pullouts-plague-auto-financing/?sh=33d9ba716977}{https://www.forbes.com/sites/edgarsten/2023/07/21/fake-paystubs-overstating-income-bank-pullouts-plague-auto-financing/?sh=33d9ba716977}}.
         
     \end{itemize}

      Secondly, given our reliance on a publicly available dataset for analysis, direct verification of feature manipulation in our real data is unfeasible. Here, we provide some existing methods to detect feature manipulation. Many methods have been developed to detect loan fraud \citep{Jerzy2021,Jiang2021,Khaled2021,Ali2022,Gu2022,Liao2022}, including supervised, unsupervised, semi-supervised methods and graph-based methods \citep{Waleed}. The traditional method is to develop some supervised models using some datasets containing customers' information and labels (fraud or not) for loan fraud detecting. In the loan pricing process, we start by gathering requisite data and employing existing methods to detect customers whose features may have been manipulated. Subsequently, we implement our proposed pricing policy. 
     
    Next, we review one method for detecting feature manipulation in detail. \cite{Liao2022} provided a supervised learning method to detect the false information in loan application. This study focused on four common types of fake information, including fake occupation information, fake ability information, fake marriage information, fake contact information. More specifically, the information contains working units, monthly income, driving ability, their spouse's basic information and working information, contact information, etc. To verify whether the information is fake, three methods are applied: phone review, home visits, and third-party data verification. For instance, if the applicant claims that he/she serves as a salesman in an electrical appliance store, the company may ask the applicant what brands of refrigerators they have and then make a judgment based on the applicant's reply, reaction, and even tone. Each observation containing fake information is labelled with the type of fake information.  After obtaining the data with labels, a logistic model is applied to detect the fake information. 
    
\section{Discussion on $O(\sqrt{T})$ Upper Regret Bound} \label{dub}
In this section, we give an outline of the proof for the $\sqrt{T}$ upper regret bound of our proposed policy. The detailed proof is presented in Section \ref{pof2}. We let $r_t(p)=p[1-F(p-\boldsymbol{\theta}_0^\top \boldsymbol{x}_t^0)]$ be the expected revenue. We define the filtration generated by all transaction records up to time $t$ as $ 
\mathcal{H}_{t}=\sigma(\boldsymbol{x}_1^0, \boldsymbol{x}_2^0, \cdots, \boldsymbol{x}_t^0, z_1, z_2, \cdots, z_t)$. We also define $\Tilde{\mathcal{H}}_{t}=\mathcal{H}_{t}\cup \{\boldsymbol{x}_{t+1}^0\}$ as the filtration obtained after augmenting by the new feature $\boldsymbol{x}_{t+1}^0$. We define the regret at time $t$ as $R_t=p_t^*\mathbb{I}(v_t\geq p_t^*)-p_t\mathbb{I}(v_t\geq p_t)$. Then the conditional expectation of the regret at time $t$ given previous information and $\boldsymbol{x}_t^0$ is
\begin{align}\label{0R}
\mathbb{E}(R_t|\Tilde{\mathcal{H}}_{t-1})&=\mathbb{E}[p_t^*\mathbb{I}(v_t\geq p_t^*)-p_t\mathbb{I}(v_t\geq p_t)|\Tilde{\mathcal{H}}_{t-1}]\nonumber \\
&=p_t^*[1-F(p_t^*-\boldsymbol{\theta}_0^\top \boldsymbol{x}_t^0)]-p_t[1-F(p_t-\boldsymbol{\theta}_0^\top \boldsymbol{x}_t^0)]\nonumber \\
&=r_t(p_t^*)-r_t(p_t).
\end{align}
Note that $p_t^*\in \mathop{\arg\max}_pr_t(p)$ and hence we have $r'_t(p_t^*)=0$, which is the key point why the accuracy is of the order $O(1/(\alpha T))$ \citep{Javanmard2019,Xu2021}. This special structure does not hold in traditional bandit algorithms. The special structure $r'_t(p_t^*)=0$ of the dynamic pricing problem leads to a better regret order.

Using the Taylor expansion, we have
\begin{equation}\label{0rt}
 r_t(p_t)=r_t(p_t^*)+\underbrace{r'_t(p_t^*)}_{0}(p_t-p_t^*)+\frac{1}{2}r''_t(\xi_t)(p_t-p_t^*)^2=r_t(p_t^*)+\frac{1}{2}r''_t(\xi_t)(p_t-p_t^*)^2,   
\end{equation}
where $\xi_t$ is some value between $p_t$ and $p_t^*$. The key point is that the term $(p_t-p_t^*)$ in \eqref{0rt} is removed because $r'_t(p_t^*)=0$. By Assumptions \ref{a3} and \ref{a1}, we have 
\begin{equation}\label{0r2}
    \begin{aligned}
     |r''_t(\xi_t)|&=|2f(\xi_t-\boldsymbol{\theta}_0^\top \boldsymbol{x}_t^0)+\xi_tf'(\xi_t-\boldsymbol{\theta}_0^\top \boldsymbol{x}_t^0)|
     \leq 2M_f+BM_{f'}.   
    \end{aligned}
\end{equation}
Now we can obtain an upper bound on the conditional expectation of the regret at time $t$ given $\Tilde{\mathcal{H}}_{t-1}$.
By (\ref{0R}), (\ref{0rt}) and (\ref{0r2}), we have
\begin{equation*}
 \mathbb{E}(R_t|\Tilde{\mathcal{H}}_{t-1})\leq \left(M_f+\frac{B}{2}M_{f'}\right)\mathbb{E}(p_t^*-p_t)^2.   
\end{equation*}
Then by \eqref{pbound}, \eqref{ej1} and \eqref{J2} in our supplementary materials, we can prove that $\mathbb{E}(p_t^*-p_t)^2$ can be bounded by $O(\mathbb{E}\|\boldsymbol{\theta}_0-\hat{\boldsymbol{\theta}}\|_2^2)$ not $O(\mathbb{E}\|\boldsymbol{\theta}_0-\hat{\boldsymbol{\theta}}\|_2)$ in traditional bandit cases.

\section{Future Directions}\label{sec7}

In this paper, we propose new strategic dynamic pricing policies for the contextual pricing problem with strategic buyers. We establish a sublinear $O(\sqrt{T})$ regret bound for the proposed policy, improving the $\Omega(T)$ regret lower bound of existing non-strategic pricing policies.

There are several promising avenues for future exploration. Firstly, we can examine the strategic dynamic pricing problem under an unknown noise distribution $F(\cdot)$. One possible solution is to incorporate the method of estimating $F(\cdot)$ proposed in \cite{Fan2022} into our policy. 
Secondly, we can study the pricing problem of strategic buyers by incorporating fairness-oriented policy \citep{chen2021fairnessaware,Fang2022}. Thirdly, we can explore the strategic pricing problem with censored demand \citep{qi2022offline}, unobserved confounding \citep{qi2023proximal}, high-dimensional features \citep{Hao2020,Shi2021,zhao2023highdimensional}, or adversarial setting \citep{Cohen2020, Xu2021, Xu2022}. Finally, when the feature distribution is non-stationary, we can explore the dynamic pricing problem with more general reinforcement learning settings \citep{zhu2015reinforcement,Gen2022,Shi2022,Hambly2023}.

\section{Sensitivity Tests}\label{sec6.2}
In this section, we investigate the sensitivity of our policies to the hyperparameters $B$, $\ell_0$, and $C_a$. Here, $B$ represents an upper bound on the price, $\ell_0$ denotes the minimum episode length, and $C_a$ is a constant used in determining the length of the exploration phase. To assess the sensitivity of our policies, we conduct simulations with different values of these hyperparameters while keeping $A=A_0$ and $\tau=0.05\%$ fixed. \par
First, we examine the sensitivity of $B$. For these simulations, we set $\ell_0=200$ and $C_a=100$. Figure \ref{fig4} illustrates the regrets of the three policies under three scenarios: $B=6$, $B=7$, and $B=8$. The figure demonstrates that the comparison results remain robust across different choices of $B$.\par
\begin{figure}[t!]
    \centering
    \begin{tabular}{ccc}  
        \includegraphics[scale = 0.3420]{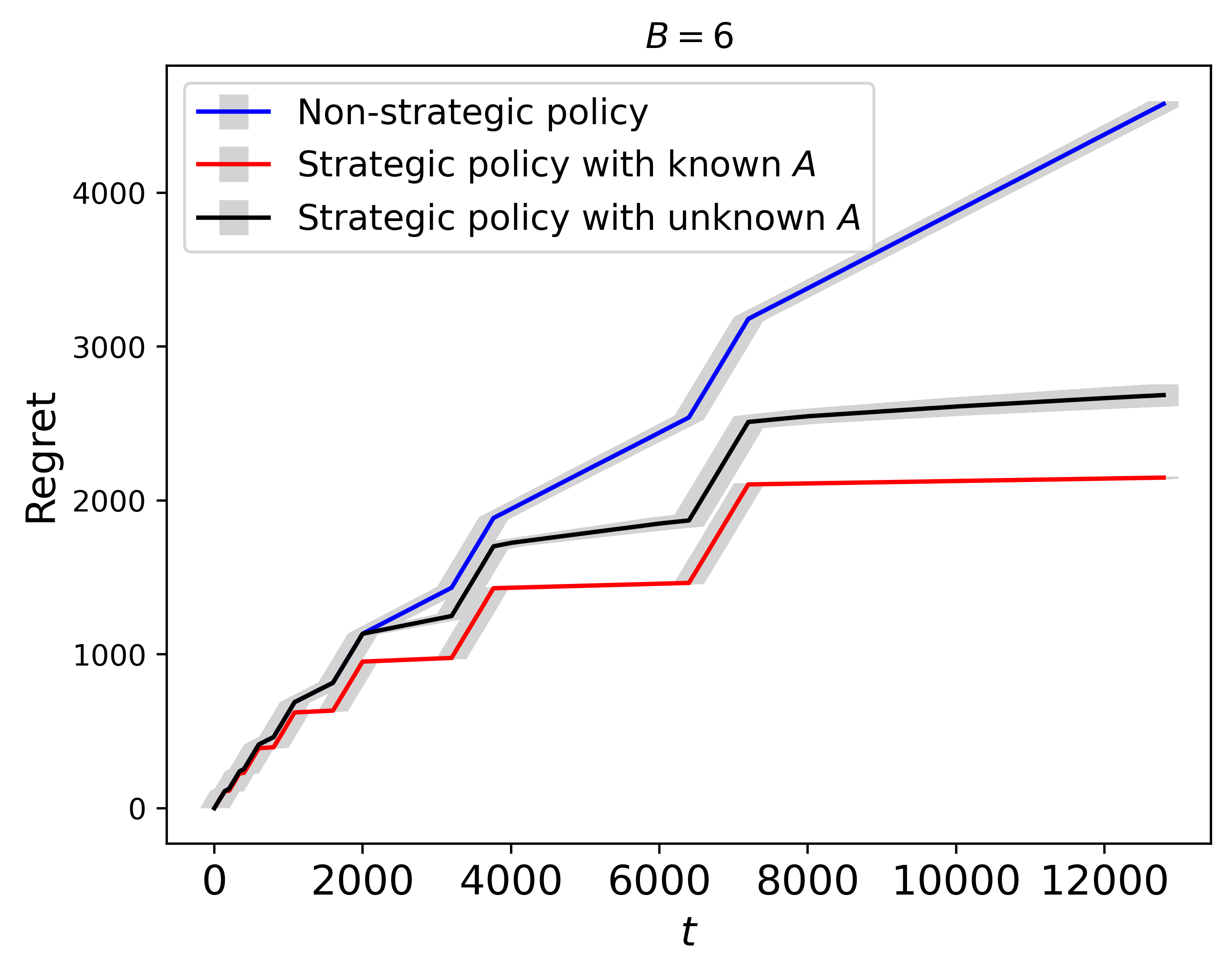}&
        \includegraphics[scale = 0.3420]{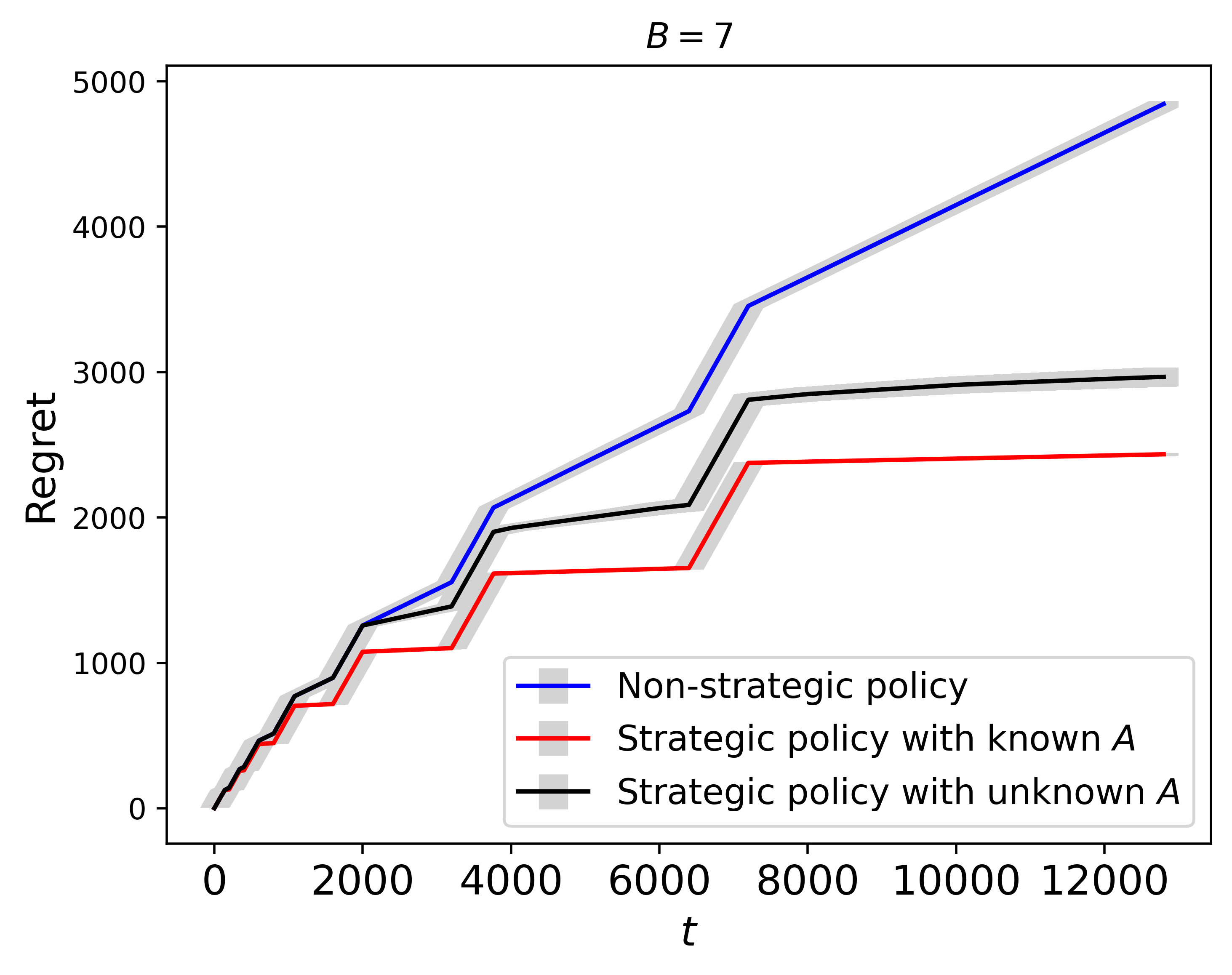}&
        \includegraphics[scale = 0.3420]{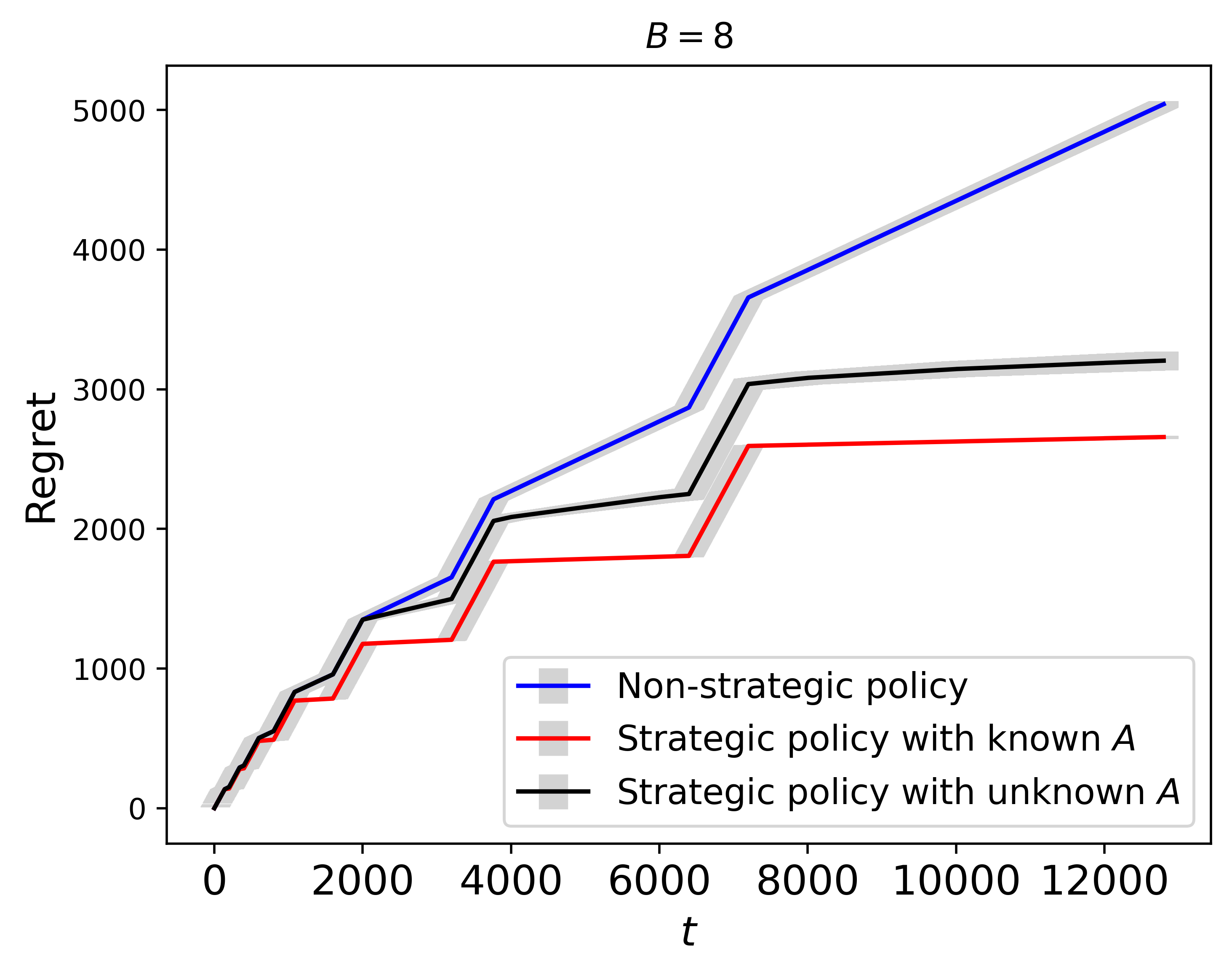}
    \end{tabular}
     \caption{Regret plots for the three policies. The three subplots show the regrets of three different scenarios, $B\in \{6,7,8\}$. The remaining caption is the same as Figure \ref{fig1}.}
         \label{fig4}
\end{figure}
Next, we evaluate the sensitivity of $\ell_0$. In these simulations, we set $B=6$ and $C_a=100$. Figure \ref{fig5} displays the regrets of the three policies for three different scenarios: $\ell_0=100$, $\ell_0=150$, and $\ell_0=200$. The figure shows that the comparison results remain consistent across different choices of $\ell_0$. \par
\begin{figure}[t!]
    \centering
    \begin{tabular}{ccc}  
        \includegraphics[scale = 0.3420]{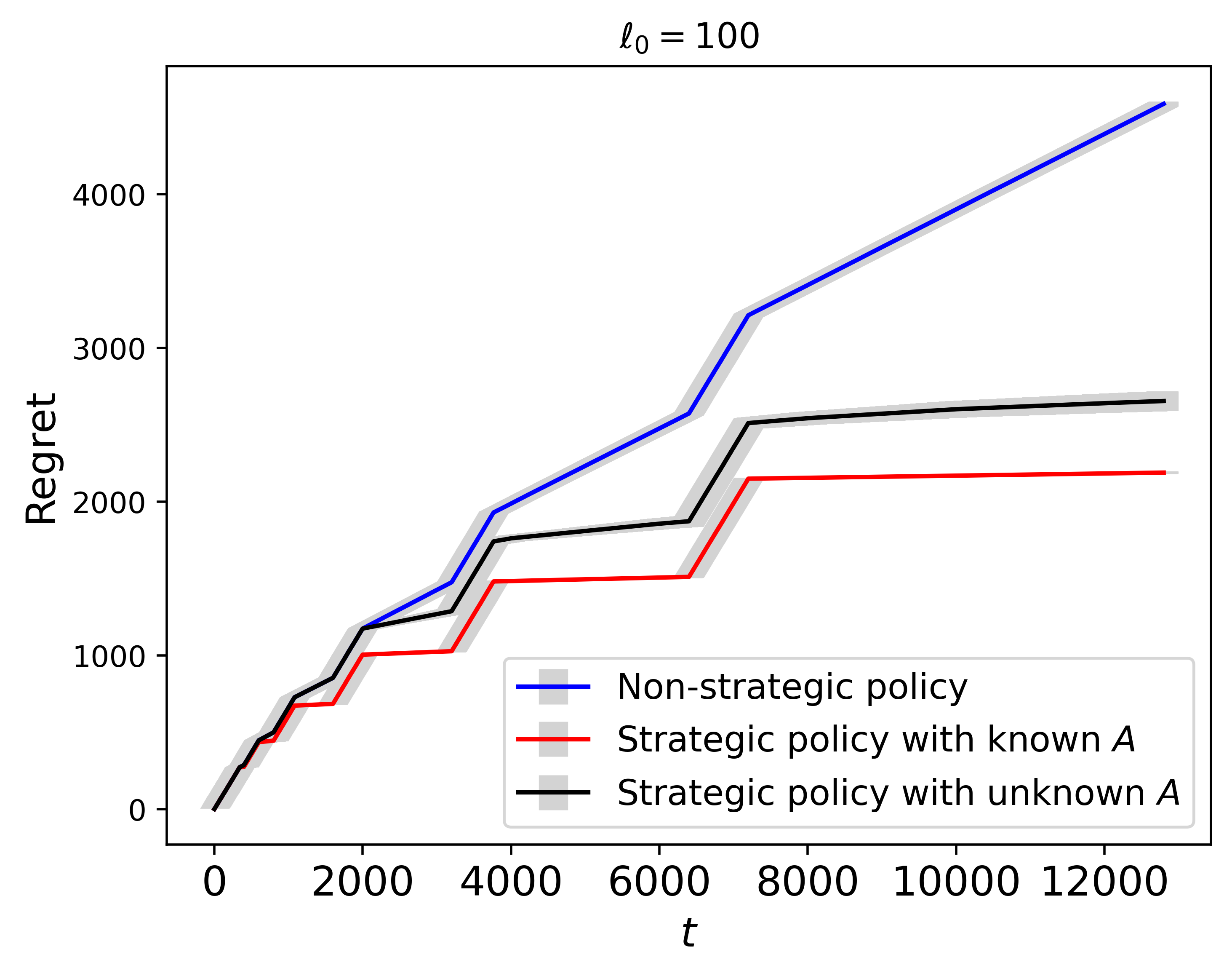}&
        \includegraphics[scale = 0.3420]{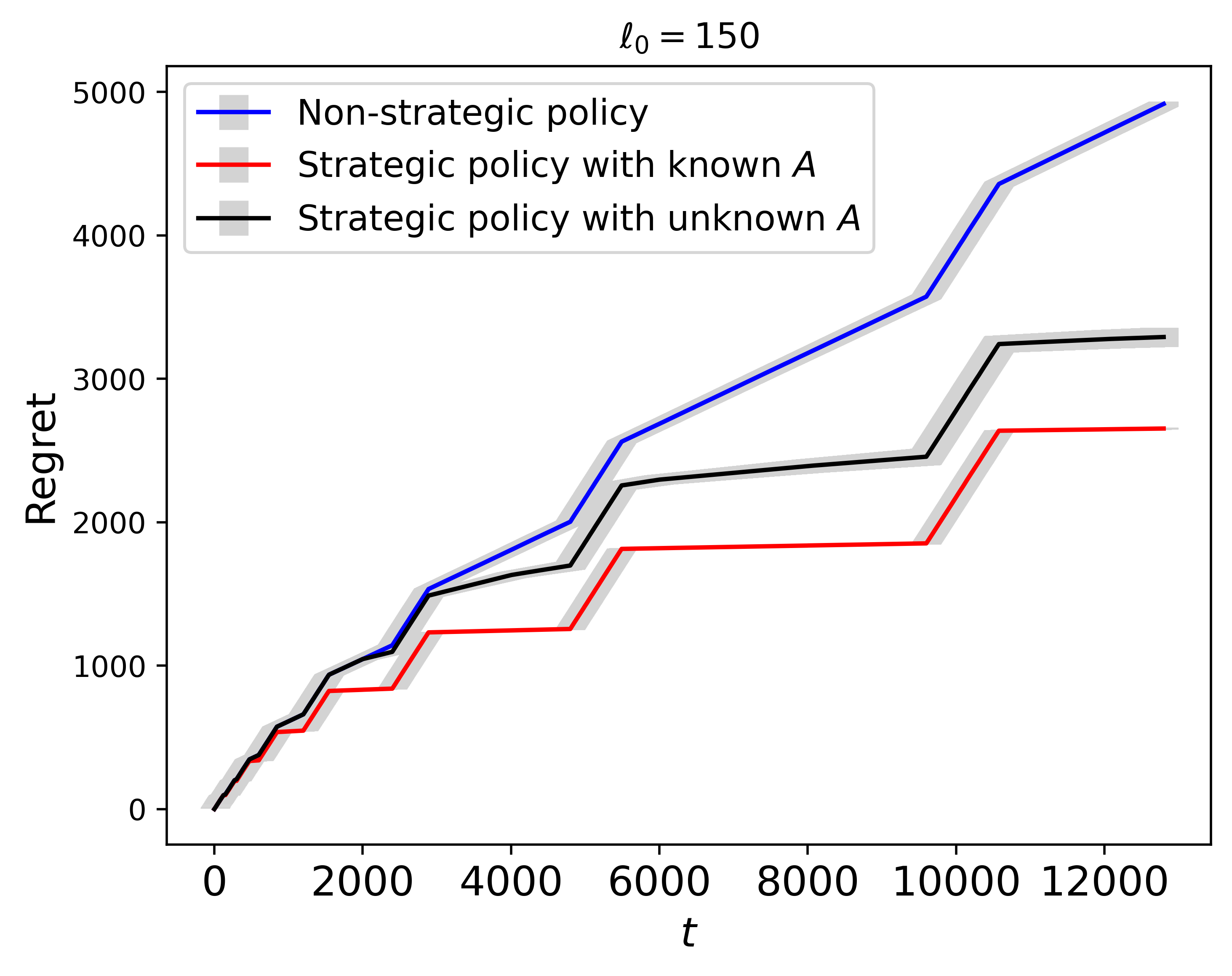}&
        \includegraphics[scale = 0.3420]{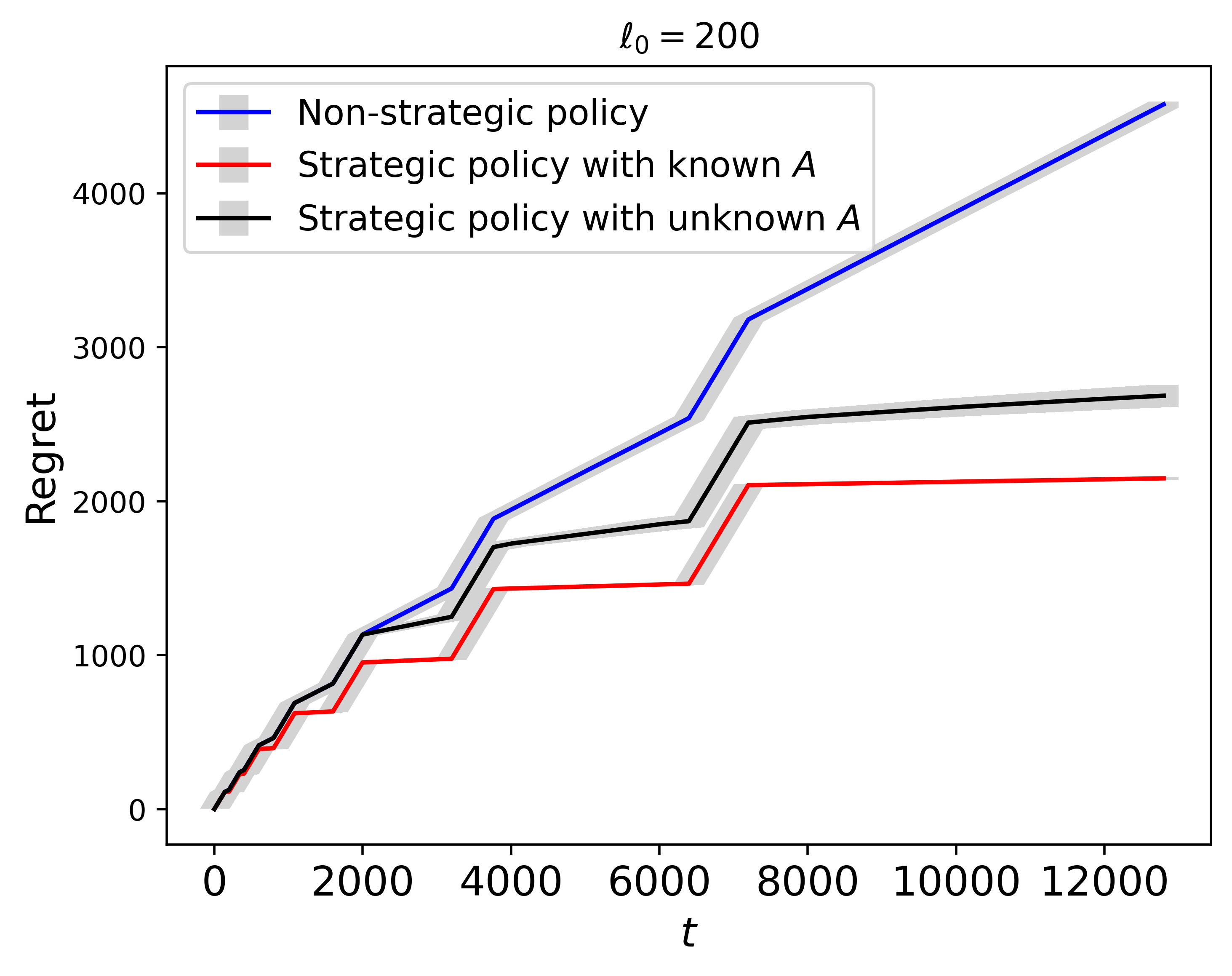}
    \end{tabular}
     \caption{Regret plots for the three policies. The three subplots show the regrets of three different scenarios, $\ell_0\in\{100,150,200\}$. }
         \label{fig5}
\end{figure}
Finally, we assess the sensitivity of $C_a$. In these simulations, we set $B=6$ and $\ell_0=100$. Figure \ref{fig6} presents the regrets of the three policies for three different scenarios: $C_a=50$, $C_a=100$, and $C_a=150$. The figure demonstrates that the comparison results are robust across different choices of $C_a$.  \par 
\begin{figure}[t!]
    \centering
    \begin{tabular}{ccc}  
        \includegraphics[scale = 0.3420]{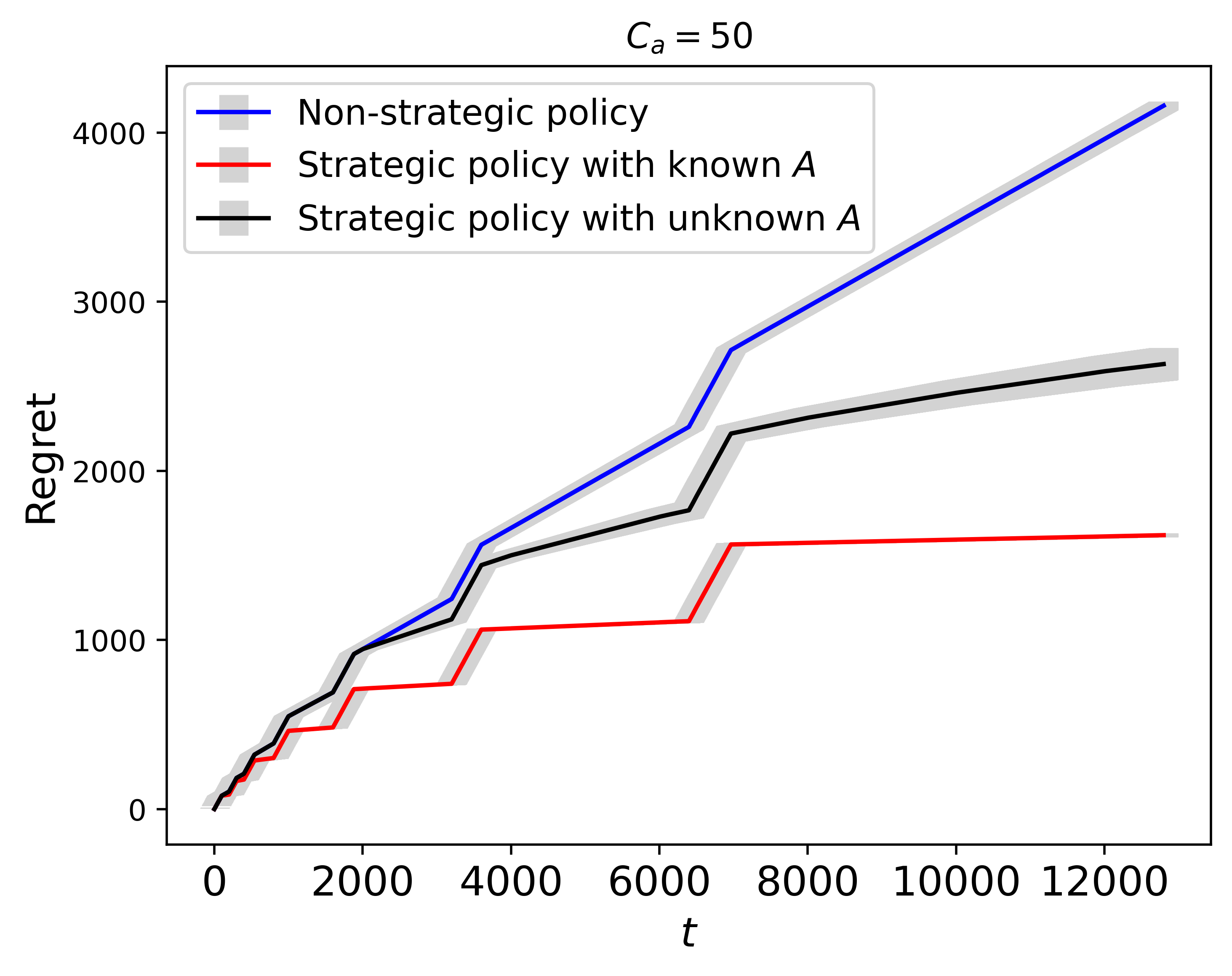}&
        \includegraphics[scale = 0.3420]{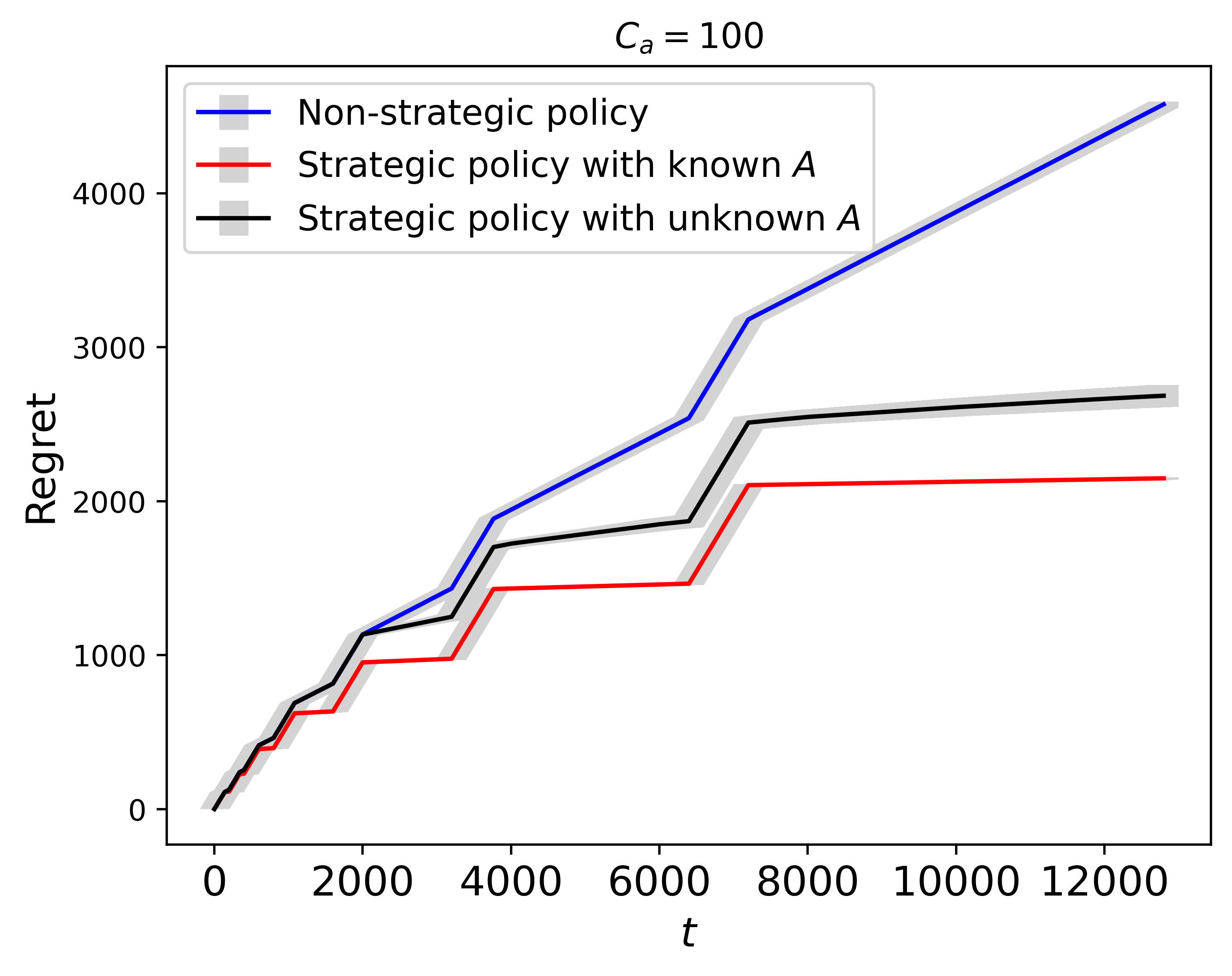}&
        \includegraphics[scale = 0.3420]{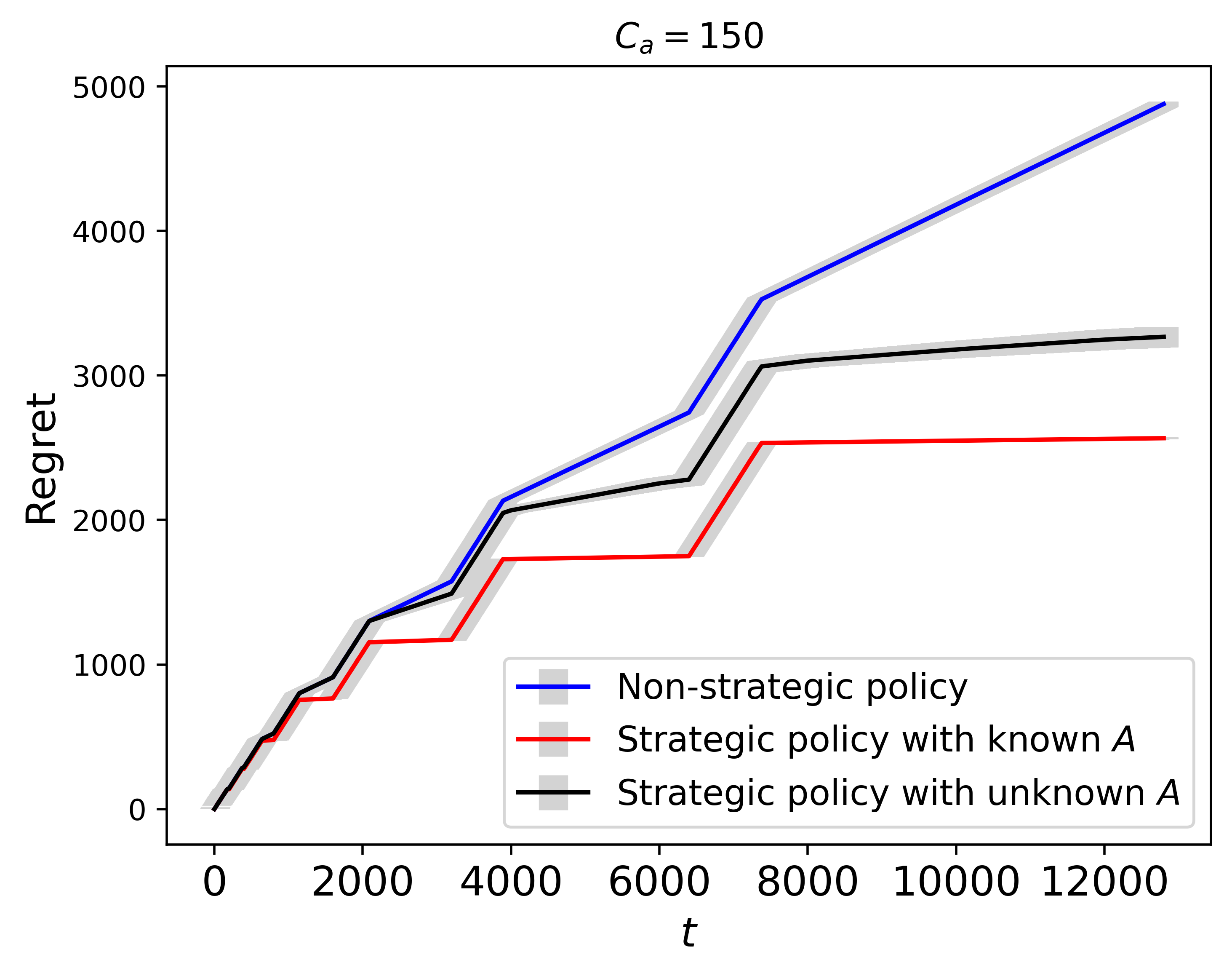}
    \end{tabular}
     \caption{Regret plots for the three policies. The three subplots show the regrets of three different scenarios, $C_a\in \{50,100,150\}$.}
         \label{fig6}
\end{figure}
Overall, our sensitivity analysis indicates that the performance of our policies remains consistent and robust under variations in the hyperparameters $B$, $\ell_0$, and $C_a$. 

\section{Additional Related Literature} \label{stra}
\textbf{Timing and Untruthful Bidding in Pricing and Auction Design.}
Existing strategic work mainly focused on timing and untruthful bidding in pricing and auction design. Timing refers to the time of purchases. In this setting, the buyers are forward looking and time-sensitive, and the strategy for these buyers is choosing the time of purchasing. The private valuations of buyers decay over time and buyers incur monitoring costs. The buyers strategize about the time of purchases to maximize the utility \citep{Chen2018}. In addition, untruthful bidding appears in repeated auctions. In auctions, the strategy used by the buyers is lie, which happens if the buyer accepts the price while the price offered is above his valuation, or when he rejects the price while his valuation is above the offered price \citep{Amin2014, Mohri2015,  Chen2022}. In the contextual auction literature, both the seller and buyers are able to observe the true features \citep{Golrezaei2023}. While the strategic behaviors of timing and untruthful bidding have received considerable attention, the manipulation of features in pricing setting, particularly in the online dynamic pricing setting, has remained relatively unexplored. By including this strategic behavior, our work enriches the understanding of strategic behaviors in dynamic pricing problems, providing a comprehensive framework for considering buyer manipulation in pricing decisions.

\section{Proof under Non-strategic Pricing Policy}\label{s2}
\subsection{Proof of Theorem \ref{theory3}}
The regret (\ref{bench}) is defined as the maximum gap between a policy and the oracle policy over different $\boldsymbol{\theta}_0\in \Theta$ and $\mathbb{P}_X\in Q(\mathcal{X})$. In order to obtain a lower bound on the regret, it suffices to consider a specific distribution in $Q(\mathcal{X})$. We consider the distribution $F$ as the uniform distribution on (-1/2, 1/2). The marginal cost matrix is $A=I$, and $\|\boldsymbol{\beta}_0\|_1= 1, B=7/16,  \|\tilde{\boldsymbol{x}}_t^0\|_{2}\leq 1/4$. \par 
In order to bound the total regret, we try to bound the regret at each time $t$. The expected revenue at time $t$ during the exploitation phase is
\begin{align*}
r_t(p)&=p[1-F(p-\boldsymbol{\theta}_0^\top \boldsymbol{x}_t^0)]=p \left (1-p+\boldsymbol{\theta}_0^\top \boldsymbol{x}_t^0-\frac{1}{2}\right)=\frac{p}{2}-p^2+p\boldsymbol{\theta}_0^\top \boldsymbol{x}_t^0.
\end{align*}
By the first-order derivative $\frac{dr_t(p)}{dp}=\frac{1}{2}-2p+\boldsymbol{\theta}_0^\top \boldsymbol{x}_t^0=0$, the oracle price is 
$$p^*_t=\frac{1}{4}+\frac{\boldsymbol{\theta}_0^\top \boldsymbol{x}_t^0}{2}:=g(\boldsymbol{\theta}_0^\top \boldsymbol{x}_t^0).$$
Therefore, the expected revenue at time $t$ by the oracle pricing policy is 
\begin{equation}\label{rstar}
r_t(p_t^*)=0.0625+0.25\boldsymbol{\theta}_0^\top \boldsymbol{x}_t^0+0.25(\boldsymbol{\theta}_0^\top \boldsymbol{x}_t^0)^2.    
\end{equation}
We first analyze the regret during the exploration phase. Since the non-strategic price $p_t$ is randomly chosen from the distribution Unif(0, 7/16), The expected revenue at time $t$ using non-strategic pricing policy is
\begin{equation}\label{rn}
\frac{16}{7}\int_{0}^{7/16}\bigg(\frac{p}{2}-p^2+p\boldsymbol{\theta}_0^\top \boldsymbol{x}_t^0\bigg)dp=0.0456+0.2188\boldsymbol{\theta}_0^\top \boldsymbol{x}_t^0.  
\end{equation}
By (\ref{rstar}) and (\ref{rn}), the expected regret at time $t$ during the exploration phase is
\begin{equation}\label{regn1}
\mathbb{E}(R_t)>0.016.    
\end{equation}
Now, we analyze the regret during the exploitation phase. By (\ref{xt}), the manipulated feature is
\begin{equation}\label{xt1}
\tilde{\boldsymbol{x}}_t=\tilde{\boldsymbol{x}}_t^0-A^{-1}\boldsymbol{\beta}_0g'(\boldsymbol{\theta}_0^\top \boldsymbol{x}_t^0)=\tilde{\boldsymbol{x}}_t^0-\frac{\boldsymbol{\beta}_0}{2}.   
\end{equation}
Assume that $t$ is in the $k$-th epoch, and $\hat{\theta}_k$ is the MLE of $\theta_0$. The non-strategic pricing policy is $p_t=\frac{1}{4}+\frac{\hat{\boldsymbol{\theta}}_k^\top \boldsymbol{x}_t}{2}$. The difference of the expected revenues between the oracle policy and the non-strategic pricing policy is
\begin{equation}\label{r}
\begin{aligned}
r_t(p^*_t)-r_t(p_t)&=\frac{p^*_t}{2}-p^{*2}_t+p^*_t\boldsymbol{\theta}_0^\top \boldsymbol{x}_t^0-\left(\frac{p_t}{2}-p_t^2+p_t\boldsymbol{\theta}_0^\top \boldsymbol{x}_t^0\right)\\
&=\left (\frac{1}{2}+\boldsymbol{\theta}_0^\top \boldsymbol{x}_t^0\right)(p^*_t-p_t)-(p_t^*-p_t)(p_t^*+p_t)\\
&=\left (\frac{1}{2}+\boldsymbol{\theta}_0^\top \boldsymbol{x}_t^0-(p_t^*+p_t)\right)(p_t^*-p_t)\\
&= \frac{(\boldsymbol{\theta}_0^\top \boldsymbol{x}_t^0-\hat{\boldsymbol{\theta}}_k^\top \boldsymbol{x}_t)^2}{4}\\
&=\frac{[\boldsymbol{\alpha}_0+\boldsymbol{\beta}_0^\top (\tilde{\boldsymbol{x}}_t+\boldsymbol{\beta}_0/2)- \hat{\boldsymbol{\theta}}_k^\top \boldsymbol{x}_t]^2}{4}\\
&=\frac{[\boldsymbol{\alpha}_0+\boldsymbol{\beta}_0^\top \tilde{\boldsymbol{x}}_t- \hat{\boldsymbol{\theta}}_k^\top \boldsymbol{x}_t+\boldsymbol{\beta}_0^{\top} \boldsymbol{\beta}_0/2 ]^2}{4}\\
&\geq \frac{1}{4}\left\{[(\boldsymbol{\theta}_0-\hat{\boldsymbol{\theta}}_k)^\top \boldsymbol{x}_t]^2+\frac{(\boldsymbol{\beta}_0^{\top} \boldsymbol{\beta}_0)^2}{4}-|(\boldsymbol{\theta}_0-\hat{\boldsymbol{\theta}}_k)^\top \boldsymbol{x}_t\boldsymbol{\beta}_0^\top\boldsymbol{\beta}_0|\right\}\\
&\geq \frac{1}{4}\left[\frac{(\boldsymbol{\beta}_0^{\top} \boldsymbol{\beta}_0)^2}{4}-|(\boldsymbol{\theta}_0-\hat{\boldsymbol{\theta}}_k)^\top \boldsymbol{x}_t\boldsymbol{\beta}_0^\top\boldsymbol{\beta}_0|\right]\\
&:=\frac{1}{4}(J_3-J_4).
\end{aligned}
\end{equation}
We need to find a lower bound of $J_3-J_4$. We fist analyze $J_3$. For $J_3$, we have
\begin{equation}\label{j1}
\mathbb{E}(J_3)=\frac{(\boldsymbol{\beta}_0^{\top} \boldsymbol{\beta}_0)^2}{4}= \frac{1}{4}.
\end{equation}
Now, we analyze $J_4$. By (\ref{xt1}), we have $\|\tilde{\boldsymbol{x}}_t\|\leq \|\tilde{\boldsymbol{x}}_t^0 \|_2+\|\boldsymbol{\beta}_0\|_2/2\leq 3/4$. Therefore, 
\begin{equation}
\begin{aligned}
\mathbb{E}(J_4) &= \mathbb{E}|(\boldsymbol{\theta}_0-\hat{\boldsymbol{\theta}}_k)^\top \boldsymbol{x}_t\boldsymbol{\beta}_0^\top\boldsymbol{\beta}_0| \\
&\leq \frac{3}{4}\mathbb{E}\|(\boldsymbol{\theta}_0-\hat{\boldsymbol{\theta}}_k)\|_2 \\
&\leq \frac{3}{4}\sqrt{\frac{2(d+1)C_{up}^2}{C^2_{down}\lambda_{min}(a_k+1)}}\\
&\leq \frac{3}{4}\sqrt{\frac{2(d+1)C_{up}^2}{C^2_{down}\lambda_{min}\sqrt{C_a \ell_k}}}
\end{aligned}    
\end{equation}
Let $0<\epsilon<0.064$ be a fixed number.
When 
\begin{align*}
\ell_k>\frac{324(d+1)^2C_{up}^4}{C_aC^4_{down}\lambda_{min}^2(1-4\epsilon)^4},
\end{align*}
we have
$\mathbb{E}(J_3-J_4)>\epsilon$. Therefore, he expected regret at time $t$ during the exploitation phase of the episode $k$ is
\begin{equation}\label{regn2}
\mathbb{E}(R_t)=\mathbb{E}\mathbb{E}(R_t|\tilde{\mathcal{H}}_{t-1})=\mathbb{E}[r_t(p_t^*)-r_t(p_t)]> \frac{\epsilon}{4}.
\end{equation}
By (\ref{regn1}) and (\ref{regn2}), the expected regrets at time $t$ during both the exploration and exploitation phases are larger than $\epsilon/4$.
Therefore, when $T>\frac{324(d+1)^2C_{up}^4}{C_aC^4_{down}\lambda_{min}^2(1-4\epsilon)^4}$, we have
$\sum_{t=1}^T\mathbb{E}(R_t)>\frac{\epsilon T}{4}.$

\section{Proof under Strategic Pricing Policy with Known \texorpdfstring{$A$}{Lg}}\label{s1}
In this section, we first prove Lemma \ref{lemma11}, which provides an upper bound on the estimation error of the maximum likelihood estimator. This lemma serves as a crucial building block for the proof of Theorem \ref{theory2}. Once we establish Lemma \ref{lemma11}, we proceed to prove Theorem \ref{theory2}.
\subsection{Proof of Lemma \ref{lemma11}}
The proof of Lemma \ref{lemma11} is inspired by the proofs in \cite{Koren2015} and \cite{Xu2021}. We first define the log-likelihood function at time period $t$ as
\begin{equation}\label{likeli}
 l_t(\boldsymbol{\theta})=\mathbb{I}(y_t=1)\log (1-F(p_t-\boldsymbol{\theta}^\top \boldsymbol{x}_t^0))+\mathbb{I}(y_t=0)\log (F(p_t-\boldsymbol{\theta}^\top \boldsymbol{x}_t^0)).   
\end{equation}
Next, we define the expected log-likelihood function
$l^e(\boldsymbol{\theta})=\mathbb{E}[l_t(\boldsymbol{\theta})].$ Before proving Lemma \ref{lemma11}, we will first establish two lemmas: Lemma \ref{lemma1}, which provides a bound on the error of the expected log-likelihood function, and Lemma \ref{est1}, which deals with the likelihood error of the maximum likelihood estimator. These lemmas will serve as building blocks for the proof of Lemma \ref{lemma11}. \par
Now, we proceed with the presentation and  proof of Lemma \ref{lemma1}, where we present a lower bound for $l^e(\boldsymbol{\theta}_0)-l^e(\boldsymbol{\theta}), \forall \boldsymbol{\theta} \in \Theta$. 
\begin{lemma}\label{lemma1}
Under Assumptions \ref{a0}-\ref{a2}, we have 
\begin{equation}
l^e(\boldsymbol{\theta}_0)-l^e(\boldsymbol{\theta}) \geq \frac{\lambda_{min}}{2}C_{down}(\boldsymbol{\theta}-\boldsymbol{\theta}_0)^\top (\boldsymbol{\theta}-\boldsymbol{\theta}_0), \forall \boldsymbol{\theta} \in \Theta,
\end{equation}
where $C_{down}$ is defined in (\ref{cdown}), and $\lambda_{min}$ is the minimum eigenvalue of $\Sigma=\mathbb{E}[\boldsymbol{x}_t^0\boldsymbol{x}_t^{0\top}]$.
\end{lemma}
\begin{proof}
By taking the derivative of $l_t(\boldsymbol{\theta})$ defined in (\ref{likeli}) with respect to $\boldsymbol{\theta}$, we have
\begin{equation*}
\nabla l_t(\boldsymbol{\theta})=\mathbb{I}(y_t=1)\frac{f(p_t-\boldsymbol{\theta}^\top \boldsymbol{x}_t^0)}{1-F(p_t-\boldsymbol{\theta}^\top \boldsymbol{x}_t^0)}\boldsymbol{x}_t^0-\mathbb{I}(y_t=0)\frac{f(p_t-\boldsymbol{\theta}^\top \boldsymbol{x}_t^0)}{F(p_t-\boldsymbol{\theta}^\top \boldsymbol{x}_t^0)}\boldsymbol{x}_t^0.
\end{equation*}
By Assumptions \ref{a0} and \ref{a3}, we have $p_t-\boldsymbol{\theta}^\top \boldsymbol{x}_t^0\in [-W, B]$, where $W=W_xW_{\theta}$.
Next, we take the derivative of $\nabla l_t(\boldsymbol{\theta})$, and get
\begin{equation}\label{sec}
\begin{aligned}
\nabla^2 l_t(\boldsymbol{\theta})&=-\mathbb{I}(y_t=1)\frac{f'(p_t-\boldsymbol{\theta}^\top \boldsymbol{x}_t^0)[1-F(p_t-\boldsymbol{\theta}^\top \boldsymbol{x}_t)]+f^2(p_t-\boldsymbol{\theta}^\top \boldsymbol{x}_t^0)}{[1-F(p_t-\boldsymbol{\theta}^\top \boldsymbol{x}_t^0)]^2}\boldsymbol{x}_t^0\boldsymbol{x}_t^{0\top}\\
&~~~+\mathbb{I}(y_t=0)\frac{f'(p_t-\boldsymbol{\theta}^\top x_t^0)F(p_t-\boldsymbol{\theta}^\top x_t)-f^2(p_t-\boldsymbol{\theta}^\top x_t^0)}{F^2(p_t-\boldsymbol{\theta}^\top x_t^0)}\boldsymbol{x}_t^0\boldsymbol{x}_t^{0\top}\\
&=\mathbb{I}(y_t=1)\log''(1-F(\omega))|_{\omega=p_t-\boldsymbol{\theta}^\top \boldsymbol{x}_t^0}\boldsymbol{x}_t^0\boldsymbol{x}_t^{0\top}+\mathbb{I}(y_t=0)\log''(F(\omega))|_{\omega=p_t-\boldsymbol{\theta}^\top \boldsymbol{x}_t^0}\boldsymbol{x}_t^0\boldsymbol{x}_t^{0\top}\\
& \preceq\mathop{\sup}_{\omega\in [-W, B]}\mathop{\max}\{\log''(1-F(\omega)), \log''(F(\omega))\}\boldsymbol{x}_t^0\boldsymbol{x}_t^{0\top}\\
&= -\mathop{\inf}_{\omega\in [-W, B]}\mathop{\min}\{-\log''(1-F(\omega)), -\log''(F(\omega))\}\boldsymbol{x}_t^0\boldsymbol{x}_t^{0\top}\\
&=-C_{down}\boldsymbol{x}_t^0\boldsymbol{x}_t^{0\top}.
\end{aligned}
\end{equation}
By Assumption \ref{a1}, $C_{down}$ exists.
By taking Taylor expansion of $l_t(\boldsymbol{\theta})$ at $\boldsymbol{\theta}=\boldsymbol{\theta}_0$, we have
\begin{equation}\label{talor}
 l_t(\boldsymbol{\theta})=l_t(\boldsymbol{\theta}_0) +  \nabla l_t(\boldsymbol{\theta})(\boldsymbol{\theta}-\boldsymbol{\theta}_0)+\frac{1}{2}(\boldsymbol{\theta}-\boldsymbol{\theta}_0)^\top \nabla^2 l(\tilde{\boldsymbol{\theta}})(\boldsymbol{\theta}-\boldsymbol{\theta}_0) ,
\end{equation}
where $\tilde{\boldsymbol{\theta}}$ is between $\boldsymbol{\theta}$ and $\boldsymbol{\theta}_0$.
Since the true parameter always maximizes the expected likelihood function, we have $\nabla l^e(\boldsymbol{\theta}_0)=0$. Taking the expectation of Equation ($\ref{talor}$), we have
\begin{align*}
l^e(\boldsymbol{\theta}_0)-l^e(\boldsymbol{\theta})& = -\frac{1}{2}(\boldsymbol{\theta}-\boldsymbol{\theta}_0)^\top \nabla^2 l^e(\tilde{\boldsymbol{\theta}})(\boldsymbol{\theta}-\boldsymbol{\theta}_0)    \\
&\geq \frac{1}{2} C_{down}(\boldsymbol{\theta}-\boldsymbol{\theta}_0)^\top \mathbb{E}(\boldsymbol{x}_t^0\boldsymbol{x}_t^{0\top}) (\boldsymbol{\theta}-\boldsymbol{\theta}_0)\\
&\geq \frac{\lambda_{min}}{2} C_{down}(\boldsymbol{\theta}-\boldsymbol{\theta}_0)^\top (\boldsymbol{\theta}-\boldsymbol{\theta}_0).
\end{align*}
The first inequality is due to (\ref{sec}), and the second inequality is due to Assumption \ref{a2}.
\end{proof}
Now, we present an upper bound on the likelihood error of the maximum likelihood estimator.
\begin{lemma}\label{est1} 
Assume that we have $n$ $i.i.d.$ samples $\{(\boldsymbol{x}_1^0, p_1, y_1),\cdots, (\boldsymbol{x}_n^0, p_n, y_n)\}$.  Let the log-likelihood function be
$L(\boldsymbol{\theta})=\frac{1}{n}\sum_{ t= 1}^n l_i(\boldsymbol{\theta})$, where $l_i(\boldsymbol{\theta})$ is defined in (\ref{likeli}). We denote the maximum likelihood estimator as
$\hat{\boldsymbol{\theta}}=\mathop{\arg\min}_{\boldsymbol{\theta}\in\Theta} L(\boldsymbol{\theta})$.
Then we have
$$\mathbb{E}[L(\boldsymbol{\theta}_0)-L(\hat{\boldsymbol{\theta}})]\leq \frac{2 (d+1) C_{up}^2}{(n+1)C_{down}},$$
where $C_{down}$ is defined in (\ref{cdown}) and $C_{up}$ is defined in (\ref{cup}).
\end{lemma}
\begin{proof}
We define the "leave-one-out" log-likelihood function as 
$$\tilde{L}_i(\boldsymbol{\theta})=\frac{1}{n}\sum_{j=1, j\neq i}^n l_j(\boldsymbol{\theta}),$$ and let $\tilde{\boldsymbol{\theta}}_i=\mathop{\arg\max}_{\boldsymbol{\theta}} \tilde{L}_i(\boldsymbol{\theta}).$ Denote $H=\sum_{t=1}^n\boldsymbol{x}_t^0\boldsymbol{x}_t^{0\top}$.
By (\ref{sec}), and noting $\boldsymbol{x}_1^0, \cdots, \boldsymbol{x}_n^0$ are $i.i.d.$, and $p_1,...,p_t$ are $i.i.d.$ in the exploration phase, we have
$$\nabla^2 L(\boldsymbol{\theta})\preceq -\frac{1}{n}C_{down}H.$$ By the singular value decomposition, we have $H=U\tilde{\Sigma} U^\top$, where $U\in\mathbb{R}^{(d+1)\times r}, U^\top U=I_r, \tilde{\Sigma}=diag\{\lambda_1,...,\lambda_r\}\succ 0$. We define $\eta:=U^\top \boldsymbol{\theta}$. There exist $V\in \mathbb{R}^{(d+1)\times (d+1-r)}, \zeta\in\mathbb{R}^{d+1-r}, V^\top V=I_{d+1-r}, V^\top U=0$,  such that $\boldsymbol{\theta}=U\eta+V\zeta$. We define the following new functions,
\begin{align*}
\hat{l}_i(\eta):=l_i(\boldsymbol{\theta})=l_i(U\eta+V\zeta),\hat{L}_i(\eta):=\tilde{L}_i(\boldsymbol{\theta})=\tilde{L}_i(U\eta+V\zeta),\hat{L}(\eta):=L(\boldsymbol{\theta})=L(U\eta+V\zeta).
\end{align*}
By taking the second derivative of $\nabla^2\hat{l}_i(\eta)$, we have
\begin{align*}
\nabla^2\hat{l}_i(\eta)=\frac{\partial^2l_i}{\partial (\boldsymbol{\theta}^\top \boldsymbol{x}_i^0)^2}\frac{\partial \boldsymbol{\theta}^\top \boldsymbol{x}_i^0}{\partial \eta}\bigg(\frac{\partial \boldsymbol{\theta}^\top \boldsymbol{x}_i^0}{\partial \eta}\bigg)^\top=\frac{\partial^2l_i}{\partial (\boldsymbol{\theta}^\top \boldsymbol{x}_i^0)^2} (U^\top \boldsymbol{x}_i^0)(U^\top \boldsymbol{x}_i^0)^\top
\preceq -C_{down}U^\top \boldsymbol{x}_i^{0} \boldsymbol{x}_i^{0\top}U.
\end{align*}
Therefore,
\begin{align*}
\nabla^2\hat{L}(\eta)=\frac{1}{n}\sum_{i=1}^n\nabla^2\hat{l}_i(\eta)\preceq -\frac{1}{n}\sum_{i=1}^n C_{down}U^\top  \boldsymbol{x}_i^{0} \boldsymbol{x}_i^{0\top}U=-\frac{1}{n} C_{down}U^\top U\tilde{\Sigma} U^\top U=-\frac{1}{n}C_{down}\tilde{\Sigma}.
\end{align*}
Thus, $-\nabla^2\hat{L}(\eta)\succeq \frac{1}{n}C_{down}\tilde{\Sigma} \succ 0$. Therefore, $-\nabla^2\hat{L}(\eta)$ is locally $\frac{C_{down}}{n}$-strongly convex with respect to $\tilde{\Sigma}$ at $\eta$. Similarly, we can prove $-\tilde{L}_i(\eta)$ is convex.  Let $g_1(\eta)=-\tilde{L}_i(\eta)$ and $g_2(\eta):=-\hat{L}(\eta)$. Then $g_2(\eta)-g_1(\eta)=-\frac{1}{n}\hat{l}_i(\eta)$. We define $\tilde{\eta}_i=U^\top \tilde{\boldsymbol{\theta}}_i$ and $\hat{\eta}=U^\top \hat{\boldsymbol{\theta}}$. According to Lemma $\ref{Koren}$, we have
\begin{equation}\label{eta0}
\|\hat{\eta}-\tilde{\eta}_i\|_{\tilde{\Sigma}}\leq \frac{2}{C_{down}}\|\nabla \hat{l}_i(\tilde{\eta}_i)\|_{\tilde{\Sigma}}^*.
\end{equation}
By the convexity of $-\hat{l}_i(\cdot)$, we have
\begin{equation}\label{ll10}
 \begin{aligned}
l_i(\hat{\boldsymbol{\theta}})-l_i(\tilde{\boldsymbol{\theta}}_i)&=[-l_i(\tilde{\boldsymbol{\theta}}_i)]-[-l_i(\hat{\boldsymbol{\theta}})]
=[-\hat{l}_i(\tilde{\eta}_i)]-[-\hat{l}_i(\hat{\eta})]
\leq -\nabla \hat{l}_i(\tilde{\eta}_i)^\top (\tilde{\eta}_i-\hat{\eta}).
\end{aligned}   
\end{equation}
Therefore,
\begin{equation}\label{ll1}
l_i(\hat{\boldsymbol{\theta}})-l_i(\tilde{\boldsymbol{\theta}}_i) \leq \|\nabla \hat{l}_i(\tilde{\eta}_i)\|_{\tilde{\Sigma}}^*\|\tilde{\eta}_i-\hat{\eta}\|_{\tilde{\Sigma}}\leq \frac{2}{C_{down}}(\|\nabla \hat{l}_i(\tilde{\eta}_i)\|_{\tilde{\Sigma}}^*)^2.   
\end{equation}
The first inequality is from (\ref{ll10}) and the Hölder's inequality, and the second inequality follows (\ref{eta0}). Since 
$$\boldsymbol{x}_i^{0\top}U\tilde{\Sigma}^{-1}U^\top \boldsymbol{x}_i^{0}=tr(\boldsymbol{x}_i^{0\top}U\tilde{\Sigma}^{-1}U^\top \boldsymbol{x}_i^{0})=tr(U\tilde{\Sigma}^{-1}U^\top \boldsymbol{x}_i^0\boldsymbol{x}_i^{0\top}),$$ we have
\begin{equation}\label{ll2}
 \begin{aligned}
 (\|\nabla \hat{l}_i(\tilde{\eta}_i)\|_{\tilde{\Sigma}}^*)^2&=   \left \|\frac{\partial l_i}{\partial (\tilde{\boldsymbol{\theta}}_i^\top \boldsymbol{x}_i^0)}\frac{\partial (\tilde{\boldsymbol{\theta}}_i^\top \boldsymbol{x}_i^0)}{\partial \tilde{\eta}_i}\right \|_{\tilde{\Sigma}}^{*2}\leq C_{up}^2\boldsymbol{x}_i^{0\top}U\tilde{\Sigma}^{-1}U^\top \boldsymbol{x}_i^{0}=C_{up}^2 tr(U\tilde{\Sigma}^{-1}U^\top \boldsymbol{x}_i^0\boldsymbol{x}_i^{0\top}).
\end{aligned}   
\end{equation}
By (\ref{ll1}) and (\ref{ll2}), we have
\begin{align*}
\sum_{i=1}^n[l_i(\hat{\boldsymbol{\theta}})-l_i(\tilde{\boldsymbol{\theta}}_i)]&\leq \frac{2 C_{up}^2}{C_{down}} tr(U\tilde{\Sigma}^{-1}U^\top \sum_{i=1}^n\boldsymbol{x}_i^0\boldsymbol{x}_i^{0\top})=\frac{2 C_{up}^2}{C_{down}} tr(U\tilde{\Sigma}^{-1}U^\top H)\leq\frac{2 (d+1) C_{up}^2}{C_{down}}.
\end{align*}
The second inequality is from
$ tr(U\tilde{\Sigma}^{-1}U^\top H)=tr(U\tilde{\Sigma}^{-1}U^\top U\tilde{\Sigma} U^\top)=tr(UU^\top)\leq d+1.$
Since $\tilde{\boldsymbol{\theta}}_i$ is the MLE of $(n-1)\ i.i.d.$ samples, $\tilde{\boldsymbol{\theta}}_1,...,\tilde{\boldsymbol{\theta}}_n$ have exactly the same distribution. Thus,
\begin{align*}
\mathbb{E}[L(\boldsymbol{\theta}_0)-L(\tilde{\boldsymbol{\theta}}_n)]&\leq \mathbb{E}[L(\hat{\boldsymbol{\theta}})-L(\tilde{\boldsymbol{\theta}}_n)]=\frac{1}{n}\sum_{i=1}^n \mathbb{E}[l_i(\hat{\boldsymbol{\theta}})-l_i(\tilde{\boldsymbol{\theta}}_i)]\leq \frac{2 (d+1) C_{up}^2}{nC_{down}}.
\end{align*}
Noting $\tilde{\boldsymbol{\theta}}_{n+1}=\hat{\boldsymbol{\theta}}$, the proof is completed.
\end{proof}
Now, we continue to prove Lemma \ref{lemma11}.
Noting that there are $a_k$ $i.i.d.$ samples for obtaining $\hat{\boldsymbol{\theta}}_k$ in the $k$-th episode, by Lemma \ref{lemma1} and Lemma \ref{est1},  we have
 \begin{equation}\label{theta1}
\begin{aligned}
\mathbb{E}\|\hat{\boldsymbol{\theta}}_k-\boldsymbol{\theta}_0\|_2^2 &\leq \frac{2}{C_{down}\lambda_{min}} \mathbb{E} [l^e(\boldsymbol{\theta}_0)-l^e(\hat{\boldsymbol{\theta}}_k)]\\
&=\frac{2}{C_{down}\lambda_{min}}[\mathbb{E}L_k(\boldsymbol{\theta}_0)-\mathbb{E}L_k(\hat{\boldsymbol{\theta}})]\\
&\leq \frac{2(d+1)C_{up}^2}{C^2_{down}\lambda_{min}(a_k+1)}.
\end{aligned}     
 \end{equation}

\subsection{Proof of Theorem \ref{theory2}}\label{pof2}
In order to bound the total regret, we first try to bound the regret at each episode $k$. The regret in the exploration phase during the $k$-th episode is bounded by $Ba_k$. Now we analyze the upper bound on the regret during the exploitation phase.\par
We let $r_t(p)=p[1-F(p-\boldsymbol{\theta}_0^\top \boldsymbol{x}_t^0)]$ be the expected revenue. We define the filtration generated by all transaction records up to time $t$ as $ 
\mathcal{H}_{t}=\sigma(\boldsymbol{x}_1^0, \boldsymbol{x}_2^0, \cdots, \boldsymbol{x}_t^0, z_1, z_2, \cdots, z_t)$. We also define $\Tilde{\mathcal{H}}_{t}=\mathcal{H}_{t}\cup \{\boldsymbol{x}_{t+1}^0\}$ as the filtration obtained after augmenting by the new feature $\boldsymbol{x}_{t+1}^0$. We define the regret at time $t$ as $R_t=p_t^*\mathbb{I}(v_t\geq p_t^*)-p_t\mathbb{I}(v_t\geq p_t)$. Then the conditional expectation of the regret at time $t$ given previous information and $\boldsymbol{x}_t^0$ is
\begin{align}\label{R}
\mathbb{E}(R_t|\Tilde{\mathcal{H}}_{t-1})&=\mathbb{E}[p_t^*\mathbb{I}(v_t\geq p_t^*)-p_t\mathbb{I}(v_t\geq p_t)|\Tilde{\mathcal{H}}_{t-1}]\nonumber \\
&=p_t^*[1-F(p_t^*-\boldsymbol{\theta}_0^\top \boldsymbol{x}_t^0)]-p_t[1-F(p_t-\boldsymbol{\theta}_0^\top \boldsymbol{x}_t^0)]\nonumber \\
&=r_t(p_t^*)-r_t(p_t).
\end{align}
Note that $p_t^*\in \mathop{\arg\max}_pr_t(p)$ and hence we have $r'_t(p_t^*)=0$. Using Taylor expansion, we have
\begin{equation}\label{rt}
 r_t(p_t)=r_t(p_t^*)+\frac{1}{2}r''_t(\xi_t)(p_t-p_t^*)^2,   
\end{equation}
where $\xi_t$ is some value between $p_t$ and $p_t^*$. By Assumptions \ref{a3} and \ref{a1}, we have 
\begin{equation}\label{r2}
    \begin{aligned}
     |r''_t(\xi_t)|&=|2f(\xi_t-\boldsymbol{\theta}_0^\top \boldsymbol{x}_t^0)+\xi_tf'(\xi_t-\boldsymbol{\theta}_0^\top \boldsymbol{x}_t^0)|
     \leq 2M_f+BM_{f'}.   
    \end{aligned}
\end{equation}
Now we can obtain an upper bound on the conditional expectation of the regret at time $t$ given $\Tilde{\mathcal{H}}_{t-1}$.
By (\ref{R}), (\ref{rt}) and (\ref{r2}), we have
\begin{equation}\label{Rbound}
 \mathbb{E}(R_t|\Tilde{\mathcal{H}}_{t-1})\leq \left(M_f+\frac{B}{2}M_{f'}\right)\mathbb{E}(p_t^*-p_t)^2.   
\end{equation}
Now we give an upper bound of $(p_t^*-p_t)^2$. During the episode $k$, for time $t$ in the exploitation phase, we have
\begin{equation}\label{pbound}
\begin{aligned}
(p_t^*-p_t)^2&=[g(\boldsymbol{\theta}_0^\top \boldsymbol{x}_t^0)-g(\hat{\boldsymbol{\theta}}_k^\top \boldsymbol{x}_t+\hat{\boldsymbol{\beta}}_k^\top A^{-1}\hat{\boldsymbol{\beta}}_kg'(\hat{\boldsymbol{\theta}}_k^\top \boldsymbol{x}_t))]^2\\
&\leq [\boldsymbol{\theta}_0^\top \boldsymbol{x}_t^0- \hat{\boldsymbol{\theta}}_k^\top \boldsymbol{x}_t-\hat{\boldsymbol{\beta}}_k^\top A^{-1}\hat{\boldsymbol{\beta}}_kg'(\hat{\boldsymbol{\theta}}_k^\top \boldsymbol{x}_t)]^2\\
&=[\boldsymbol{\theta}_0^\top \boldsymbol{x}_t^0- \hat{\boldsymbol{\theta}}_k^\top \boldsymbol{x}_t^0+\hat{\boldsymbol{\beta}}_k^\top A^{-1}\boldsymbol{\beta}_0 g'(\boldsymbol{\theta}_0 \boldsymbol{x}_t)-\hat{\boldsymbol{\beta}}_k^\top A^{-1}\hat{\boldsymbol{\beta}}_kg'(\hat{\boldsymbol{\theta}}_k^\top \boldsymbol{x}_t)]^2\\
&\leq 2| (\boldsymbol{\theta}_0-\hat{\boldsymbol{\theta}}_k)^\top \boldsymbol{x}_t^0|^2+2[\hat{\boldsymbol{\beta}}_k^\top A^{-1}\boldsymbol{\beta}_0 g'(\boldsymbol{\theta}_0 \boldsymbol{x}_t)-\hat{\boldsymbol{\beta}}_k^\top A^{-1}\hat{\boldsymbol{\beta}}_kg'(\hat{\boldsymbol{\theta}}_k^\top \boldsymbol{x}_t)]^2\\
&:=2J_1+2J_2.
\end{aligned}
\end{equation}
The first inequality is due to Lemma \ref{t1}. The second equality is from Equation (\ref{xt}).\par
We first analyze $J_1$. By Lemma \ref{lemma11}, we have
\begin{equation}\label{ej1}
\begin{aligned}
\mathbb{E}J_1&=\mathbb{E} | (\boldsymbol{\theta}_0-\hat{\boldsymbol{\theta}}_k)^\top \boldsymbol{x}_t^0|^2\\
&=\mathbb{E} [(\boldsymbol{\theta}_0-\hat{\boldsymbol{\theta}}_k)^\top \boldsymbol{x}_t^0\boldsymbol{x}_t^{0\top} (\boldsymbol{\theta}_0-\hat{\boldsymbol{\theta}}_k)]\\
&\leq \lambda_{max}\mathbb{E}\|\boldsymbol{\theta}_0-\hat{\boldsymbol{\theta}}_k\|_2^2\\
&\leq \frac{2(d+1)C_{up}^2\lambda_{max}}{C^2_{down}\lambda_{min}(a_k+1)}.
\end{aligned}
\end{equation}
Next, we analyze $J_2$. By Lemma \ref{t2}, we assume $\|g''(\cdot)\|<C_{g''}$ on the bounded interval $[-W, B]$ for some constant $C_{g''}>0$. Therefore,
\begin{equation}\label{J2}
\begin{aligned}
J_2&=|\hat{\boldsymbol{\beta}}_k^\top A^{-1}\boldsymbol{\beta}_0 g'(\boldsymbol{\theta}_0 \boldsymbol{x}_t)-\hat{\boldsymbol{\beta}}_k^\top A^{-1}\hat{\boldsymbol{\beta}}_kg'(\hat{\boldsymbol{\theta}}_k^\top \boldsymbol{x}_t)|^2\\
&=|\hat{\boldsymbol{\beta}}_k^\top A^{-1}\boldsymbol{\beta}_0 g'(\boldsymbol{\theta}_0 \boldsymbol{x}_t)-\hat{\boldsymbol{\beta}}_k^\top A^{-1}\hat{\boldsymbol{\beta}}_k g'(\boldsymbol{\theta}_0 \boldsymbol{x}_t)+\hat{\boldsymbol{\beta}}_k^\top A^{-1}\hat{\boldsymbol{\beta}}_k g'(\boldsymbol{\theta}_0 \boldsymbol{x}_t)-\hat{\boldsymbol{\beta}}_k^\top A^{-1}\hat{\boldsymbol{\beta}}_kg'(\hat{\boldsymbol{\theta}}_k^\top \boldsymbol{x}_t)|^2\\
&=|\hat{\boldsymbol{\beta}}_k^\top A^{-1}(\boldsymbol{\beta}_0 -\hat{\boldsymbol{\beta}}_k)g'(\boldsymbol{\theta}_0 \boldsymbol{x}_t)+\hat{\boldsymbol{\beta}}_k^\top A^{-1}\hat{\boldsymbol{\beta}}_k[g'(\boldsymbol{\theta}_0^\top \boldsymbol{x}_t)-g'(\hat{\boldsymbol{\theta}}_k^\top \boldsymbol{x}_t)]|^2\\
&\leq 2|\hat{\boldsymbol{\beta}}_k^\top A^{-1}(\boldsymbol{\beta}_0 -\hat{\boldsymbol{\beta}}_k)g'(\boldsymbol{\theta}_0 \boldsymbol{x}_t)|^2+2|\hat{\boldsymbol{\beta}}_k^\top A^{-1}\hat{\boldsymbol{\beta}}_k[g'(\boldsymbol{\theta}_0^\top \boldsymbol{x}_t)-g'(\hat{\boldsymbol{\theta}}_k^\top \boldsymbol{x}_t)]|^2\\
&\leq 2|\hat{\boldsymbol{\beta}}_k^\top A^{-1}(\boldsymbol{\beta}_0 -\hat{\boldsymbol{\beta}}_k)|^2+2C^2_{g''}|\hat{\boldsymbol{\beta}}_k^\top A^{-1}\hat{\boldsymbol{\beta}}_k\boldsymbol{x}_t^\top(\boldsymbol{\theta}_0-\hat{\boldsymbol{\theta}}_k)|^2\\
&\leq 2[\|\hat{\boldsymbol{\beta}}_k^\top A^{-1}\|_2^2+C^2_{g''}\|\hat{\boldsymbol{\beta}}_k^\top A^{-1}\hat{\boldsymbol{\beta}}_k\boldsymbol{x}_t^\top\|_2^2]\|\boldsymbol{\theta}_0-\hat{\boldsymbol{\theta}}_k\|_2^2.
\end{aligned}    
\end{equation}
The last second inequality is due to Lemma $\ref{t2}$. The last inequality is from $\boldsymbol{\theta}_0= (\boldsymbol{\beta}_0^\top, \alpha_0)^\top$. Now, we derive a upper bound of $\|\boldsymbol{x}_t\|_2^2$. By Equation (\ref{xt}), we have
\begin{equation}\label{c3}
\begin{aligned}
\boldsymbol{x}_t^\top \boldsymbol{x}_t&=1+[\tilde{\boldsymbol{x}}_t^0- A^{-1}\boldsymbol{\beta}_0 g'(\boldsymbol{\theta}_0^\top \boldsymbol{x}_t)]^\top [\tilde{\boldsymbol{x}}_t^0- A^{-1}\boldsymbol{\beta}_0 g'(\boldsymbol{\theta}_0^\top \boldsymbol{x}_t)]^\top\\
&=1+\tilde{\boldsymbol{x}}_t^{0\top}\tilde{\boldsymbol{x}}_t^0+\boldsymbol{\beta}_0^{\top}(A^{-1})^2\boldsymbol{\beta}_0[g'(\boldsymbol{\theta}_0^\top \boldsymbol{x}_t)]^2-2\tilde{\boldsymbol{x}}_t^{0\top}A^{-1}\boldsymbol{\beta}_0g'(\boldsymbol{\theta}_0^\top \boldsymbol{x}_t)\\
&\leq 1+ W_x^2+\frac{W_\theta^2}{\lambda_{Amin}^2}+\frac{2W_xW_\theta}{\lambda_{Amin}}:=C_x.
\end{aligned}
\end{equation} 
The first inequality is because of Assumption \ref{a0}.
Then, by (\ref{J2}) and (\ref{c3}), we have
\begin{equation}\label{ej2}
\begin{aligned}
\mathbb{E}J_2&\leq 2\mathbb{E}\{[\|\hat{\boldsymbol{\beta}}_k^\top A^{-1}\|_2^2+C^2_{g''}\|\hat{\boldsymbol{\beta}}_k^\top A^{-1}\hat{\boldsymbol{\beta}}_k\boldsymbol{x}_t^\top\|_2^2]\|\boldsymbol{\theta}_0-\hat{\boldsymbol{\theta}}_k\|_2^2\}\\
&\leq \frac{2W_{\theta}^2+2W_{\theta}^4C_{g''}^2C_x}{\lambda_{Amin}^2} \mathbb{E}\|\boldsymbol{\theta}_0-\hat{\boldsymbol{\theta}}_k\|_2^2\\
&\leq \frac{2W_\theta^2+2W_\theta^4C_{g''}^2C_x}{\lambda_{Amin}^2}\frac{2(d+1)C_{up}^2}{C^2_{down}\lambda_{min}(a_k+1)}.
\end{aligned}
\end{equation}
The first inequality is from Assumption \ref{a0}, and the third inequality is from Lemma \ref{lemma11}.
By Equations (\ref{Rbound}), (\ref{pbound}), (\ref{ej1}) and (\ref{ej2}), the expected regret at time $t$ during the exploitation phase of the episode $k$ is 
\begin{align*}
 \mathbb{E}(R_t)&=\mathbb{E}[\mathbb{E}(R_t|\Tilde{\mathcal{H}}_{t-1})] \\
 &\leq (M_f+\frac{B}{2}M')\mathbb{E}(2J_1+2J_2)\\
 &=\frac{2(d+1)(2M_f+BM_{f'})C_{up}^2}{C^2_{down}\lambda_{min}(a_k+1)}\left(\lambda_{max}+\frac{2W_\theta^2+2W_\theta^4C_{g''}^2C_x}{\lambda_{Amin}^2}\right).
\end{align*}
Therefore, 
The total expected regret during the $k$-th episode including the exploration phase and the exploitation phase is
\begin{align*}
Regret_k&\leq Ba_k+(\ell_k-a_k)\frac{2(d+1)(2M_f+BM_{f'})C_{up}^2}{C^2_{down}\lambda_{min}(a_k+1)}\left(\lambda_{max}+\frac{2W_\theta^2+2W_\theta^4C_{g''}^2C_x}{\lambda_{Amin}^2}\right)\\
&<Ba_k+(\ell_k-a_k)\frac{2(d+1)(2M_f+BM_{f'})C_{up}^2}{C^2_{down}\lambda_{min}a_k}\left(\lambda_{max}+\frac{2W_\theta^2+2W_\theta^4C_{g''}^2C_x}{\lambda_{Amin}^2}\right).
\end{align*}
Denote $C_{a}=\frac{2(d+1)(2M_f+BM_{f'})C_{up}^2}{BC^2_{down}\lambda_{min}}\left(\lambda_{max}+\frac{2W_\theta^2+2W_\theta^4C_{g''}^2C_x}{\lambda_{Amin}^2}\right)$. We have
\begin{align*}
Regret_k< Ba_k+\frac{BC_{a}(\ell_k-a_k)}{a_k}
=Ba_k+\frac{BC_{a}\ell_k}{a_k}-BC_{a}\leq 2B\sqrt{C_{a}\ell_k}-C_{a}.
\end{align*}
Noting that $a_k=\sqrt{C_{a}\ell_k}$ minimizes the upper bound of $Regret_k$. 
Now, since the length of episodes grows exponentially, the number of episodes by period $T$ is logarithmic in $T$. Specifically, $T$ belongs to episode $K=\lfloor \log_2 \frac{T}{l_0}\rfloor +1$. Therefore,
\begin{equation*}
\sum_{k=1}^K\sqrt{\ell_k}= \sqrt{\ell_0}\sum_{k=1}^K2^{\frac{k-1}{2}}=\sqrt{\ell_0}\frac{2^{\frac{K}{2}}-1}{\sqrt{2}-1}
\leq \sqrt{\ell_0}\frac{\sqrt{\frac{2T}{\ell_0}}-1}{\sqrt{2}-1}< (2+\sqrt{2})\sqrt{T}.   
\end{equation*}
Thus, the total expected regret up to time period $T$ can be bounded by
\begin{align*}
Regret(T)&=\sum_{k=1}^KRegret_k\leq \sum_{k=1}^K(2B\sqrt{C_{a}\ell_k}-C_{a})<2(2+\sqrt{2})B\sqrt{C_{a}T}.
\end{align*}
Finally, we define two new constants,
\begin{align*}
C_1^*&=\frac{8(2+\sqrt{2})^2B(2M_f+BM_{f'})C_{up}^2\lambda_{max}}{C^2_{down}\lambda_{min}},\\
C_2^*&=\frac{16(2+\sqrt{2})^2B(2M_f+BM_{f'})C_{up}^2(W_\theta^2+W_\theta^4C_{g''}^2C_x)}{C^2_{down}\lambda_{min}}.
\end{align*}
The proof is completed.

\section{Proof under Strategic Pricing Policy with Unknown \texorpdfstring{$A$}{Lg}}\label{s3}
In this section, our first step is to prove Lemma \ref{lemmaunA}, which establishes an upper bound on the estimation error of $\boldsymbol{\gamma}=-A^{-1}\boldsymbol{\beta}_0$. This lemma plays a pivotal role as a fundamental component in the proof of Theorem \ref{theory5}. Once we have successfully demonstrated Lemma \ref{lemmaunA}, we will proceed with the subsequent step of proving Theorem \ref{theory5}.
\subsection{Proof of Lemma \ref{lemmaunA}}
Assume that we obtain $n$ samples $\{(\boldsymbol{\delta}_{1}, u_1),\cdots,(\boldsymbol{\delta}_{n}, u_n)\}$, and the latest sample is obtained in the $k$-th episode.
We define 
$\boldsymbol{\varepsilon}_j=(\boldsymbol{\epsilon}_{j1},...,\boldsymbol{\epsilon}_{jn})^\top.$ 
By Equation (\ref{ols}), the estimation error of the $j$-th component of $\boldsymbol{\gamma}$ is
\begin{equation}\label{c0}
\begin{aligned}
|\hat{\boldsymbol{\gamma}}_j-\boldsymbol{\gamma}_j|&=\bigg|\frac{\boldsymbol{u}^{\top}\boldsymbol{\Delta}_j}{\boldsymbol{u}^\top \boldsymbol{u}}-\boldsymbol{\gamma}_j\bigg|
=\bigg |\frac{\boldsymbol{u}^{\top}(\boldsymbol{\gamma}_j\boldsymbol{u}+\boldsymbol{\varepsilon}_j)}{\boldsymbol{u}^\top \boldsymbol{u}}-\boldsymbol{\gamma}_j\bigg |
=\frac{|\boldsymbol{u}^{\top}\boldsymbol{\varepsilon}_j|}{\boldsymbol{u}^\top \boldsymbol{u}}.
\end{aligned}    
\end{equation}
To establish an upper bound on $|\hat{\boldsymbol{\gamma}}_j-\boldsymbol{\gamma}_j|$, we need to bound the terms $|\boldsymbol{u}^{\top}\boldsymbol{\varepsilon}_j|$ and $\boldsymbol{u}^\top \boldsymbol{u}$ separately.\par
Firstly, we derive a lower bound on $\boldsymbol{u}^\top \boldsymbol{u}$.
Since $0<g'(\cdot)<1$ by Lemma \ref{t1} and $g'(\cdot)$ is continuous, there exists $c_{g'}>0$ such that $g'(\cdot)>c_{g'}$ over the bounded interval $[-W, B]$. Let $\hat{\boldsymbol{\theta}}_k$ be the estimate of $\boldsymbol{\theta}_0$ calculated from (\ref{est}).  By the definition of $u_t$ in (\ref{gamma1}), we have
\begin{equation}\label{c1}
\begin{aligned}
\boldsymbol{u}^\top \boldsymbol{u}=u_1^2+\cdots+u_n^2=[g'(\hat{\boldsymbol{\theta}}_{k}^\top \boldsymbol{x}_1)]^2+\cdots +[g'(\hat{\boldsymbol{\theta}}_{k}^\top \boldsymbol{x}_n)]^2
>nc_{g'}^2.
\end{aligned}
\end{equation}

Secondly, we derive an upper bound on $\mathbb{E}|\boldsymbol{u}^\top \boldsymbol{\varepsilon}_j|^2$. By Lemma \ref{t2}, $g'(\cdot)$ is locally Lipschitz continuous on $[-W, B]$. Then there exists constant $C_{g''}>0$ such that $|g''(\cdot)|<C_{g''}$ on the bounded interval $[-W, B]$. Therefore,
\begin{equation}\label{g1}
\begin{aligned}
\mathbb{E}|g'(\hat{\boldsymbol{\theta}}_k^\top \boldsymbol{x}_t)-g'(\boldsymbol{\theta}_0^\top \boldsymbol{x}_t)|^2&\leq C_{g''}^2\mathbb{E}|\hat{\boldsymbol{\theta}}_k^\top \boldsymbol{x}_t-\boldsymbol{\theta}_0^\top \boldsymbol{x}_t|_2^2\}\\
&\leq C_{g''}^2\mathbb{E}(\|\hat{\boldsymbol{\theta}}_k-\boldsymbol{\theta}_0\|_2^2\|\boldsymbol{x}_t\|_2^2)\\
&\leq C_{g''}^2C_x\mathbb{E}\|\hat{\boldsymbol{\theta}}_k-\boldsymbol{\theta}_0\|_2^2.
\end{aligned}    
\end{equation}
The last inequality is due to (\ref{c3}). Next, we have
\begin{equation}\label{ee2}
\begin{aligned}
\mathbb{E}\|\boldsymbol{\varepsilon}_j\|_2^2&=\boldsymbol{\gamma}_j^2\sum_{i=1}^n\mathbb{E}|g'(\hat{\boldsymbol{\theta}}_k^\top \boldsymbol{x}_i)-g'(\boldsymbol{\theta}_0^\top \boldsymbol{x}_i)|^2\\
&\leq n\boldsymbol{\gamma}_j^2C_{g''}^2C_x\mathbb{E}\|\hat{\boldsymbol{\theta}}_k-\boldsymbol{\theta}_0\|_2^2\\
&\leq  \frac{2n\boldsymbol{\gamma}_j^2C_{g''}^2C_x(d+1)C_{up}^2}{C^2_{down}\lambda_{min}(a_k+1)}\\
&\leq \frac{4\boldsymbol{\gamma}_j^2\sqrt{\ell_k}C_{g''}^2C_x(d+1)C_{up}^2}{C^2_{down}\lambda_{min}}.
\end{aligned}
\end{equation}
The second inequality is from Lemma \ref{lemma11}. The last inequality is from  $n\leq  \tau \sum_{i=1}^{k}\ell_i=\tau (2\ell_k-\ell_0)<2\tau \ell_k$, $a_k=\lfloor \sqrt{C_a\ell_k} \rfloor$ and the fact $\tau<\sqrt{C_a}$. Noting that $\|\boldsymbol{u}\|_2^2<n$, by (\ref{ee2}), we have
\begin{equation}\label{c2}
 \mathbb{E}|\boldsymbol{u}^\top \boldsymbol{\varepsilon}_j|^2\leq \mathbb{E}\|\boldsymbol{u}\|_2^2\|\boldsymbol{\varepsilon}_j\|_2^2 \leq   \frac{4n\boldsymbol{\gamma}_j^2\sqrt{\ell_k}C_{g''}^2C_x(d+1)C_{up}^2}{C^2_{down}\lambda_{min}}.
\end{equation}

Finally, we derive an upper bound on $\mathbb{E}\|\hat{\boldsymbol{\gamma}}-\boldsymbol{\gamma}\|_2^2$. When $k>1$, we have $\ell_0<\ell_k/2$. Then $n\geq \tau \sum_{i=1}^{k-1}\ell_i=\tau (\ell_k-\ell_0)\geq \tau \ell_k/2$ for $k>1$. By (\ref{c0}), (\ref{c1}) and (\ref{c2}), we have
\begin{equation}\label{c4}
 \begin{aligned}
\mathbb{E}(\hat{\boldsymbol{\gamma}}_j-\boldsymbol{\gamma}_j)^2&\leq \frac{\mathbb{E}|\boldsymbol{u}^\top \boldsymbol{\varepsilon}_j|^2}{n^2c_{g'}^4}  
\leq \frac{4n\boldsymbol{\gamma}_j^2\sqrt{\ell_k}C_{g''}^2C_x(d+1)C_{up}^2}{C^2_{down}\lambda_{min}n^2c_{g'}^4}
\leq \frac{8\boldsymbol{\gamma}_j^2C_{g''}^2C_x(d+1)C_{up}^2}{\tau C^2_{down}\lambda_{min}c_{g'}^4\sqrt{\ell_k}}.
\end{aligned}   
\end{equation}
Noting that $\sum_{j=1}^d\boldsymbol{\gamma}_j^2=\|A^{-1}\boldsymbol{\beta}_0\|_2^2\leq \frac{W_\theta^2}{\lambda_{Amin}^2}$, by (\ref{c4}), we have for $k>1$
\begin{equation}\label{gamma3}
\mathbb{E}\|\hat{\boldsymbol{\gamma}}-\boldsymbol{\gamma}\|_2^2=\mathbb{E}\sum_{j=1}^d(\hat{\boldsymbol{\gamma}}_j-\boldsymbol{\gamma}_j)^2
 \leq  \frac{8C_{g''}^2C_x(d+1)C_{up}^2\sum_{j=1}^d\gamma_j^2}{\tau C^2_{down}\lambda_{min}c_{g'}^4\sqrt{\ell_k}}\leq \frac{8W_\theta^2C_{g''}^2C_x(d+1)C_{up}^2}{\tau \lambda_{Amin}^2C^2_{down}\lambda_{min}c_{g'}^4\sqrt{\ell_k}}.  
\end{equation}
Denote $C_{\gamma}^*=\frac{8W_\theta^2C_{g''}^2C_x C_{up}^2}{ \lambda_{Amin}^2C^2_{down}\lambda_{min}c_{g'}^4}$. The proof is completed.

\subsection{Proof of Theorem \ref{theory5}}
The main idea of the proof of Theorem \ref{theory5} is similar to the proof of Theorem \ref{theory2}. The regret in the exploration phase during the $k$-th episode is bounded by $Ba_k$. Thus, our focus now shifts to analyzing the upper bound of the regret during the exploitation phase. By referring to equation (\ref{Rbound}), we observe that the conditional expectation of regret at time $t$ in the exploitation phase can be bounded by $(p_t^* - p_t)^2$. Consequently, our proof begins by deriving a bound for $(p_t^* - p_t)^2$. \par 
During the episode $k$, for $t$ in the exploitation phase, we have
\begin{equation}\label{pbound1}
\begin{aligned}
(p_t^*-p_t)^2&=[g(\boldsymbol{\theta}_0^\top \boldsymbol{x}_t^0)-g(\hat{\boldsymbol{\theta}}_k^\top \boldsymbol{x}_t-\hat{\boldsymbol{\beta}}_k^\top \hat{\boldsymbol{\gamma}}g'(\hat{\boldsymbol{\theta}}_k^\top \boldsymbol{x}_t))]^2\\
&\leq [\boldsymbol{\theta}_0^\top \boldsymbol{x}_t^0- \hat{\boldsymbol{\theta}}_k^\top \boldsymbol{x}_t+\hat{\boldsymbol{\beta}}_k^\top \hat{\boldsymbol{\gamma}}g'(\hat{\boldsymbol{\theta}}_k^\top \boldsymbol{x}_t)]^2\\
&=[\boldsymbol{\theta}_0^\top \boldsymbol{x}_t^0- \hat{\boldsymbol{\theta}}_k^\top \boldsymbol{x}_t^0+\hat{\boldsymbol{\beta}}_k^\top A^{-1}\boldsymbol{\beta}_0 g'(\boldsymbol{\theta}_0 \boldsymbol{x}_t)+\hat{\boldsymbol{\beta}}_k^\top \hat{\boldsymbol{\gamma}}g'(\hat{\boldsymbol{\theta}}_k^\top \boldsymbol{x}_t)]^2\\
&\leq 2| (\boldsymbol{\theta}_0-\hat{\boldsymbol{\theta}}_k)^\top \boldsymbol{x}_t^0|^2+2[\hat{\boldsymbol{\beta}}_k^\top \boldsymbol{\gamma} g'(\boldsymbol{\theta}_0 \boldsymbol{x}_t)-\hat{\boldsymbol{\beta}}_k^\top \hat{\boldsymbol{\gamma}}g'(\hat{\boldsymbol{\theta}}_k^\top \boldsymbol{x}_t)]^2\\
&:=2J_1+2J_5.
\end{aligned}
\end{equation}
The first inequality is because of Lemma \ref{t1}. The second equality follows Equation (\ref{xt}). Noting that $J_1$ in (\ref{pbound1}) is exactly the same with that in (\ref{pbound}). Therefore, we only need to find a upper bound on $J_5$. \par
Now, we analyze $J_5$. By Lemma \ref{t2}, we assume $\|g''(\cdot)\|<C_{g''}$ on the bounded interval $[-W, B]$ for some constant $C_{g''}>0$. From (\ref{pbound1}), we have
\begin{equation}\label{j3}
 \begin{aligned}
J_5&=[\hat{\boldsymbol{\beta}}_k^\top \boldsymbol{\gamma} g'(\boldsymbol{\theta}_0^\top \boldsymbol{x}_t)-\hat{\boldsymbol{\beta}}_k^\top \hat{\boldsymbol{\gamma}}g'(\hat{\boldsymbol{\theta}}_k^\top \boldsymbol{x}_t)]^2\\
&=|\hat{\boldsymbol{\beta}}_k^\top (\boldsymbol{\gamma} -\hat{\boldsymbol{\gamma}})g'(\boldsymbol{\theta}_0^\top \boldsymbol{x}_t)+\boldsymbol{\hat{\beta}}_k^\top \hat{\boldsymbol{\gamma}}[g'(\boldsymbol{\theta}_0^\top \boldsymbol{x}_t)-g'(\hat{\boldsymbol{\theta}}_k^\top \boldsymbol{x}_t)]|^2\\
&\leq 2|\boldsymbol{\hat{\beta}}_k^\top (\boldsymbol{\gamma} -\hat{\boldsymbol{\gamma}})g'(\boldsymbol{\theta}_0^\top \boldsymbol{x}_t)|^2+2|\hat{\boldsymbol{\beta}}_k^\top \hat{\boldsymbol{\gamma}}[g'(\boldsymbol{\theta}_0^\top \boldsymbol{x}_t)-g'(\hat{\boldsymbol{\theta}}_k \boldsymbol{x}_t)]|^2\\
&\leq 2|\boldsymbol{\hat{\beta}}_k^\top (\boldsymbol{\gamma} -\hat{\boldsymbol{\gamma}})|^2+2C^2_{g''}|\hat{\boldsymbol{\beta}}_k^\top \hat{\boldsymbol{\gamma}}\boldsymbol{x}_t^\top(\boldsymbol{\theta}_0-\hat{\boldsymbol{\theta}}_k)|^2\\
&\leq 2\|\boldsymbol{\hat{\beta}}_k\|^2\|\boldsymbol{\gamma} -\hat{\boldsymbol{\gamma}}\|_2^2+2C^2_{g''}\|\hat{\boldsymbol{\beta}}_k^\top\hat{\boldsymbol{\gamma}}\boldsymbol{x}_t^\top\|_2^2\|\boldsymbol{\theta}_0-\hat{\boldsymbol{\theta}}_k\|_2^2.
\end{aligned}   
\end{equation}
The last second inequality is due to Lemma $\ref{t2}$. Now, we derive a upper bound of $\|\hat{\boldsymbol{\gamma}}\|_2^2$. Assume that  $\hat{\boldsymbol{\gamma}}$ is estimated from $n$ samples. We have
\begin{equation}\label{ue}
|\boldsymbol{u}^{\top}\boldsymbol{\varepsilon}_j|=|\boldsymbol{\gamma}_j\sum_{t=1}^nu_t[g'(\hat{\boldsymbol{\theta}}_k^\top \boldsymbol{x}_t)-g'(\boldsymbol{\theta}_0^\top \boldsymbol{x}_t)]| \leq n|\gamma_j|.    
\end{equation}
The last inequality is due to $0<u_t=g'(\hat{\boldsymbol{\theta}}_k\boldsymbol{x}_t)<1$ by Lemma \ref{t1}, and hence $|g'(\hat{\boldsymbol{\theta}}_k\boldsymbol{x}_t)-g'(\boldsymbol{\theta}^0\boldsymbol{x}_t)|<1$ By (\ref{ols}), we have
\begin{equation}\label{gamma2}
|\hat{\boldsymbol{\gamma}}_j|=\bigg|\frac{\boldsymbol{u}^{\top}\boldsymbol{\Delta}_j}{\boldsymbol{u}^\top \boldsymbol{u}}\bigg|
=\bigg|\frac{\boldsymbol{u}^{\top}(\boldsymbol{\gamma}_j\boldsymbol{u}+\boldsymbol{\varepsilon}_j)}{\boldsymbol{u}^\top \boldsymbol{u}}\bigg|
\leq |\boldsymbol{\gamma}_j|+\bigg|\frac{\boldsymbol{u}^{\top}\boldsymbol{\varepsilon}_j}{\boldsymbol{u}^\top \boldsymbol{u}}\bigg|
\leq \bigg(1+\frac{1}{c_{g'}^2}\bigg)|\boldsymbol{\gamma}_j|.
\end{equation}
The last inequality is due to (\ref{c1}) and (\ref{ue}).
Therefore,
\begin{equation}\label{gamma}
\begin{aligned}
\|\hat{\boldsymbol{\gamma}}\|_2^2&=\sum_{j=1}^d \hat{\boldsymbol{\gamma}}_j^2
\leq \bigg(1+\frac{1}{c_{g'}^2}\bigg)^2\sum_{t=1}^d\boldsymbol{\gamma}_j^2
\leq \bigg(1+\frac{1}{c_{g'}^2}\bigg)^2\frac{W_{\theta}^2}{\lambda_{Amin}^2}.
\end{aligned}
\end{equation}
The last inequality is due to $\sum_{t=1}^d\boldsymbol{\gamma}_j^2=\|\boldsymbol{\gamma}\|_2^2=\|A^{-1}\boldsymbol{\theta}_0\|_2^2\leq \frac{W_\theta^2}{\lambda_{Amin}^2}$.
Then, by (\ref{j3}) and (\ref{gamma}), we have for $k>1$
\begin{equation}\label{ej21}
\begin{aligned}
\mathbb{E}J_5&\leq 2\mathbb{E}\{ \|\boldsymbol{\hat{\beta}}_k\|^2_2\|\boldsymbol{\gamma} -\hat{\boldsymbol{\gamma}}\|_2^2+C^2_{g''}\|\hat{\boldsymbol{\beta}}_k\|_2^2\|\hat{\boldsymbol{\gamma}}\|_2^2\|\boldsymbol{x}_t\|_2^2\|\boldsymbol{\theta}_0-\hat{\boldsymbol{\theta}}_k\|_2^2\}\\
&\leq 2W_{\theta}^2\mathbb{E}\|\boldsymbol{\gamma} -\hat{\boldsymbol{\gamma}}\|_2^2+2C^2_{g''}\bigg(1+\frac{1}{c_{g'}^2}\bigg)^2\frac{W_{\theta}^4C_x}{\lambda_{Amin}^2}\mathbb{E}\|\boldsymbol{\theta}_0-\hat{\boldsymbol{\theta}}_k\|_2^2\\\
&\leq \frac{16W_\theta^4C_{g''}^2C_x(d+1)C_{up}^2}{\tau \lambda_{Amin}^2C^2_{down}\lambda_{min}c_{g'}^4\sqrt{\ell_k}}+\bigg(1+\frac{1}{c_{g'}^2}\bigg)^2\frac{2C^2_{g''}W_{\theta}^4C_x}{\lambda_{Amin}^2} \frac{2(d+1)C_{up}^2}{C^2_{down}\lambda_{min}(a_k+1)}\\
&=\frac{4W_\theta^4C_{g''}^2C_xC_{up}^2(d+1)}{\lambda_{Amin}^2C^2_{down}\lambda_{min} c_{g'}^4}\bigg[\frac{4}{\tau\sqrt{\ell_k}}+\frac{(1+c_{g'})^2}{a_k+1}\bigg].
\end{aligned}
\end{equation}
The second inequality is due to (\ref{c3}) and (\ref{gamma}). The third inequality follows Lemma \ref{lemma11} and \ref{lemmaunA}. By equations (\ref{Rbound}), (\ref{ej1}), (\ref{pbound1}) and (\ref{ej21}), for $k>1$, the expected regret at time $t$ is 
\begin{align*}
 \mathbb{E}(R_t)&=\mathbb{E}[\mathbb{E}(R_t|\Tilde{\mathcal{H}}_{t-1})] \\
 &\leq \bigg(M_f+\frac{B}{2}M_{f'}\bigg)\mathbb{E}(2J_1+2J_5)\\
 &\leq (2M_f+BM_{f'})\bigg\{\frac{2(d+1)C_{up}^2\lambda_{max}}{C^2_{down}\lambda_{min}(a_k+1)}+\frac{4W_\theta^4C_{g''}^2C_xC_{up}^2(d+1)}{\lambda_{Amin}^2C^2_{down}\lambda_{min} c_{g'}^4}\bigg[\frac{4}{\tau\sqrt{\ell_k}}+\frac{(1+c_{g'})^2}{a_k+1}\bigg]\bigg\}\\
 &= \frac{2(2M_f+BM_{f'})(d+1)C_{up}^2}{C^2_{down}\lambda_{min}}\left \{ \frac{\lambda_{max}}{a_k+1}+\frac{2W_\theta^4C_{g''}^2C_x}{\lambda_{Amin}^2 c_{g'}^4}\bigg[\frac{4}{\tau\sqrt{\ell_k}}+\frac{(1+c_{g'})^2}{a_k+1}\bigg]\right\}.
\end{align*}
To simplify the above formula, we define
\begin{align*}
C_1&=\frac{2(2M_f+BM_{f'})(d+1)C_{up}^2}{C^2_{down}\lambda_{min}}\bigg[\lambda_{max}+\frac{2W_\theta^4C_{g''}^2C_x(1+c_{g'})^2}{\lambda_{Amin}^2 c_{g'}^4}\bigg], \\
C_2&=\frac{2(2M_f+BM_{f'})(d+1)C_{up}^2}{C^2_{down}\lambda_{min}}\frac{8W_\theta^4C_{g''}^2C_x}{\tau \lambda_{Amin}^2 c_{g'}^4}=\frac{16(2M_f+BM_{f'})(d+1)C_{up}^2W_\theta^4C_{g''}^2C_x}{\tau \lambda_{Amin}^2 c_{g'}^4C^2_{down}\lambda_{min}}.
\end{align*}
Therefore,
\begin{equation}
\mathbb{E}(R_t)\leq\frac{C_1}{a_k}+\frac{C_2}{\sqrt{\ell_k}}.    
\end{equation}
Therefore, The total expected regret during the $k$-th episode including the exploration phase
and the exploitation phase is
\begin{align*}
Regret_k&\leq Ba_k+(\ell_k-a_k)\left (\frac{C_1}{a_k}+\frac{C_2}{\sqrt{\ell_k}}\right)
<Ba_k+\frac{C_1\ell_k}{a_k}+C_2\sqrt{\ell_k}.
\end{align*}
We choose $a_k=\sqrt{\frac{C_1\ell_k}{B}}$, which minimizes the upper bound of $Regret_k$.
Therefore,
\begin{align*}
Regret_k&< 2\sqrt{BC_1\ell_k}+C_2\sqrt{\ell_k}=(2\sqrt{BC_1}+C_2)\sqrt{\ell_0}2^{\frac{k-1}{2}}.
\end{align*}
Since the length of episodes grows exponentially, the number of episodes by period $T$ is logarithmic in $T$. Specifically, $T$ belongs to episode $K=\lfloor \log \frac{T}{l_0}\rfloor +1$.
The total expected regret can be bounded by
\begin{align*}
Regret(T)&=(2\sqrt{BC_1}+C_2)\sqrt{\ell_0}\sum_{k=1}^K2^{\frac{k-1}{2}}\\
&=(2\sqrt{BC_1}+C_2)\sqrt{\ell_0}\frac{2^{\frac{K}{2}}-1}{\sqrt{2}-1}\\
&\leq (2\sqrt{BC_1}+C_2)\sqrt{\ell_0}\frac{\sqrt{\frac{2T}{l_0}}-\sqrt{2}}{\sqrt{2}-1}\\
&< (2\sqrt{BC_1}+C_2)(2+\sqrt{2})\sqrt{T}.
\end{align*}
Finally, we define two new constants,
\begin{align*}
C_3^*&=(4+2\sqrt{2})\sqrt{\frac{BC_1}{d+1}}=\frac{(4\sqrt{2}+4)C_{up}}{C_{down}}\sqrt{\frac{(2M_f+BM_{f'})B}{\lambda_{min}}\bigg[\lambda_{max}+\frac{2W_\theta^4C_{g''}^2C_x(1+c_{g'})^2}{\lambda_{Amin}^2 c_{g'}^4}\bigg]},\\
C_4^*&=\frac{(2+\sqrt{2})C_2\tau \lambda_{Amin}^2}{d+1}=\frac{16(2+\sqrt{2})(2M_f+BM_{f'})C_{up}^2W_\theta^4C_{g''}^2C_x}{ c_{g'}^4C^2_{down}\lambda_{min}}.
\end{align*} The proof is completed.
\section{Technical Lemmas}\label{s4}
\begin{lemma}\label{t1}
If $1-F$ is log-concave, the pricing function $g(\cdot)$ is 1-Lipschitz continuous.
\end{lemma}
\begin{proof}
We write the virtual valuation function as $\phi(v)=v-1/\lambda(v)$ where $\lambda(v)=\frac{f(v)}{1-F(v)}=-\log'(1-F(v)$ is the hazard function. Since $1-F$ is log-concave, the hazard function $\lambda(v)$ is increasing, $i.e.$, $\lambda'(v)\geq 0$. Then, 
\begin{equation}\label{phi}
\phi'(v)=1- \bigg [\frac{1}{\lambda(v)}\bigg]'=1+\frac{\lambda'(v)}{\lambda^2(v)}>1.    
\end{equation}
Since $g(v)=v+\phi^{-1}(-v)$, we have $g'(v)=1-1/\phi'(\phi^{-1}(-v)).$ By equation (\ref{phi}), we obtain $0<g'(v)<1$. Therefore, $g(\cdot)$ is 1-Lipschitz continuous.
\end{proof}
\begin{lemma}\label{t2}
If $1-F$ is log-concave, the first derivative $g'(\cdot)$ is locally Lipschitz continuous on $[-W, B]$.
\end{lemma}
\begin{proof}
Noting $\phi(v)=v-[1-F(v)]/f(v)$, by (\ref{phi}), we have
$$\phi'(v)=1-\frac{-f^2(v)-[1-F(v)]f'(v)}{f^2(v)}=\frac{2f^2(v)+[1-F(v)]f'(v)}{f^2(v)}>1.$$ 
Thus,
\begin{equation}
\label{e1}
    2f^2(v)+[1-F(v)]f'(v)>f^2(v).
\end{equation}
\begin{align*}
\phi''(v)&=\frac{(-f(v)f'(v)+(1-F(v))f''(v))f^2-2(1-F(v))f'(v)f(v)f'(v)}{f(v)^4}\\
&=\frac{-f(v)^2f'(v)+(1-F(v))(f''(v)f(v)-2f'^2(v))}{f^3(v)}.
\end{align*}

 Let $u=\phi^{-1}(-v)$. Since $0<g(v)=v+\phi^{-1}(-v)\leq B$, we have $-v<\phi^{-1}(-v)\leq B-v$.
\begin{align*}
g''(v)&=-\frac{\phi''(u)}{[\phi'(u)]^2}\frac{1}{\phi'(\phi^{-1}(-v))}\\
&=-\frac{\phi''(u)}{[\phi'(u)]^3}\\
&-\frac{-f^2(u)f'(u)+(1-F(u))(f''(u)f(u)-2f'^2(u))}{f^3(u)}\frac{f^6(u)}{(2f^2(u)+(1-F(u))f'(u))^3}\\
&=\frac{f^3(u)[f^2(u)f'(u)-(1-F(u))(f''(u)f(u)-2f'^2(u))]}{(2f^2(u)+(1-F(u))f'(u))^3}.
\end{align*}
By (\ref{e1}), we have
\begin{align*}
|g''(v)|&\leq \frac{|f^3(u)[f^2(u)f'(u)-(1-F(u))(f''(u)f(u)-2f'^2(u))]|}{f^6(u)}\\
&=\frac{|f^2(u)f'(u)-(1-F(u))(f''(u)f(u)-2f'^2(u))|}{f^3(u)}.
\end{align*}
By assumption \ref{a1}, $g''(v)$ is bounded. Therefore, $g'(\cdot)$ is locally Lipschitz continuous.
\end{proof}
We next present a lemma from \cite{Koren2015} as our supporting Lemma.
\begin{lemma}\label{Koren}
(Lemma 5, \cite{Koren2015}) Let $g_1, g_2: \mathcal{K}\rightarrow \mathbb{R}$ be two convex functions defined over a closed and convex domain $\mathcal{K}\subseteq \mathbb{R}^d$, and let $x_1=\mathop{\arg\min}_{x\in \mathcal{K}}g_1(x)$ and $x_2=\mathop{\arg\min}_{x\in \mathcal{K}}g_2(x)$. Assume that $g_2$ is locally $\sigma$-strongly-convex at $x_1$ with respect to a norm $\|\cdot\|$. Then, for $h=g_2-g_1$, we have
\begin{align*}
\|x_2-x_1\|\leq \frac{2}{\sigma}\|\nabla h(x_1)\|^*,
\end{align*}
where $\|\cdot\|^*$ is a dual norm.
\end{lemma}
\end{document}

%% file: def-smile.tex
\let\hat\widehat
\let\tilde\widetilde

\newtheorem{lemma}{{\bf Lemma}}

\newtheorem{theorem}{{\bf Theorem}}

\newtheorem{assumption}{{\bf Assumption}}

\newtheorem{definition}{{\bf Definition}}

\newtheorem{remark}{{\bf Remark}}